\documentclass{article}

\usepackage{arxiv}
\usepackage[utf8]{inputenc}
\usepackage[T1]{fontenc}
\usepackage{amsmath,amssymb,amsthm,mathtools}
\usepackage{booktabs,array,multirow}
\usepackage{graphicx}
\usepackage[rightcaption]{sidecap}
\usepackage{xcolor}
\usepackage{enumitem}
\usepackage{adjustbox}
\usepackage{natbib}
\usepackage{microtype}
\usepackage{hyperref}
\usepackage{url}

\definecolor{darkgreen}{rgb}{0.0, 0.5, 0.0}
\definecolor{darkred}{rgb}{0.6, 0.0, 0.0}
\definecolor{gray}{rgb}{0.5, 0.5, 0.5}
\definecolor{myblue}{RGB}{0, 0, 180}
\definecolor{mygreen}{RGB}{0, 108, 0}
\definecolor{myred}{RGB}{196, 0, 0}
\definecolor{darkmagenta}{RGB}{149, 0, 149}

\definecolor{supportlimitcolor}{HTML}{014C99}
\definecolor{approximationgapcolor}{HTML}{990000}
\definecolor{alignmentgapcolor}{HTML}{994C00}
\definecolor{executiongapcolor}{HTML}{4C0099}
\DeclareRobustCommand{\supportlimit}{\textcolor{supportlimitcolor}{\mbox{\textbf{support limit}}}}
\DeclareRobustCommand{\approximationgap}{\textcolor{approximationgapcolor}{\mbox{\textbf{approximation gap}}}}
\DeclareRobustCommand{\alignmentgap}{\textcolor{alignmentgapcolor}{\mbox{\textbf{alignment gap}}}}
\DeclareRobustCommand{\executiongap}{\textcolor{executiongapcolor}{\mbox{\textbf{execution gap}}}}
\hypersetup{
    colorlinks=true,
    linkcolor=myred,
    citecolor=mygreen,
    filecolor=darkmagenta,
    urlcolor=darkmagenta
}

\newcommand{\Etail}{E_{\mathrm{tail}}}
\newcommand{\norm}[1]{\left\lVert #1\right\rVert}
\newcommand{\RR}{\mathbb{R}}
\newcommand{\mat}[1]{\mathbf{#1}}
\newcommand{\vect}[1]{\boldsymbol{#1}}
\newcommand{\set}[1]{\mathcal{#1}}
\newcommand{\Id}{\mat{I}}
\newcommand{\epsnum}{\varepsilon_{\mathrm{num}}}

\newtheorem{proposition}{Proposition}
\newtheorem{theorem}{Theorem}
\newtheorem{corollary}{Corollary}

\sidecaptionvpos{figure}{t}
\sidecaptionvpos{table}{t}

\newcommand{\papertitle}{%
  Reachability Is Not Enough:\\
  Diagnosing Long-Range Behavior in GNNs%
}
\newcommand{\paperauthor}{Filippo Maria Bianchi}
\newcommand{\paperaffiliations}{%
  UiT The Arctic University of Norway\\
  NORCE Norwegian Research Centre%
}
\newcommand{\paperemail}{filippo.m.bianchi@uit.no}

\title{\papertitle}
\author{\paperauthor}
\date{}

\author{%
    \paperauthor\\
    \normalfont\paperaffiliations\\
    \href{mailto:\paperemail}{\texttt{\paperemail}}%
}
\renewcommand{\headeright}{}
\renewcommand{\undertitle}{}
\renewcommand{\shorttitle}{Reachability is not enough}
\hypersetup{
    pdftitle={Reachability is not enough: Diagnosing long-range behavior in GNNs},
    pdfauthor={\paperauthor}
}

\begin{document}
\maketitle

\begin{abstract}
Graph neural networks (GNNs) are often called long-range because their architecture can connect distant nodes, but this does not show whether they use distant information correctly.
We introduce a framework that measures how strongly inputs at each graph distance affect predictions and separates limitations due to architecture, finite approximation, training, and numerical execution.
Our analysis shows that local message-passing can spread influence slowly, so a finite implementation may rely mainly on nearby inputs even when the ideal computation uses the whole graph.
We also explain why mathematically equivalent filters can differ in how easily they are learned and how reliably they run.
Across controlled tasks, models with similar architectural reach use distant information very differently, while low average error can hide failures on distant interactions.
Together, these results show that long-range capability depends on learning to use information at the distances required by the task and preserving that use during computation.
\end{abstract}

\section{Introduction}

Learning dependencies between distant nodes remains a difficult and open challenge in graph deep learning. 
Models must preserve and use information across long graph distances, whether retrieving a marked source, predicting effects of remote atoms, or aggregating weak interactions over a large region.
Supporting such dependencies has motivated deeper message passing, global attention, adaptive propagation, spectral filters, and state-space recurrences~\citep{gutteridge2023drew,rampasek2022gps,errica2025amp,behrouz2024graphmamba,ceni2025mpssm}. 
These mechanisms differ substantially in computational cost and inductive bias, so it matters whether their advertised long-range capability is actually used.
Long-range capability is often inferred from architectural properties: receptive-field size, global attention, or the number of propagation steps.
These properties only establish that communication is possible. 
They do not show that training uses the available route, that the resulting influence, i.e., the input-output sensitivity, fullfills the task requirements, or that finite computation and numerical precision preserve it.
In addition, aggregate prediction metrics can hide these kinds of failure when distant interactions contribute only a small part of the loss.

To address this, we measure the \emph{realized interaction profile} of a trained model: how strongly inputs at each graph distance affect its predictions.
Its \emph{realized range} summarizes how far this influence extends.
We compare this profile with what the architecture permits, what an ideal operator would produce, and what the task requires. 
Through input perturbations we verify that predictions follow the intended distant information, while errors measured on distant interactions expose effects that typically go unnoticed.
This turns long-range capability from an architectural label into a diagnosis of where information is lost: communication, finite approximation, training, or numerical execution.

Our contributions are:
\begin{enumerate}[leftmargin=*,itemsep=1pt,topsep=2pt]
    \item \textbf{A framework for diagnosing long-range failures.}
    We separate limits in architectural reach from failures in approximation, training, and numerical execution.
    \item \textbf{Theory linking range to implementation.}
    We show how architecture limits how far information can travel and why stopping an iterative diffusion process after a limited number of steps can greatly shorten its realized range.
    For propagation methods that repeatedly average over neighboring nodes, doubling their realized range requires roughly four times as many updates.
    We also explain why a model that in principle may learn the desired solution can still be difficult to fit or execute reliably in practice.
    \item \textbf{Mechanistic evidence across tasks and models.}
    Through controlled experiments across several GNN families, we show where long-range behavior is lost.
    Our results demonstrate that being able to communicate over long distances is not enough: models must also learn to preserve and use the information that matters for the downstream task.
\end{enumerate}

\section{From Possible Communication to Realized Interaction}
\label{sec:framework}

\begin{figure}[ht]
\centering
\includegraphics[width=\textwidth,trim=4.6cm 0 0 2.8cm,clip]{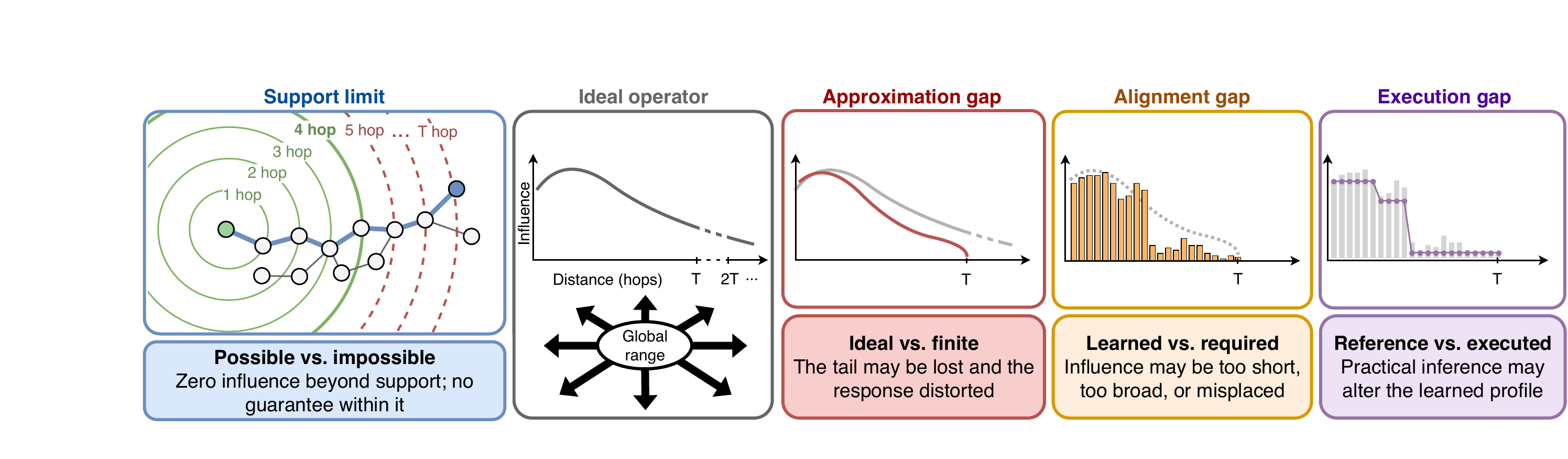}
\caption{\footnotesize
Architectural support determines which node pairs may interact.
When an ideal operator exists, a finite implementation approximates it; training may fail to learn the profile required by the task, and approximations and rounding errors during inference may further alter the learned profile.
}
\label{fig:overview}
\end{figure}

\paragraph{Architectural support.}
The \emph{support} of an output tells which input nodes can possibly affect it.
A local model with $T$ one-hop message-passing steps has support only within $T$ hops, while rewiring, virtual nodes, hierarchy, or global attention can enlarge it.

\paragraph{Realized interaction profiles.}
Support identifies which nodes can affect an output.
For a trained model on given inputs, its realized interaction profile measures how strongly inputs at each distance affect that output.
Models with identical support can have very different realized interaction profiles.
We compare this profile with two references: the \emph{ideal operator profile}, from the exact mechanism approximated by the model (e.g., an equilibrium), and the \emph{target profile}, from the interactions required by the task.
The two references could not coincide.

\paragraph{Learning and numerical execution.}
Learning and numerical execution both shape a model's predictions.
The \emph{learned profile} describes the interactions implied by the trained parameters under accurate computation.
The \emph{executed profile} also reflects the effects of finite iteration budgets, sparsification, solver tolerances, and arithmetic precision, which can alter these interactions even when the parameters remain fixed.

\paragraph{Measuring influence and realized range.}
Let $\set{G}=(\set{V},\set{E})$ be a graph and $d(u,v)$ the shortest-path distance between nodes $u$ and $v$. Let $\mat{X}\in\RR^{|\set{V}|\times F}$ be the node-feature matrix, where $F$ is the number of features per node, and let $\vect{x}_u\in\RR^F$ denote its $u$\textsuperscript{th} row. For an output at node $v$, let $s_v(\mat{X})$ be a scalar task score, e.g., a class margin for classification or one prediction coordinate for regression. For continuous inputs, we measure the influence of input node $u$ by
\[
a_v(u)
=
\norm{\frac{\partial s_v(\mat{X})}{\partial \vect{x}_u}}_2.
\]
For discrete features, we instead use finite, task-valid perturbations rather than relying only on local derivatives: $a_v(u)$ is the absolute change in $s_v$ after replacing the input at $u$ with another value allowed by the task.
Let $A_v=\sum_{u\in\set{V}}a_v(u)$. When $A_v>0$, the normalized node profile is $p_v(u)=a_v(u)/A_v$, and its distance profile is $p_v(r)=\sum_{u:\,d(u,v)=r}p_v(u)$. For $q\in(0,1)$, its realized range $R_q(v)$ is the smallest radius containing at least a fraction $q$ of this measured influence:
\begin{equation}
R_q(v)
=
\min\left\{
r:
\sum_{u:\,d(u,v)\leq r}p_v(u)\geq q
\right\}.
\label{eq:Rq}
\end{equation}
$R_q$ summarizes the realized interaction profile: two models can have the same $R_{90}$\footnote{We use percentile subscripts for reported quantiles, so $R_{90}\equiv R_{q}$ with $q=0.9$.} while placing their influence at different distances.
When a target profile $p_v^\star(r)$ is known, we assess the \alignmentgap{} through profile agreement.
A larger range is useful only if it reduces this gap.
Matching the target distance profile does not guarantee that the model uses the correct nodes or features, since inputs at the same distance are grouped together.
We need task-specific interventions to check if predictions respond correctly to changes in relevant inputs and remain stable when irrelevant inputs are changed.

\paragraph{Tail fidelity.}
Some controlled tasks provide a known target input--output map, so we can compare the model's and target's interactions for individual input nodes.
We define the normalized tail error $\Etail(v;r_0)$ at radius $r_0\geq0$ as a measure of how closely the model's map matches the target map for input nodes more than $r_0$ hops from $v$.
A model with zero response throughout a nonzero target tail has $\Etail=1$, whereas exact recovery gives $\Etail=0$. 
Appendix~\ref{app:range} gives the full definition and edge cases, along with measures of how numerical execution changes model outputs and realized range.

These definitions lead to four questions: which interactions are impossible (the \supportlimit); how a finite implementation differs from its ideal operator (the \approximationgap); how the learned and target profiles differ (the \alignmentgap); and how practical inference changes the reference learned profile (the \executiongap).
Figure~\ref{fig:overview} illustrates the \supportlimit{} and the three gaps from the computational graph and any associated ideal operator to the learned and executed profiles.

\section{Why Possible and Realized Interactions Diverge}
\label{sec:theory}

The following results explain mechanisms behind the \supportlimit, the \approximationgap, and the \executiongap.
Local computation limits support, iterative propagation may require many message-passing steps to approach its ideal limit, and limited numerical precision can distort the final computation. We also show that some mathematical forms of the same filter are harder to fit, although theory alone cannot explain which profile training will select. 
We therefore measure the \alignmentgap{} empirically in Section~\ref{sec:experiments}.
Full theoretical statements and proofs are in Appendix~\ref{app:proofs}.

\subsection{Finite local computation sets the support limit}

\begin{proposition}[Finite local support]
\label{prop:support}
Suppose that every computational path from an input to an output contains at most $T$ operations that aggregate across one edge of the graph, while all remaining operations are node-wise. Then the output at node $v$ is independent of every input node $u$ with $d(u,v)>T$. Parallel branches do not increase this radius unless they are composed sequentially.
\end{proposition}

Whenever $R_q(v)$ is defined for a local model with $T$ one-hop message-passing steps, $R_q(v)\leq T$.
If the target map depends on inputs more than $T$ hops away, the model's response to those inputs is zero and $\Etail=1$.
Within the \supportlimit, however, Proposition~\ref{prop:support} says nothing about whether and how training uses the available routes.

\subsection{Why finite updates can greatly reduce interaction range}

Some architectures approximate an ideal operator. 
Their \approximationgap{} is the difference between that operator and the one produced by the finite updates used in practice.

\begin{theorem}[Finite realization of a contractive fixed point]
\label{thm:truncation}
Let $\mat{H}_t$ denote the state after $t$ updates in a finite-dimensional normed vector space and, for each fixed input $\mat{X}$, let $\mat{H}_{t+1}=\Phi_{\mat{X}}(\mat{H}_t)$, where $\Phi_{\mat{X}}$ is contractive in its state argument with contraction factor $0\leq\theta<1$ in the chosen norm. 
Then $\Phi_{\mat{X}}$ has a unique fixed point $\mat{H}^\star$, and for every integer $T\geq0$,
\begin{equation}
\norm{\mat{H}_T-\mat{H}^\star}
\leq
\theta^T\norm{\mat{H}_0-\mat{H}^\star}.
\label{eq:contractive-gap}
\end{equation}
If, in addition, the update is affine, let $\mat{M}$ be the state-transition operator, $\mat{B}$ the input-injection operator, and $\mat{C}$ the initialization operator, so that $\mat{H}_{t+1}=\mat{M}\mat{H}_t+\mat{B}\mat{X}$ and $\mat{H}_0=\mat{C}\mat{X}$. With $\Id$ denoting the identity on the state space, the fixed point is $\mat{H}^\star=\mat{R}\mat{X}$ for $\mat{R}=(\Id-\mat{M})^{-1}\mat{B}$, and the finite operator $\mat{P}_T$, defined by $\mat{H}_T=\mat{P}_T\mat{X}$, satisfies
\begin{equation}
\mat{R}-\mat{P}_T
=
\mat{M}^T(\mat{R}-\mat{C}).
\label{eq:finite-gap-identity}
\end{equation}
Hence $\norm{\mat{R}-\mat{P}_T}\leq\norm{\mat{M}^T}\norm{\mat{R}-\mat{C}}$, and contractivity gives $\norm{\mat{R}-\mat{P}_T}\leq \theta^T\norm{\mat{R}-\mat{C}}$.
\end{theorem}

Equation~\eqref{eq:finite-gap-identity} shows why the number of updates alone does not tell us how closely the finite operator $\mat{P}_T$ matches the ideal operator $\mat{R}$.
The remaining gap $\mat{R}-\mat{P}_T$ has two main factors: $\mat{M}^T$ describes how much of the initial error survives after $T$ updates, while $\mat{R}-\mat{C}$ measures the mismatch between the initialization and the ideal operator.
If $\mat{M}^T$ decays slowly or the initial mismatch is large, stopping after a fixed number of updates can make a process that would eventually use information from across the graph rely mostly on nearby nodes or distribute influence over the wrong distances.
Theorem~\ref{thm:truncation} applies beyond any particular graph-filter family.
For normalized restart diffusion, the gap has a closed-form bound, proved in Appendix~\ref{app:restart}.

\begin{corollary}[Normalized restart diffusion]
\label{cor:restart}
Let $\mat{S}$ be a symmetric graph propagation matrix, such as a normalized adjacency matrix, with $\norm{\mat{S}}_2\leq1$. Let $a$ be the signed recurrence coefficient and set $0<\beta=|a|<1$.
The exact equilibrium $\mat{H}^\star=\mat{R}_a\mat{X}$ satisfies
\[
\mat{H}^\star=a\mat{S}\mat{H}^\star+(1-\beta)\mat{X},
\qquad
\mat{R}_a=(1-\beta)(\Id-a\mat{S})^{-1}.
\]
The finite operator $\mat{P}_{a,T}$ is obtained by running the same update for $T$ steps:
\[
\mat{H}_0=\mat{X},
\qquad
\mat{H}_{t+1}=a\mat{S}\mat{H}_t+(1-\beta)\mat{X},
\qquad
\mat{H}_T=\mat{P}_{a,T}\mat{X}.
\]
Then
\begin{equation}
\norm{\mat{R}_a-\mat{P}_{a,T}}_2
\leq
\frac{2\beta^{T+1}}{1+\beta}.
\label{eq:truncation}
\end{equation}
\end{corollary}

Thus, $\mat{R}_a$ is the result of running the recurrence to equilibrium, whereas $\mat{P}_{a,T}$ is the result of stopping it after $T$ updates.
The exact operator combines walks of all lengths, while the finite operator includes only those reached within its update budget.
Equation~\eqref{eq:truncation} controls worst-case operator error; near-unit $\beta$ makes that error decay slowly.

We next characterize the distance profile of the finite realization itself.
The missing ingredient is how quickly influence travels during those updates.
A positive diffusion uses nonnegative weights, so contributions from different walks add rather than cancel; it can therefore be interpreted as an average over random walks.
A common case is ordinary diffusion: although a $T$-step walk can reach $T$ hops, its back-and-forth motion leaves most endpoints only about $\sqrt{T}$ hops from the start.
This creates an \approximationgap{} between the exact equilibrium $\mat{R}_a$, obtained by running the recurrence until convergence, and the finite implementation $\mat{P}_{a,T}$, obtained by stopping after $T$ updates.
If the exact equilibrium relies mainly on walks of length about $m$, the finite implementation is limited to an effective walk length of $\min\{m,T\}$.
On an ordinarily diffusive graph, the equilibrium range is about $\sqrt{m}$, while the finite range is about $\sqrt{\min\{m,T\}}$.
Preserving a nonvanishing fraction of the equilibrium range therefore requires $T$ to be comparable to $m$.
In other words, doubling the desired interaction range requires roughly four times as many updates.
Theorem~\ref{thm:transport} in Appendix~\ref{app:transport-assumptions} states this result formally and also covers graphs on which walks spread faster or slower.

This general argument applies to the restart recurrence above when $a=\beta>0$: each update mixes propagated information with the original input, using weights $\beta$ and $1-\beta$.
Its typical walk length is $m\approx1/(1-\beta)$, so $\beta$ near one produces a longer-range equilibrium but requires more updates to recover it (Corollary~\ref{cor:restart-range}, Appendix~\ref{app:transport-assumptions}).

Finally, this quadratic cost is a consequence of slow spreading under nonnegative averaging, rather than a requirement of all local methods.
For example, Chebyshev polynomial filters also use local updates, but combine positive and negative contributions to approximate the same equilibrium more efficiently.
For a fixed approximation accuracy, their required update count grows nearly in proportion to the equilibrium's range (Proposition~\ref{prop:cheb-acceleration}, Appendix~\ref{app:cheb-acceleration}).

\subsection{Equivalent formulas can behave very differently}

Being able to express the right filter on paper does not mean that a model can easily learn or  execute it. Two problems can arise. First, the building blocks of a filter may look so similar on the fitting points that the fit cannot tell them apart. A tiny change in the data can then produce a large change in the fitted coefficients. 
Second, the filter response may be a small remainder after large positive and negative terms nearly cancel.
Rounding operations can then noticeably change that remainder.

\begin{theorem}[Simplified conditioning and cancellation bounds]
\label{thm:realization}
Consider a nonsingular square basis fit $\mat{B}\vect{c}=\vect{y}$ with nonzero target $\vect{y}$. Let $\kappa_2(\mat{B})$ be its condition number, let $\delta_B$ and $\delta_y$ be the relative sizes of perturbations to $\mat{B}$ and $\vect{y}$, and let $\Delta\vect{c}$ be the resulting coefficient change. Then
\begin{equation}
\frac{\norm{\Delta\vect{c}}_2}{\norm{\vect{c}}_2}
\leq
\frac{\kappa_2(\mat{B})}
{1-\kappa_2(\mat{B})\delta_B}
\left(\delta_B+\delta_y\right).
\label{eq:synthesis-bound-simple}
\end{equation}
If the $m$ already formed terms $c_jb_j(z)$ are summed sequentially in floating-point arithmetic with unit roundoff $u$, let $r(z)=\sum_{j=1}^{m}c_jb_j(z)$, let $\operatorname{fl}(r(z))$ denote the computed sum, and let $\eta>0$ be a small denominator safeguard. Then
\begin{equation}
\frac{|\operatorname{fl}(r(z))-r(z)|}
{|r(z)|+\eta}
\leq
\frac{(m-1)u}{1-(m-1)u}
\chi_\eta(z),
\qquad
\chi_\eta(z)
=
\frac{\sum_j|c_jb_j(z)|}{|r(z)|+\eta}.
\label{eq:cancellation-bound-simple}
\end{equation}
The bounds hold when their denominators are positive.
\end{theorem}

The condition number $\kappa_2(\mat{B})$ is large when the fit has difficulty distinguishing the basis functions. The cancellation factor $\chi_\eta(z)$ is large when the individual signed terms are much larger than their final sum. These are worst-case upper bounds: they identify ways in which errors can be amplified, but do not say that training or execution must fail.
In the terminology of Section~\ref{sec:framework}, high basis sensitivity can enlarge the \alignmentgap, while severe cancellation can create an \executiongap. Thus, two parameterizations can have the same architectural support and represent the same functions under exact arithmetic, yet still differ in how easily they are learned and how faithfully they run on a computer.
Theorem~\ref{thm:realization} is a simplified statement of Theorem~\ref{thm:realization-formal} in Appendix~\ref{app:precision}, which gives the exact assumptions and discusses the additional rounding errors incurred when the basis values and products are themselves computed in floating point.

\section{Related Work}

Jacobian-based influence distributions, hop-wise sensitivity profiles, range summaries, and effective receptive fields are established ways to link nodes to a GNN output~\citep{xu2018jumping,topping2022oversquashing,finder2025effective,bamberger2025range,marisca2025over,attali2026diagnostics,liang2026city}.
The quantities $p_v(r)$ and $R_q$ build on this family of measures and help diagnose the \supportlimit, the \approximationgap, the \alignmentgap, and the \executiongap{} defined in Section~\ref{sec:framework}.
The \supportlimit{} identifies underreaching if the task requires interactions beyond those the model can represent~\citep{errica2025amp}.
When the required interactions lie within this support, weak distant sensitivity may be consistent with oversquashing but does not identify its mechanism; oversmoothing and heterophily are distinct phenomena~\citep{arnaiz2025demystifying}.

Long-range designs modify paths or dynamics through rewiring and attention, adaptive depth, and rational, state-space, or non-dissipative propagation~\citep{gutteridge2023drew,rampasek2022gps,errica2025amp,eliasof2025grama,ceni2025mpssm,heilig2025phdgn, gravina2023adgn}.
PageRank and implicit GNNs instead define fixed-point or infinite-depth operators~\citep{klicpera2019predict,gu2020implicit,liu2021eignn}.
MGNNI relates the effective range of implicit GNNs to the convergence of their iterative computation~\citep{liu2022mgnni}.
We, instead, directly quantify the cost of a finite update budget: under ordinary diffusion, doubling the interaction range requires roughly four times as many restart-diffusion updates.
Because this cost is not evident from the architecture alone, we compare the profiles that models actually learn and how well they match the task.

Spectral filters use Chebyshev, monomial, Bernstein, and rational coordinates~\citep{defferrard2016convolutional,chien2021gpr,tremblay2026generalized,he2021bernnet,he2022chebnetii,klicpera2019predict,levie2019cayley,bianchi2021arma,eliasof2025grama}.
Prior related work studies approximation, stability, and transferability~\citep{levie2021transfer,hariri2025chebnet}.
We instead test conditioning, target behavior between sampled frequencies, and finite-precision reliability, which differs from formal transferability of a fixed regular filter.

Long-range benchmarks test models on different tasks and graph types.
Good overall performance alone does not establish that predictions depend on distant nodes, since local information and optimization choices can dominate the scores~\citep{dwivedi2022lrgb,tonshoff2023gap,liang2026city,mathys2026lrim,miglior2026echo}.
Further benchmark discussion is in Appendix~\ref{app:datasets}.

\section{Experiments}
\label{sec:experiments}

\subsection{Controlled filters isolate long-range failure modes}
\label{sec:controlled}

In the following we consider three complementary synthetic tasks, detailed in Appendix~\ref{app:datasets}, and discuss their results in the same order.
In \emph{Marked PathCopy}, a path contains a marked source carrying one of five class labels, a marked query at a variable distance, and three marked distractors carrying unrelated labels; the model must reproduce the source label at the query.
\emph{ResolventCopy} asks to regress the input features $\mat{X}$ to $\mat{Y}=(1-\alpha)(\Id-\alpha\mat{S})^{-1}\mat{X}$, where $\mat{S}$ is a connectivity matrix. The inverse $(\Id-\alpha\mat{S})^{-1}$ is the \emph{resolvent} of $\mat{S}$. For $\alpha\in(0,1)$, it can be written as an infinite sum, $(\Id-\alpha\mat{S})^{-1}=\sum_{k=0}^{\infty}\alpha^k\mat{S}^k$, so the target combines information from every possible number $k$ of neighbor-mixing steps, while the weight $\alpha^k$ decreases for longer propagation.
Larger $\alpha$ makes this decay slower and produces a longer-range tail, making the task harder for models with limited reach or a fixed number of propagation steps.
Finally, in \emph{RingTransfer}~\citep{bodnar2021cwn}, a cycle of $2d$ nodes contains a marked source and a marked query at the antipode, $d$ hops away; the model must reproduce the source's class label at the query. 
In RingTransfer, we test whether a filter fitted to one cycle's spectrum remains accurate on graphs with different spectra, including cycles with different numbers of nodes and paths.

Across these tasks, we make three comparisons.
First, Finite ARMA~\citep{bianchi2021arma}, which runs learned recurrent propagation for 20 steps, is compared with the Exact bank, which combines eight fixed recurrent filters evaluated at equilibrium, where further propagation updates leave the node representations unchanged.
Second, the Exact bank is compared with 20-step restart, which uses the same eight filters but stops after 20 steps; the two models learn how to combine their outputs independently.
Third, Chebyshev~\citep{defferrard2016convolutional} and Monomial use different bases for the same degree-20 polynomial filter class.
Full experimental details are in Appendix~\ref{app:controlled}.

\begin{figure}[ht]
\centering
\begin{minipage}[t]{0.297\textwidth}
\centering
\includegraphics[width=\linewidth]{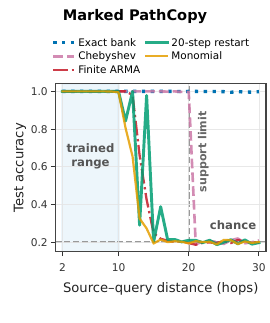}
\end{minipage}\hfill%
\begin{minipage}[t]{0.350\textwidth}
\centering
\includegraphics[width=\linewidth]{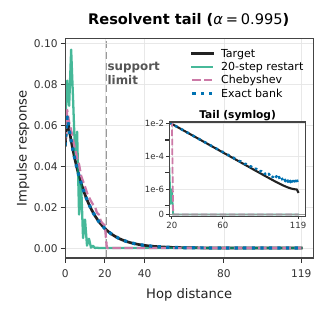}
\end{minipage}\hfill%
\begin{minipage}[t]{0.350\textwidth}
\centering
\includegraphics[width=\linewidth]{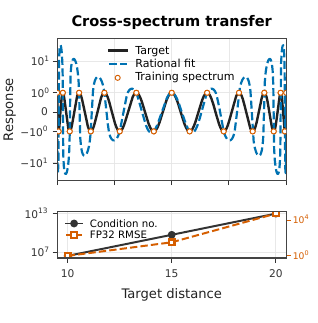}
\end{minipage}
\caption{
\textbf{Left}: Marked PathCopy accuracy; shading marks training distances and the dashed line marks the support limit at 20 hops.
\textbf{Middle}: aggregate error on ResolventCopy hides a tail that finite models miss completely. \textbf{Right}: in RingTransfer, a rational fit matches the frequencies present in the training graph but becomes inaccurate at frequencies that graph does not contain as conditioning worsens.
}
\label{fig:controlled-diagnostics}
\end{figure}

On Marked PathCopy (Figure~\ref{fig:controlled-diagnostics}-left), Finite ARMA, 20-step restart, and Monomial lose accuracy beyond the trained range despite their 20-hop reach.
Restart's accuracy peaks at even distances because its output retains a strong contribution from walks of exactly 20 steps.
On a path without self-loops, these walks can end only at even distances from the source (Appendix~\ref{app:pathcopy-parity}).
Chebyshev stays accurate through 20 hops and falls to chance beyond its \supportlimit; the Exact bank stays accurate through 30 hops.
Source-class and distractor perturbations, together with separate query-relocation tests, confirm that the Exact bank uses the marked source (Appendix~\ref{app:controlled}).
Chebyshev and Monomial span the same degree-20 filters, so their difference reflects ease of fitting rather than reach or expressivity.

On ResolventCopy, Figure~\ref{fig:controlled-diagnostics}-middle shows how a signal placed at one node affects nodes farther away for the long-range $\alpha=0.995$ target. 
Before the dotted 20-hop boundary, Chebyshev closely follows the black target curve, while 20-step restart already oscillates away from it. 
Beyond the \supportlimit, both finite filters become exactly zero and discard the distant tail.
The inset enlarges this region: only the Exact bank continues alongside the target.
This task reveals two different failures. 
Chebyshev captures the main response but loses all influence beyond its support, whereas 20-step restart approximates the equilibrium poorly even at shorter distances despite using the same recurrence coefficients of the Exact bank.
This \approximationgap{} illustrates a limitation of stationary restart diffusion: convergence is slow when each update reduces the remaining error only slightly, as quantified by Corollary~\ref{cor:restart}.
Corollary~\ref{cor:restart-range} explains why restart diffusion needs many updates to recover the target's range: on these graphs, its realized range grows only as the square root of the update count.
The Chebyshev realization can approximate the same target faster, as described in Appendix~\ref{app:cheb-acceleration}.
Because the distant tail contributes only a small share of the total signal, losing it entirely may increase the average error only slightly.
The tail-normalized error $\Etail$ is therefore needed to show whether remote influence survives.
Notably, when we leave the Exact bank unchanged but stop it after 20 updates, its measured influence disappears beyond 20 hops (Appendix~\ref{app:controlled-replay}).

Finally, the RingTransfer results in Figure~\ref{fig:controlled-diagnostics}-right show how ill-conditioned fitting can create both an \alignmentgap{} and an \executiongap.
The orange circles mark the frequencies of the ring used for training.
The blue rational fit passes through those points and therefore looks successful on that graph, but it oscillates far from the black target curve between them.
A different graph topology can contain frequencies in these gaps and expose the error.
Let FP32 error be the root-mean-square difference between the target and the fitted response when that response is evaluated in 32-bit floating-point arithmetic across a dense grid of graph frequencies.
The lower plot shows that, as the target distance grows, both the conditioning problem and the FP32 error rise sharply, leading to poor transfer to new graph spectra.
Theorem~\ref{thm:realization} explains these trends: ill-conditioning makes the RingTransfer fit sensitive to small perturbations during fitting, while cancellation amplifies finite-precision errors during execution. 
We also observe that transfer accuracy also depends on the target: the same rational family accurately fits the smooth resolvent response (Appendix~\ref{app:oracle}).

\subsection{Support, range, and alignment across GNN families}
\label{sec:gnn-families}

We compare GNN families by their \supportlimit, their realized interaction profiles and ranges, and how closely these profiles match the task's requirements (the \alignmentgap).
We compare GatedGCN~\citep{bresson2018gatedgcn}, Stable-ChebNet~\citep{hariri2025chebnet}, AMP~\citep{errica2025amp}, GraphGPS~\citep{rampasek2022gps}, MP-SSM~\citep{ceni2025mpssm}, Anti-Symmetric DGN (A-DGN)~\citep{gravina2023adgn}, Tikhonov GNN~\citep{tremblay2026generalized}, SONAR~\citep{trenta2025sonar}, and GRIT~\citep{ma2023grit}; a graph-agnostic MLP provides a no-propagation control.
Complete architecture and training details are in Appendix~\ref{app:models}.
The models share a pointwise encoder and prediction head.
GatedGCN, Stable-ChebNet, AMP, MP-SSM, A-DGN, and SONAR have a maximum spatial support of 20 hops through local steps, polynomial degree, or recurrent updates.
GraphGPS, GRIT, and the Tikhonov inverse operator instead have whole-graph support, whereas the MLP cannot combine information among nodes.

\begin{figure}[h]
\centering
\includegraphics[width=\textwidth]{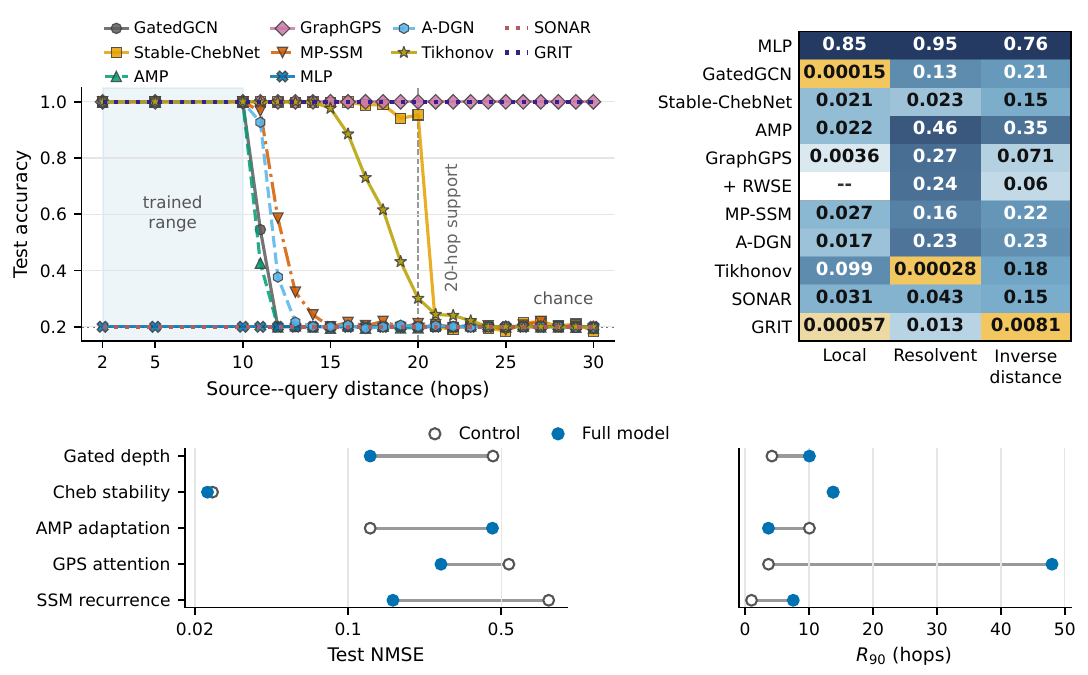}
\caption{
    \textbf{Top-left}: copying accuracy on the test set.
    Lines average five runs; shading marks the trained distances 2--10.
    The horizontal line marks five-class chance and the vertical line marks the local models' 20-hop support.
    \textbf{Top-right}: RangeProfileCopy test NMSE.
    \textbf{Bottom}: differences between each model and its matched control in prediction error and realized range $R_{90}$ on the long-range resolvent. Open markers denote matched controls, filled blue markers denote models with the mechanism being tested, and error bars show sample standard deviations across three runs.}
\label{fig:architecture-summary}
\end{figure}

We consider two complementary tasks.
First, we reuse Marked PathCopy with a different aim than in Section~\ref{sec:controlled}: the latter isolates failure mechanisms using controlled filters; here, we compare how GNNs generalize beyond the trained distances.
We train at distances 2--10 and verify success on training and new test graphs throughout this range before interpreting performance at longer distances.
Across all five runs, every GNN except SONAR achieves 100\% training and at least 99.6\% test accuracy at each trained distance; SONAR and the MLP remain at chance-level (20\% test accuracy).
Their outcomes at longer distances nevertheless differ sharply (Figure~\ref{fig:architecture-summary}, upper-left).
GraphGPS and GRIT retain 100\% accuracy through 40 hops.
Stable-ChebNet retains 95\% at distance 20 but drops to 20\% at distance 21, just beyond its 20-hop \supportlimit{} (Proposition~\ref{prop:support}).
GatedGCN, AMP, MP-SSM, and A-DGN fall to about 20\% at distance 15, inside their nominal 20-hop support.
Tikhonov reaches 97\% at distance 15 but only 30\% at distance 20 and is near chance by distance 24, despite its global support.
Besides SONAR and the MLP, these failures occur beyond the trained distances, after successful fitting and generalization within them.
To determine whether long-distance failures reflect loss of the source signal, reliance on distractors, or simply the larger test graphs, we run counterfactual tests that alter the source label, distractors, and source/query positions while keeping the trained checkpoint and graph topology fixed. 
These tests (Appendix~\ref{app:competence-diagnostics}) show that, where long-distance accuracy drops, predictions no longer reliably match the new source label when it is changed.
Moving the query to a node 2--10 hops from the source makes the GNNs that learned the training task predict the source label correctly, without changing the graph's size or connections.

Next, we consider an additional task called \emph{RangeProfileCopy}, which tests whether a model can aggregate information distributed across many nodes.
Each node in a path- or ladder-graph receives a random input, and the model must predict, at every node, a weighted average of inputs from across the graph.
We use three rules: ``Local'' averages inputs within two hops, whereas ``Resolvent'' and ``Inverse Distance'' also give weight to distant inputs.
Together, they reveal whether the model combines information over the distances each rule requires.
The upper-right heat map in Figure~\ref{fig:architecture-summary} compares performance under the three RangeProfileCopy averaging rules.
GatedGCN performs best for the ``Local'' two-hop average, Tikhonov for the smooth global ``Resolvent'', and GRIT for the ``Inverse Distance'' rule.
Success therefore depends on matching the task's distance weighting, not simply on reaching farther.
Additional results and numerical checks are in Appendices~\ref{app:architecture-results} and~\ref{app:tikhonov-replay}.

Figure~\ref{fig:architecture-summary}-bottom reports the individual propagation mechanisms using the long-range ``Resolvent'' rule.
In this task, every output combines inputs from across the graph, with influence that decays slowly enough for distant nodes to matter.
Each comparison pairs a model with the mechanism of interest to a closely matched control that removes or alters it.
The bottom-left plot shows prediction error.
The bottom-right summarizes each realized interaction profile by its range $R_{90}$ (Equation~\eqref{eq:Rq}).
Moving right therefore means that a larger radius is needed to accumulate the same fraction of sensitivity.
The five control-to-model changes are: GatedGCN from 5 to 20 propagation steps; vanilla degree-20 ChebNet to Stable-ChebNet's parameterization; fixed 20-step GatedGCN to AMP's learned depth; GraphGPS with local message passing only to GraphGPS with also global self-attention; and one-step MP-SSM to its 20-step recurrence.
Increasing GatedGCN depth and adding MP-SSM recurrence improve prediction and extend realized range.
Global attention greatly extends GraphGPS's range and lowers its prediction error, while AMP has higher error and a shorter range than its fixed-depth control.
The two ChebNet variants nearly overlap, indicating little effect from the stabilization mechanism.
Table~\ref{tab:e2-architectural-contrasts} reports the numerical comparisons and agreement with the target profile used to assess the \alignmentgap.
Repeating inference with the same trained parameters in BF16 instead of FP64 increases prediction error for Stable-ChebNet and MP-SSM while barely changing their measured range (Table~\ref{tab:e2-precision}).
This illustrates the \executiongap: numerical errors can degrade predictions even when the measured range of input influence remains nearly unchanged.
Checking $R_{90}$ alone would miss this loss of accuracy.

\subsection{Is remote context useful for prediction?}
\label{sec:remote-pairs}

\begin{figure}[ht]
\centering
\includegraphics[width=.97\textwidth]{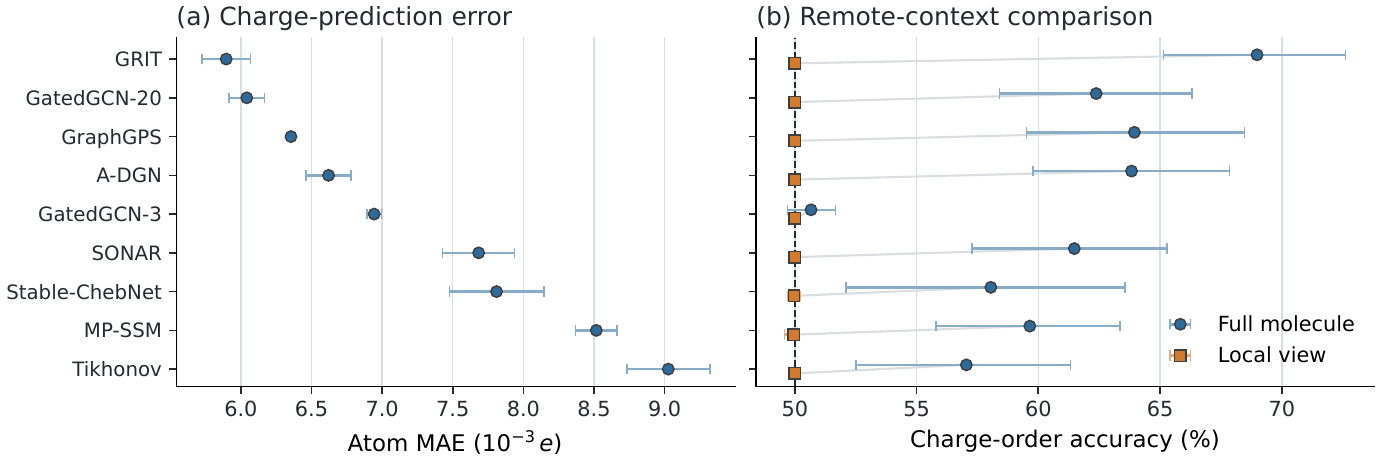}
\caption{
    \textbf{Left}: mean test error across runs; lower is better.
    \textbf{Right}: how often each model correctly identifies which atom in a matched pair has the larger charge when it sees the complete molecules (blue) or only the matched local views (orange).
    A blue--orange gap means that the rest of the molecule helps the model.
    The dashed line marks the $0.5$ local-only baseline.
}
\label{fig:e3-molecular}
\end{figure}

ECHO-Charge provides molecules together with a reference partial charge for every atom~\citep{miglior2026echo}.
We train different GNNs to predict these charges from atoms and bonds, omitting each atom's distance from the molecular center of mass because it encodes global position.
The left panel of Figure~\ref{fig:e3-molecular} reports their usual average error over 8,519 held-out molecules.
This tells us which model predicts charges well overall, but not whether it uses information from atoms far away.
Since the target distance profile is unknown, we cannot measure the \alignmentgap{} by comparing profiles.
We instead use a paired test, \emph{RemotePairs}, to check whether distant information improves the model's predictions.
Its goal is to identify which atom in each pair has the larger charge.
We first find atoms in different molecules whose chemical surroundings within 2, 3, or 4 bond hops are indistinguishable to the models, even though the molecules differ beyond the matched radius.
For every such pair, we give the same already-trained model two versions of the problem.
In the first, it sees both complete molecules.
In the second, we cut away everything outside the matched surroundings and show it only the two identical local views.
If the model is more accurate with the complete molecules than with the identical local views, the removed context was useful; if its accuracy is unchanged, it was not.

The right panel of Figure~\ref{fig:e3-molecular} shows this comparison, averaging the three radii equally.
Eight models are reliably better when they can see the complete molecules.
GRIT reaches 69\% matched-pair accuracy with full context versus 50\% after cropping, the largest mean improvement in this comparison.
The three-layer GatedGCN illustrates the \supportlimit{} in Proposition~\ref{prop:support}: matching through three or four hops covers its entire support, so additional context cannot affect its predictions.
As expected, the two views then give the same predictions and the score returns to the $0.5$ baseline.
Appendix~\ref{app:echo-remotepairs} gives the exact matching, cropping, scoring, and uncertainty procedure; Appendix~\ref{app:molecular} gives the training details and complete numerical results.

\section{Conclusions}
\label{sec:conclusion}

This study challenges the common assumption that a GNN is long-range simply because its architecture can connect distant nodes.
Our framework separates architectural support, finite approximation, task alignment, and execution, revealing failures to use distant information correctly even when overall prediction error is low.
Across controlled and molecular tasks, models with similar architectural support differ in their realized interactions: which distances affect predictions, whether these influences match the task, and whether practical inference preserves them.
The results point toward a conceptually simple but consequential standard: long-range claims should report where interaction is possible, where it occurs, whether it matches the task, and whether execution preserves it.
The goal is not maximal reach, but the right computation at the right distance.

\section*{Acknowledgments}
This work is supported by the Research Council of Norway through \textit{RELAY:~Relational Deep Learning for Energy Analytics} (project no. 345017).
The author wishes to thank NVIDIA Corporation for donating the GPUs used in this project.

\bibliographystyle{unsrtnat}
\bibliography{biblio}

\appendix
\section{Measurement Definitions and Edge Cases}
\label{app:range}

This appendix makes the measurements introduced in Section~\ref{sec:framework} precise and specifies their edge cases.
In the experiments, task-valid perturbations test source use on Marked PathCopy (Section~\ref{sec:controlled} and Appendix~\ref{app:competence-diagnostics}), while gradient-based distance profiles, range quantiles, and task-alignment metrics characterize the RangeProfileCopy architectural contrasts (Section~\ref{sec:gnn-families} and Appendix~\ref{app:mechanism-precision}).
Interaction blocks and tail error measure whether models recover the distant part of the known target map in ResolventCopy and RangeProfileCopy (Appendices~\ref{app:controlled} and~\ref{app:rangeprofilecopy}).
Output drift and range retention quantify the effects of numerical precision when replaying fixed checkpoints (Appendix~\ref{app:mechanism-precision}).

In the theory, range and tail error express the consequences of the support bound in Proposition~\ref{prop:support} (Section~\ref{sec:theory}).
The operator-induced profiles and their total-variation distance are used in the positive-diffusion range bounds and proofs in Appendices~\ref{app:transport-assumptions}--\ref{app:path-restart}, including Theorem~\ref{thm:transport} and its restart-diffusion specializations.
The conventions below specify how these measurements handle zero influence, zero target-tail energy, and disconnected graphs.

\subsection{Measurement definitions}

\paragraph{Graphs and outputs.}
Let $\set{G}=(\set{V},\set{E})$ be an undirected graph with node set $\set{V}$ and edge set $\set{E}$.
Let $\mat{X}\in\RR^{|\set{V}|\times F}$ be its node-feature matrix, where $F$ is the number of features per node and the $u$\textsuperscript{th} row $\vect{x}_u\in\RR^F$ contains the features of node $u$.
For nodes $u,v\in\set{V}$, let $d(u,v)$ be their shortest-path distance.
We let $d(u,v)=\infty$ when $u$ and $v$ lie in different connected components.
At each node $v$, we reduce the model output to a scalar score $s_v$: for classification, the correct-class logit minus the largest competing logit; for regression, a selected output channel.

\paragraph{Derivative and perturbation influence.}
For continuous inputs, we measure the influence of node $u$ on the score at node $v$ by the size of the corresponding gradient:
\[
a^{\mathrm{grad}}_v(u)=\norm{\frac{\partial s_v(\mat{X})}{\partial \vect{x}_u}}_2.
\]
The Euclidean norm combines the score's sensitivity to the $F$ features of node $u$ into one nonnegative value.
Discrete inputs cannot be varied infinitesimally, so we instead make a valid change to node $u$ and measure the resulting change in the score at node $v$.
The allowed changes depend on the task: they may substitute another valid class, swap node roles, alter a distractor value, or replace a feature with a conditional baseline.
Let $\set{Q}_u$ be a finite probability distribution over the changes allowed by the task, and let $\boldsymbol{\Delta}_u$ be a perturbation matrix sampled from $\set{Q}_u$ and added to $\mat{X}$.
We define the counterfactual, or finite-change, influence as the average absolute score change:
\[
a^{\mathrm{cf}}_v(u)=\mathbb{E}_{\boldsymbol{\Delta}_u\sim \set{Q}_u}\lvert s_v(\mat{X}+\boldsymbol{\Delta}_u)-s_v(\mat{X})\rvert.
\]
Here, $\mathbb{E}$ denotes the average over perturbations drawn from $\set{Q}_u$.

\paragraph{Realized interaction profile, range, and task alignment.}
Let $a_v(u)$ denote the appropriate influence measure for the input type: $a^{\mathrm{grad}}_v(u)$ for continuous inputs or $a^{\mathrm{cf}}_v(u)$ for discrete inputs.
Its sum over all input nodes, $A_v=\sum_{u\in\set{V}}a_v(u)$, is the total measured influence on the score at node $v$.
When $A_v>0$, the normalized node influence $p_v(u)=a_v(u)/A_v$ is the fraction of this total attributable to node $u$.
For a finite graph distance $r\in\{0,1,2,\ldots\}$, the distance profile
\[
p_v(r)=\sum_{u:\,d(u,v)=r}p_v(u)
\]
is the fraction of measured influence contributed by nodes exactly $r$ hops from $v$; it describes the model's realized interaction profile at the evaluated input and settings.
For $q\in(0,1)$, its realized range $R_q(v)$ in Equation~\eqref{eq:Rq} is the smallest radius containing at least a fraction $q$ of the total influence.
When $A_v=0$, no input node has measured influence on the score, so normalized-profile summaries are undefined and the case is recorded as zero influence.

We also define the expected influence distance $\sum_{r=0}^{\infty}r p_v(r)$, the tail mass $\sum_{r>r_{\mathrm{task}}}p_v(r)$ beyond a task-defined radius $r_{\mathrm{task}}$, and the complete distance profile.
When the task specifies a target profile $p^\star(r)$, which describes how much influence should occur at each distance $r$, we compare it with $p_v(r)$ using cosine similarity and the diameter-normalized Wasserstein distance
\[
W_{1,\mathrm{norm}}(p_v,p^\star)
=\frac{W_1(p_v,p^\star)}{D_v},
\]
where $W_1$ is the one-dimensional Wasserstein distance between the two distance profiles and $D_v$ is the diameter of the connected component containing $v$.
Lower normalized Wasserstein distance and higher cosine similarity indicate better alignment; both are accompanied by the profiles themselves.
If no finite radius contains a fraction $q$ of the normalized influence mass, we set $R_q(v)=\infty$. This convention covers disconnected graphs.

\paragraph{Operator-induced profiles.}
Theorem~\ref{thm:transport} in Appendix~\ref{app:transport-assumptions} relates walk depth to spatial range using the endpoint distribution of a nonnegative row-stochastic operator.
For such an operator $\mat{K}$, define
\[
p_{\mat{K}}^{\,v}(r)
=
\sum_{u:\,d(u,v)=r}K_{vu},
\qquad
R_q(\mat{K};v)
=
\min\left\{r:\sum_{s=0}^{r}p_{\mat{K}}^{\,v}(s)\geq q\right\}.
\]
This is the distance distribution of an endpoint sampled from row $v$ of $\mat{K}$. For a scalar positive row-normalized linear map, its entries are also the nonnegative input derivatives, so this definition agrees with Equation~\eqref{eq:Rq}. General signed or multichannel operators continue to use the influence definitions in Appendix~\ref{app:range} and are not automatically covered by Theorem~\ref{thm:transport}. For two distance profiles $p$ and $p'$, total variation means $\norm{p-p'}_{\mathrm{TV}}=\frac12\sum_{r\geq0}|p(r)-p'(r)|$.

\paragraph{Interaction blocks and tail errors.}
For a differentiable model with vector output $\widehat{\vect{y}}_v$ at node $v$, the complete input--output block realized at the evaluated input is
\[
\widehat{\mat{K}}_{vu}
=
\frac{\partial \widehat{\vect{y}}_v}{\partial \vect{x}_u}.
\]
In a nonlinear model, the block must be evaluated at a specific input because the model's sensitivity may vary from one input to another.
In a fixed linear model, the block is constant across inputs and equals the corresponding block of its linear operator.
On controlled tasks with a known linear target, $\mat{K}_{vu}$ denotes the corresponding block of the exact target operator.
For radius $r_0$, let $\set{T}_v(r_0)=\{u\in\set{V}:d(u,v)>r_0\}$.
With $\norm{\cdot}_F$ denoting the Frobenius norm, the normalized error for a target with positive energy in this tail is

\begin{equation}
\Etail(v;r_0)
=
\frac{
\sum_{u\in\set{T}_v(r_0)}
\norm{\widehat{\mat{K}}_{vu}-\mat{K}_{vu}}_F^2
}{
\sum_{u\in\set{T}_v(r_0)}
\norm{\mat{K}_{vu}}_F^2
}.
\label{eq:tail}
\end{equation}

If the exact target assigns no influence beyond $r_0$, the denominator in Equation~\eqref{eq:tail} is zero and the normalized error cannot be computed.
We then compute the absolute tail error $\sum_{u\in\set{T}_v(r_0)}\norm{\widehat{\mat{K}}_{vu}}_F^2$, which measures the spurious influence that the model assigns beyond $r_0$.
When both operators depend only on graph distance, all nodes at the same distance share the same interaction block. Equation~\eqref{eq:tail} then compares the blocks at each distance, weighted by the number of nodes at that distance.

\paragraph{Output drift and range retention.}
Given a reference implementation $f^{\mathrm{ref}}$ and a practical implementation $f^{\mathrm{impl}}$, let $R_q^{\mathrm{ref}}$ and $R_q^{\mathrm{impl}}$ be their respective realized ranges and let $\epsnum>0$ be a small numerical safeguard.
Using the Euclidean norm after vectorizing the outputs, we define the relative output drift and range retention as
\[
D_{\mathrm{out}}=
\frac{\norm{f^{\mathrm{impl}}(\mat{X})-f^{\mathrm{ref}}(\mat{X})}}
{\max\!\left\{\norm{f^{\mathrm{ref}}(\mat{X})},\epsnum\right\}},
\qquad
\operatorname{Ret}_q=\frac{R_q^{\mathrm{impl}}}{R_q^{\mathrm{ref}}}.
\]
Range retention equals 1 when the ranges agree; values below or above 1 indicate a shorter or longer range under practical execution, respectively.
These metrics compare implementations that share trained parameters or form a specified ideal/approximate pair. Range retention is reported when both ranges are finite and $R_q^{\mathrm{ref}}>0$; all other cases are recorded as undefined.



\section{Datasets and Tasks}
\label{app:datasets}

Long-range benchmarks differ in their tasks and in how much distant information those tasks require.
LRGB includes tasks with different degrees of dependence on remote information; its initial evaluation reported that graph Transformers outperformed message-passing GNN baselines~\citep{dwivedi2022lrgb}.
Subsequent hyperparameter tuning substantially improved GCN, GINE, and GatedGCN, closing their performance gap to the GPS graph Transformer on both peptide datasets~\citep{tonshoff2023gap}.
On Peptides-Struct, for example, replacing the baselines' linear prediction head with a two-layer MLP accounted for most of the improvement, showing that the reported gap depended strongly on the baseline configuration~\citep{tonshoff2023gap}.
ECHO covers non-local algorithmic and molecular tasks, LRIM provides known power-law dependencies and size extrapolation, and City-Networks uses large road graphs with prescribed dependency radii~\citep{miglior2026echo,mathys2026lrim,liang2026city}.

Our experiments ask where a model loses the distant information required by a task: through limited support, finite approximation, training, or numerical execution.
Repeating standard benchmark evaluations would not by itself separate these causes, since the same prediction error can arise from different failures and a low overall error can hide a missing distant contribution.
We therefore use tasks with explicit reference computations: Marked PathCopy identifies the exact source and its distance from the query, while ResolventCopy and RangeProfileCopy specify the contribution required from every input node (Appendices~\ref{app:pathcopy}--\ref{app:rangeprofilecopy}).
Their small paths and ladders make complete interaction profiles and exact reference solutions tractable, allowing us to measure errors by distance and compare finite and reference implementations with trained parameters held fixed.

For real-data validation, we reuse ECHO-Charge with additional controls (Appendix~\ref{app:echo-remotepairs}).
Average charge-prediction error alone does not establish whether remote atoms help a model.
We remove the center-of-mass-distance feature, which directly supplies global information, and construct RemotePairs with identical model-visible local neighborhoods but different surrounding molecules.
Comparing full-molecule and radius-cropped predictions then tests whether access to the remote context improves prediction.
This section specifies the tasks, adaptations, and data splits used for these comparisons.
Figure~\ref{fig:task-graph-examples} shows representative graphs, and Table~\ref{tab:datasets} summarizes the role of each task.

\begin{figure}[!ht]
\centering
\includegraphics[width=\textwidth]{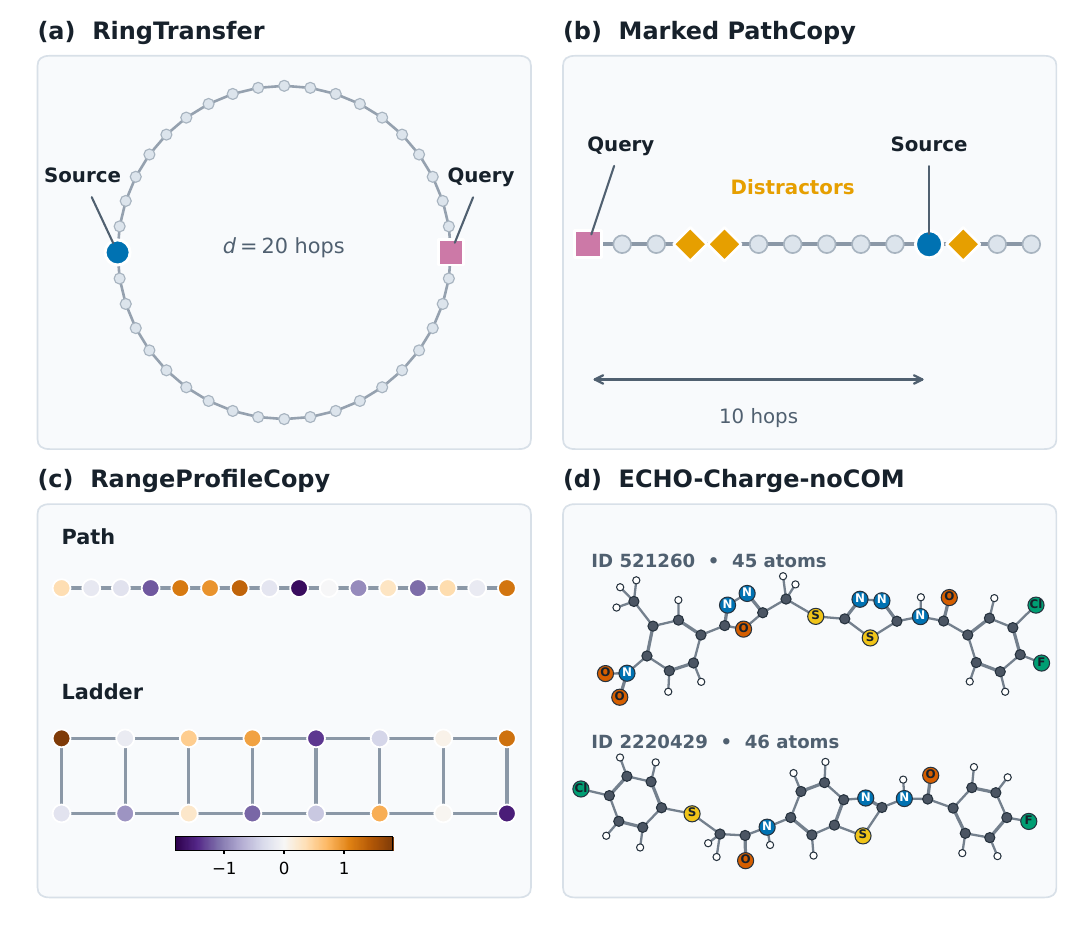}
\caption{Representative graphs from each task.
(a) The distance-20 RingTransfer cycle, with the antipodal source and query distinguished from ordinary nodes.
(b) A representative Marked PathCopy sample showing the source, query, three distractors, and padding nodes at their sampled positions.
(c) The path and ladder topologies used for RangeProfileCopy; node color shows the observed first input channel.
(d) Two molecules from the validation split of the ECHO-Charge dataset.
Element color identifies atomic number, bond width identifies bond type, and the layout uses observed bond lengths; every atom is a charge-prediction target.}
\label{fig:task-graph-examples}
\end{figure}

\begin{table}[h]
\centering
\small
\setlength{\tabcolsep}{3pt}
\caption{Tasks and the questions they address.}
\label{tab:datasets}
\begin{tabular}{@{}>{\raggedright\arraybackslash}p{0.28\textwidth}>{\raggedright\arraybackslash}p{0.69\textwidth}@{}}
\toprule
Task or test & Main question \\
\midrule
RingTransfer & How does a filter's mathematical form affect fitting and numerical accuracy? \\
\addlinespace
Marked PathCopy & Can the model copy a source label from increasingly distant nodes? \\
\addlinespace
ResolventCopy & Does the model recover weak distant effects that average error can hide? \\
\addlinespace
RangeProfileCopy & Does the model weight nearby and distant inputs as the task requires? \\
\addlinespace
ECHO-Charge-noCOM & How accurately does the model predict atomic charges in real molecules? \\
\addlinespace
RemotePairs & Does the rest of the molecule improve predictions when local surroundings are identical? \\
\bottomrule
\end{tabular}
\end{table}

\subsection{RingTransfer}
\label{app:ringtransfer}

RingTransfer~\citep{bodnar2021cwn} uses a cycle with a marked source and its antipodal query, as shown in Figure~\ref{fig:task-graph-examples}(a). The source stores one of five classes, and the query must predict it.
For a cycle with $2d$ nodes, source and query are $d$ hops apart.
We use source--query distances $d\in\{10,15,20\}$ to test whether a filter fitted on a cycle remains accurate on other graphs and when evaluated at lower numerical precision (Appendix~\ref{app:oracle}).

\subsection{Marked PathCopy}
\label{app:pathcopy}

Marked PathCopy uses a bidirectional path with one source marker, one query marker, five class channels, exactly three distractors, random source--query direction, and 3--10 padding nodes; Figure~\ref{fig:task-graph-examples}(b) shows a representative example. The query must predict the class stored at the source. No node identifiers or positional encodings are provided.

The experiments reported in Sections~\ref{sec:controlled} (controlled-filter study) and~\ref{sec:gnn-families} (GNN-family comparison) use the same task with separate protocols.
Both studies train on 5,000 graphs approximately balanced over source--query distances 2--10, using separately generated datasets.

In the controlled-filter study, we train each method in three independent runs and select one checkpoint per run by validation loss on 500 graphs at distance 12.
We evaluate test accuracy on 500 class-balanced graphs at every integer distance 2--30, generated independently of the training and validation graphs.
We report means across runs and estimate 95\% intervals from 2,000 paired bootstrap resamples of runs and test graphs, conditional on the selected checkpoints.
To assess source use, we substitute the source class and exhaustively vary the distractor classes at distance 20. Separate source/query role swaps and query-relocation tests probe dependencies beyond 20 hops.

The GNN-family comparison uses 250 independent, class-balanced validation graphs at every distance 2--10 and 500 independent, class-balanced test graphs at every distance 2--20, 30, and 40. The validation and test sets contain 50 and 100 graphs per class and distance, respectively.
Within each study, all models and repeated runs use the same validation and test graphs, allowing direct comparison on identical examples.
Before interpreting performance beyond 10 hops, we check that the model already copies labels reliably at every distance 2--10 on both training graphs and new test graphs (Appendix~\ref{app:pathcopy-competence}).
This separates failure at longer distances from failure to learn the original task.
The additional tests at distances 11--14 and 16--19 use already trained and selected models; their results are not used for tuning, model selection, or the checks at distances 2--10.

\subsection{ResolventCopy}
\label{app:resolventcopy}

ResolventCopy uses bidirectional paths with four-dimensional Gaussian node inputs $\mat{X}$. With $\mat{S}$ denoting the symmetric neighbor-only normalized adjacency, every node is supervised against
\begin{equation}
\mat{Y}=(1-\alpha)(\Id-\alpha\mat{S})^{-1}\mat{X},
\qquad \alpha\in\{0.95,0.995\}.
\end{equation}
Targets are computed in FP64 and stored in FP32.
Training paths contain 10--30 nodes, validation paths contain 40 or 60 nodes, and test paths contain 80 or 120 nodes.
Impulse tests record the complete response kernel for endpoint and center impulses on paths with $N\in\{80,120\}$ and report errors separately at distances at most and greater than 20.

\subsection{RangeProfileCopy}
\label{app:rangeprofilecopy}

RangeProfileCopy uses canonical bidirectional paths and two-rail ladders with $m$ nodes on each rail, shown in Figure~\ref{fig:task-graph-examples}(c); $N=|\set{V}|$ denotes the total number of graph nodes.
Every node receives $\vect{x}_v\in\RR^4$ sampled from a standard normal distribution.
Inputs are generated in FP32, promoted to FP64 for target construction, and stored with FP32 targets.
The same base graphs and inputs are shared across target profiles and models.

Training uses 5,000 graphs over $N\in\{16,24,32,40\}$, with 625 graphs in each topology--size cell.
Validation uses 500 graphs at each of $N\in\{48,64\}$, and test uses 500 graphs at each of $N\in\{80,120\}$; every validation and test size contains 250 paths and 250 ladders.

Let $\mat{Y}$ collect the node-wise target vectors $\vect{y}_v$ as rows. The three target profiles are
\begin{align}
\text{local:}\quad
\vect{y}_v &= \frac{\sum_{u:d(u,v)\leq2}\vect{x}_u}{|\{u:d(u,v)\leq2\}|},\\
\text{resolvent:}\quad
\mat{Y} &= (1-\alpha)(\Id-\alpha\mat{S})^{-1}\mat{X},\qquad \alpha=0.995,\\
\text{inverse distance:}\quad
\vect{y}_v &= \frac{\sum_u(1+d(u,v))^{-1}\vect{x}_u}{\sum_u(1+d(u,v))^{-1}},
\end{align}
where $\mat{A}$ is the neighbor-only adjacency matrix, $\mat{D}$ is its degree matrix, and $\mat{S}=\mat{D}^{-1/2}\mat{A}\mat{D}^{-1/2}$ is the symmetric normalized adjacency. Resolvent residuals are checked in FP64 before caching.

For each of the three RangeProfileCopy rules above, the model predicts the four target values at every node.
Training minimizes their squared prediction error: we sum over the four channels, average over nodes within each graph, and then average over graphs.
This gives each graph equal weight, regardless of its size.

For evaluation, the normalized mean squared error (NMSE) reported in Table~\ref{tab:e1-range-predictive} measures prediction error relative to the size of the target values.
For each graph, we divide the total squared prediction error by the sum of squared target values, then average these ratios over graphs.
We also record mean absolute error (MAE), averaging absolute errors over channels, nodes, and then graphs, and relative $L_2$ error, averaging the square root of each graph's normalized squared error.
Graphs whose sum of squared target values is at most $\epsnum=10^{-12}$ are counted and excluded from the two normalized metrics.
We also report these metrics separately for paths and ladders at each graph size to show how performance depends on graph structure and node count.

For the hard-resolvent architectural contrasts, influence analysis uses all 250 graphs in every topology--size test cell and evaluates one central and one near-boundary output on each path and ladder. The complete $4\times4$ Jacobian block $\mat{J}_{vu}$, mapping the four input channels at $u$ to the four output channels at $v$, is computed for every input node. Let $K_{vu}$ be the scalar target-kernel entry and $\mat{I}_4$ the $4\times4$ identity matrix. Frobenius norms form influence mass, while strict tail error compares $\mat{J}_{vu}$ directly with $K_{vu}\mat{I}_4$. We report raw and normalized profiles, $R_{50}$, $R_{90}$, mass beyond 20, cosine similarity, diameter-normalized one-dimensional Wasserstein distance, and normalized and absolute tail errors. When total influence or target-tail energy is at most $\epsnum$, the corresponding normalized metrics are undefined and flagged.

\subsection{ECHO-Charge-noCOM and RemotePairs}
\label{app:echo-remotepairs}

ECHO-Charge-noCOM adapts the ECHO-Charge task~\citep{miglior2026echo}, using its official split of 153,330 training, 8,518 validation, and 8,519 test molecules; representative molecular graphs appear in Figure~\ref{fig:task-graph-examples}(d). The task predicts the partial charge of every atom. We remove the center-of-mass-distance node feature because it directly encodes whole-molecule information. Atomic number and bond type remain categorical. Each raw bond length is first rounded to four decimal places with round-half-up, assigned to a half-open $0.05$~\AA{} bin, and replaced by that bin's midpoint. This midpoint is the continuous bond input to every model; the unrounded length is stored but never used as model input. Laplacian and random-walk positional encodings, virtual nodes, and global molecular descriptors are disabled.

Models are trained only on this quantized ECHO-Charge-noCOM dataset. We evaluate the same validation-selected checkpoints, without retraining, on the full test set and with RemotePairs. Full-test MAE averages absolute error over all test atoms, counting each atom once. Because the center-of-mass-distance feature is omitted and bond lengths are quantized, these values form a controlled internal comparison rather than standard ECHO-Charge leaderboard results.

RemotePairs uses a matching radius $r\in\{2,3,4\}$, measured in bond steps. Rooted neighborhoods must agree exactly in topology, atomic numbers, bond types, and bond-length bin IDs. Each boundary node additionally carries its full-molecule degree and the multiset of labels on all incident edges, including edges leaving the crop. This boundary signature retains the information needed by degree-normalized propagation without matching the rest of the molecule.
To find possible matches quickly, we first group local neighborhoods using a compact summary of their atoms and bonds (a Weisfeiler--Leman hash).
We then check each candidate pair with an exact graph-matching algorithm (VF2).
A pair is accepted only if the atoms in the two neighborhoods can be matched one to one, with the two central atoms corresponding and all connections and the atom and bond information listed above preserved.
Pairs come from different molecules, match through radius $r$, and have different context beyond the match.
We first find matching atom pairs and select 1,000 of them without using their reference charges.
A fixed, repeatable rule assigns the two atoms in each pair to the left and right positions, also without using their charges.
Only afterward do we use the reference charges to keep pairs with a large enough charge difference, using a threshold chosen beforehand from validation data (see below).
The reference charges also tell us which atom has the larger charge, so we can check whether the model's prediction is correct.

One label-blind common selection of 1,000 test pairs is frozen for every model and seed: 334 pairs at radius two and 333 at each of radii three and four. The charge-difference eligibility threshold, $0.009304$ elementary charge, is the 75th percentile fixed from validation pairs. It leaves 116, 69, and 55 test pairs at radii two, three, and four, respectively. Among these 240 pairs, the left endpoint has the larger charge in 130 and the right endpoint in 110. The validation-frozen majority orientation is left-higher; its validation accuracy is $0.504$, and its equal-radius test accuracy is $0.544$. The exact-local descriptor baseline predicts a tie and scores $0.5$.

For each selected pair, we first predict on both complete molecules. We then replace each molecule by the induced radius-$r$ ego graph, preserving all quantized node and edge inputs, confirm that the two rooted crops are isomorphic, and evaluate the same checkpoint again. Full and cropped conditions use the same validation-threshold-eligible pair IDs. Predicted charge differences with magnitude at most $10^{-6}$ are ties and receive score $0.5$; otherwise charge-order accuracy records whether the predicted sign matches the reference sign.
Each represented radius receives equal weight. Paired hierarchical intervals use 2,000 replicates that jointly resample the different training seeds and connected components of the molecule-pair graph while preserving model and full/crop pairing. We claim remote-context use only when the 95\% interval for full-minus-crop accuracy is strictly positive and full-molecule balanced accuracy exceeds $0.5$.

\section{Controlled Diagnostic Study}
\label{app:controlled}

This section provides the full results and methods for the controlled diagnostic study presented in Section~\ref{sec:controlled}. It covers trained-filter comparisons on Marked PathCopy and ResolventCopy and the separate RingTransfer cross-spectrum diagnostic defined in Appendix~\ref{app:ringtransfer}.
Table~\ref{tab:diagnostic-models} summarizes the five controlled filters.

\subsection{Controlled filter models}

\begin{table}[h]
\centering
\small
\setlength{\tabcolsep}{4pt}
\caption{Controlled filter models. Parameter counts are shown as PathCopy/ResolventCopy and include the task-specific encoder and decoder.}
\label{tab:diagnostic-models}
\begin{tabular}{lccr}
\toprule
Model & Support or order & Propagation & Parameters \\
\midrule
Chebyshev & degree 20 & polynomial & 87,237/86,660 \\
Monomial & degree 20 & polynomial & 87,237/86,660 \\
Finite ARMA & 20 stacks, 20 steps & tied recurrence & 88,089/87,746 \\
Exact bank & 8 fixed poles & equilibrium solves & 90,095/89,203 \\
20-step restart & 8 poles, 20 steps & matched recurrence & 90,095/89,203 \\
\bottomrule
\end{tabular}
\end{table}
The Exact bank and 20-step restart use fixed poles $a_k\in\pm\{0.6480543,0.95,0.995,0.9994446\}$.
Since the Exact bank explicitly contains the two target resolvent poles, $\alpha=0.95$ and $\alpha=0.995$, it serves as an oracle-aligned positive control whose purpose is to verify that the diagnostics recover a known equilibrium tail.
Each branch is one of eight parallel filters applied to the same input feature matrix $\mat{H}$, with its own fixed coefficient $a_k$.
In the Exact bank, branch $k$ returns
\[
\mat{Z}_k=(1-|a_k|)(\Id-a_k\mat{S})^{-1}\mat{H},
\]
where $\mat{S}$ is the symmetric normalized adjacency matrix defined above.
The factor $1-|a_k|$ keeps the output scale bounded as $|a_k|$ approaches one.
The 20-step restart model approximates this output by starting from $\mat{Z}_k^{(0)}=\mat{H}$ and applying
\[
\mat{Z}_k^{(t+1)}
=a_k\mat{S}\mat{Z}_k^{(t)}+(1-|a_k|)\mat{H},
\qquad t=0,\ldots,19.
\]
Each update combines information from neighboring nodes and adds a scaled copy of the original input.
Continuing these updates to convergence would give the Exact-bank branch output above.
The input representation and eight branch outputs are concatenated and mixed by a signed pointwise projection. The equilibrium control uses differentiable dense Cholesky solves and serves solely as a diagnostic reference.
The range scaling in Theorem~\ref{thm:transport} transfers to each positive-pole branch on paths by the bounded-degree-ratio argument in Appendix~\ref{app:transport-assumptions}; the negative-pole branches and learned signed projection lie outside its nonnegative-mixture premise. For the positive ResolventCopy target, its parameter is $\beta=\alpha$.

All neural results use three independent training runs, FP32 Adam with learning rate $10^{-3}$, row-wise LayerNorm for Marked PathCopy, and identity activations without normalization for ResolventCopy. 
Results use the checkpoint selected by validation loss.

\subsection{Full controlled results}

\begin{table}[h]
\centering
\small
\setlength{\tabcolsep}{4pt}
\caption{Controlled-filter results, averaged across three independent training runs. PathCopy models train at distances 2--10. ResolventCopy NMSE is evaluated on 120-node paths, and tail error is measured strictly beyond hop 20 for $\alpha=0.995$.}
\label{tab:controlled-summary}
\begin{tabular}{lrrrrrr}
\toprule
& \multicolumn{3}{c}{PathCopy accuracy} & \multicolumn{2}{c}{ResolventCopy NMSE} & \\
\cmidrule(lr){2-4}\cmidrule(lr){5-6}
Model & $d=10$ & $d=15$ & $d=20$ & $\alpha=.95$ & $\alpha=.995$ & Tail error \\
\midrule
Exact bank & 1.000 & 1.000 & 1.000 & $4.3{\times}10^{-7}$ & $1.8{\times}10^{-6}$ & $2.03{\times}10^{-5}$ \\
Chebyshev & 1.000 & 1.000 & 0.997 & $2.6{\times}10^{-6}$ & $2.75{\times}10^{-2}$ & 1.000 \\
Finite ARMA & 1.000 & 0.247 & 0.195 & $3.1{\times}10^{-4}$ & $1.22{\times}10^{-1}$ & 1.000 \\
20-step restart & 1.000 & 0.201 & 0.210 & $4.3{\times}10^{-3}$ & $2.91{\times}10^{-1}$ & 1.000 \\
Monomial & 0.995 & 0.193 & 0.193 & $5.1{\times}10^{-4}$ & $1.54{\times}10^{-1}$ & 1.000 \\
\bottomrule
\end{tabular}
\end{table}

\paragraph{Odd--even effect in the 20-step restart baseline.}
\label{app:pathcopy-parity}
Without self-loops, each propagation step switches between nodes at odd and even distances from the source.
Starting from $\mat{H}$ leaves a term $a_k^{20}\mat{S}^{20}\mat{H}$ after 20 updates, in addition to the restart sum.
This term carries source information only to even distances and remains strongly weighted for poles close to $\pm1$, explaining the alternating peaks.
The pattern across training runs is consistent with this expected dependence on walk parity.
On the shared test graphs at distances 11--16, FP64 propagation leaves every predicted class unchanged, and independent polynomial spot checks agree with the recurrence.
With trained parameters fixed, 19 or 21 updates change the parity of the retained term and remove the even-distance peaks.
The per-hop test sweep complements checkpoint selection at distance 12 by measuring behavior at both odd and even distances.

\paragraph{Task-valid perturbations.}
At $d=20$, standard PathCopy accuracy tests extrapolation beyond the training distances of 2--10 hops. The perturbations keep this distance fixed and ask a different question: does the prediction follow the marked source while ignoring the distractors? The Exact bank scores 1.000 after both source-class substitution and exhaustive distractor variation; Chebyshev scores 0.998 in both cases. We then place the relevant source--query dependency beyond Chebyshev's 20-hop support, either by moving the source marker to a distractor or by relocating the query. Chebyshev falls to accuracies of 0.182 and 0.189, respectively, while the Exact bank retains 0.999 and 0.998.

\paragraph{Tail decomposition.}
Tail error isolates the distant part of the response: it is the normalized squared prediction error computed only at nodes more than 20 hops from the impulse, as defined by Equation~\eqref{eq:tail} with $r_0=20$. Errors at nearer nodes therefore do not affect it.
We test four combinations of path length and impulse position: 80 nodes with an endpoint impulse, 80 nodes with a center impulse, 120 nodes with an endpoint impulse, and 120 nodes with a center impulse.
Table~\ref{tab:controlled-summary} averages tail error over these four combinations and the three independent training runs.
Its NMSE columns instead measure the complete output on 120-node paths. For $\alpha=0.995$, the region beyond 20 hops contains 1.63--1.90\% of the target's squared response energy. Every finite model has zero response in this region and therefore a tail error of 1.000. The Exact bank closely reproduces the target tail, with error $2.03\times10^{-5}$, cosine similarity $0.999992$, and projected gain $1.0018$, where a gain of one denotes the correct amplitude.
For $\alpha=0.95$, only approximately $1.68\times10^{-6}$ of the target energy lies beyond 20 hops, so the table reports tail error only for $\alpha=0.995$.

\subsection{Separating finite approximation from training}
\label{app:controlled-replay}

Figure~\ref{fig:fixed-pole-replay} isolates the recurrence by fixing the learned parameters: blue uses the Exact-bank parameters and orange the 20-step-restart parameters.
At ``32-bit equilibrium,'' both sets---including the orange one---are run all the way to equilibrium; numbered points run the same parameters for that many 64-bit restart steps.
``Original model'' restores the usual evaluation: 32-bit equilibrium for blue and 20 32-bit steps for orange.
Each difference is measured from that model's own 64-bit equilibrium and averaged over three independent runs.

\begin{figure}[h]
\centering
\includegraphics[width=.94\textwidth]{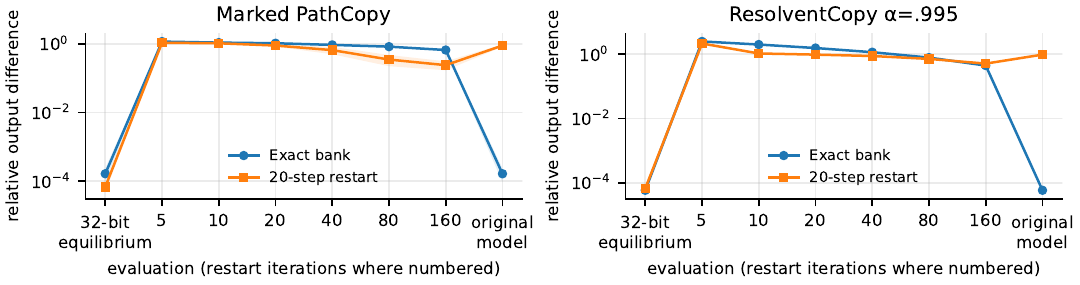}
\caption{Fixed-pole models with their learned parameters held fixed. Blue uses the Exact-bank parameters and orange the 20-step-restart parameters; the models were trained separately. ``32-bit equilibrium'' runs both sets to equilibrium, including the orange one. Numbered points run the same parameters for the indicated number of 64-bit restart steps. ``Original model'' uses 32-bit equilibrium for blue and 20 32-bit steps for orange. Relative difference is measured from the same model's 64-bit equilibrium; lines and bands show the mean and sample standard deviation over three runs.}
\label{fig:fixed-pole-replay}
\end{figure}

\begin{table}[h]
\centering
\small
\setlength{\tabcolsep}{3.5pt}
\caption{Selected values for the Exact-bank models (blue curves in Figure~\ref{fig:fixed-pole-replay}) on 120-node ResolventCopy paths. Relative difference is measured from the same model's 64-bit equilibrium. Profile statistics use an endpoint impulse, and tail error is measured beyond hop 20. Entries are means over three runs.}
\label{tab:fixed-pole-replay}
\begin{tabular}{llrrrrr}
\toprule
$\alpha$ & Evaluation & Rel. difference & NMSE & $R_{90}$ & Mass $>20$ & Tail error \\
\midrule
.95 & 64-bit equilibrium & 0 & $5.53{\times}10^{-7}$ & 7 & .00183 & .191 \\
.95 & 20 steps & .316 & .100 & 6 & 0 & 1.000 \\
.95 & 160 steps & .00847 & $7.50{\times}10^{-5}$ & 7 & .00120 & .0272 \\
\addlinespace
.995 & 64-bit equilibrium & 0 & $7.35{\times}10^{-7}$ & 23 & .1259 & .00219 \\
.995 & 20 steps & 1.586 & 2.514 & 7 & 0 & 1.000 \\
.995 & 160 steps & .448 & .200 & 18 & .0559 & .658 \\
\bottomrule
\end{tabular}
\end{table}

Figure~\ref{fig:fixed-pole-replay} shows that, for both separately trained models, stopping the recurrence early rather than numerical precision is the dominant source of error; Table~\ref{tab:fixed-pole-replay} quantifies this effect for the blue Exact-bank curve on ResolventCopy.
For the long-range $\alpha=.995$ target, 20 steps remove all measured influence beyond hop 20 and differ from equilibrium by $1.59$.
Even 160 steps leave a relative difference of $0.448$ and a tail error of $0.658$.
The shorter $\alpha=.95$ target converges faster, although its 20-step output still differs by $0.316$.
The same slow convergence appears on Marked PathCopy, whereas evaluating the equilibrium in 32-bit precision changes the output only slightly.

\subsection{Cross-spectrum and precision study}
\label{app:oracle}

This study asks whether a filter fitted to the graph frequencies seen during training remains accurate at unseen frequencies, on other graph structures, and at lower numerical precision. For each $d\in\{10,15,20\}$, we fit the filters only at the distinct eigenvalues of the neighbor-only cycle $C_{2d}$ and then freeze them. We evaluate the frozen filters, without refitting, on nearby cycles, on paths, and on a dense grid of 2,049 frequencies spanning $[-1,1]$. We compare degree-$d$ polynomials expressed in the Chebyshev and monomial bases with rational filters containing $K=d$ stable real poles under several stability margins. The targets are the exact-distance response $T_d(\lambda)$ and one-pole resolvents with $\alpha\in\{0.95,0.995\}$. We record approximation error in high precision, FP64, and FP32, together with diagnostics for conditioning, coefficient size, cancellation, gain, and non-finite outputs.

The \emph{synthesis matrix} maps the filter's coefficients to its values at the fitting frequencies.
Write the filter as $r(\lambda)=\sum_{j=0}^{d}c_j b_j(\lambda)$, where $b_j$ are the chosen basis functions and $c_j$ are the fitted coefficients.
For the rational fits, $b_0(\lambda)=1$ and $b_j(\lambda)=(1-a_j\lambda)^{-1}$ for $j\geq1$.
If $\lambda_i$ occurs $m_i$ times in the spectrum of the $2d$-node training cycle, the weighted synthesis matrix and its \emph{synthesis condition number} are
\begin{equation}
B_{ij}=\sqrt{\frac{m_i}{2d}}\,b_j(\lambda_i),
\qquad
\kappa_2(\mat{B})=\frac{\sigma_{\max}(\mat{B})}{\sigma_{\min}(\mat{B})},
\label{eq:synthesis-condition-number}
\end{equation}
where $\sigma_{\max}$ and $\sigma_{\min}$ are the largest and smallest singular values.
The weights account for repeated eigenvalues in the fitting error.
A large condition number means that small changes in the target values or basis evaluations can produce large changes in the fitted coefficients; the corresponding sensitivity bound is given in Theorem~\ref{thm:realization-formal}.

Theorem~\ref{thm:mode} establishes that fewer than $d$ stable real-pole modes cannot approximate $T_d$ uniformly on $[-1,1]$ with error below one, and the $K=d-1$ fits follow this bound. Increasing to $K=d$ removes this mode-count barrier: in high precision, the rational filter can match the finite set of training eigenvalues almost exactly. Those eigenvalues leave the response between them unconstrained. At $d=20$, the same fit has dense-grid RMSE $6.63$ and a synthesis condition number of $9.91\times10^{12}$; in FP32, its RMSE rises to $5.24\times10^4$. This conclusion concerns scalar filters with fixed real poles. Matrix-valued, complex-pole, and adaptive rational models lie outside the analysis.

Cross-spectrum behavior also depends on the target response. For the smooth $\alpha=0.995$ resolvent at $d=20$, the rational filter attains dense-grid RMSE $5.39\times10^{-4}$, compared with $1.95\times10^{-2}$ for the polynomial approximation. Table~\ref{tab:cross-spectrum} reports these dense-interval errors.

\begin{table}[h]
\centering
\small
\setlength{\tabcolsep}{4pt}
\caption{Compact cross-spectrum results on the dense interval for the reported real-pole configurations. The smooth-resolvent rows compare the degree-20 Chebyshev and rational fits.}
\label{tab:cross-spectrum}
\begin{tabular}{llrrr}
\toprule
Target & Fit & Synthesis condition & High-precision RMSE & FP32 RMSE \\
\midrule
$T_{10}$ & 10 real poles & $2.74{\times}10^{6}$ & $0.986$ & $0.985$ \\
$T_{15}$ & 15 real poles & $5.12{\times}10^{9}$ & $2.38$ & $32.8$ \\
$T_{20}$ & 20 real poles & $9.91{\times}10^{12}$ & $6.63$ & $5.24{\times}10^{4}$ \\
\addlinespace
Resolvent $\alpha=0.995$ & degree-20 Chebyshev & $1.41$ & $1.95{\times}10^{-2}$ & $1.95{\times}10^{-2}$ \\
Resolvent $\alpha=0.995$ & 20 real poles & $6.76{\times}10^{12}$ & $5.39{\times}10^{-4}$ & $5.56{\times}10^{-4}$ \\
\bottomrule
\end{tabular}
\end{table}

\section{GNN-Family and Architectural-Contrast Experiments}
\label{app:models}

This section details the comparison of trained configurations from different GNN families and the architectural contrasts in Section~\ref{sec:gnn-families}. The experiments use Marked PathCopy and RangeProfileCopy, defined in Appendices~\ref{app:pathcopy} and~\ref{app:rangeprofilecopy}.
The primary model families and their task-specific sizes are summarized in Table~\ref{tab:gnn-families}.

\subsection{Compared models and propagation mechanisms}

\begin{table}[h]
\centering
\small
\setlength{\tabcolsep}{4pt}
\caption{Models in the GNN-family comparison. Parameter counts include the common encoder and decoder. GraphGPS+RWSE is a structural-information control.}
\label{tab:gnn-families}
\begin{tabular}{@{}lp{0.56\textwidth}cc@{}}
\toprule
Model & Primary propagation mechanism & \multicolumn{1}{c}{\shortstack{PathCopy\\params.}} & \multicolumn{1}{c}{\shortstack{RangeProfile\\params.}} \\
\midrule
GatedGCN & 20 residual local steps with persistent directed-edge states & 108,709 & 108,548 \\
Stable-ChebNet & one stabilized degree-20 Chebyshev block & 97,493 & 97,012 \\
AMP & 20 GatedGCN steps, global depth posterior, and channelwise message gates & 109,991 & 109,830 \\
GraphGPS & four blocks combining GatedGCN and within-graph global attention & 87,085 & 86,884 \\
MP-SSM & one sequential normalized-adjacency state recurrence with 20 updates & 95,461 & 104,804 \\
MLP & two pointwise layers and no graph communication & 98,069 & 96,988 \\
A-DGN & 20 shared-weight antisymmetric neighbor-sum iterations & 96,989 & 103,044 \\
Tikhonov & one degree-5 regularized inverse with learned node-wise $Q$ & 98,140 & 101,739 \\
SONAR & 20 shared-weight wave updates with persistent velocity & 101,894 & 101,253 \\
GRIT & four global-attention layers with relative random-walk features & 108,581 & 108,340 \\
\bottomrule
\end{tabular}
\end{table}

\paragraph{Common design.}
The models share a pointwise encoder and prediction head. Except for GRIT's structural encodings, BatchNorm and attention dropout, they use row-wise LayerNorm, zero dropout, the same neighbor-only undirected graph for local operations, and no virtual node or positional encoding. Self-loops are absent unless required by the architecture. The MLP is exactly invariant to graph and edge changes.

\paragraph{GatedGCN.}
GatedGCN~\citep{bresson2018gatedgcn} keeps a separate feature vector for each direction of an undirected edge: one for messages from $u$ to $v$ and another for messages from $v$ to $u$.
Each layer updates these edge states using their previous values and the features of the two connected nodes, then passes the updated states to the next layer; this carryover is referred to as ``persistent''.
The edge updates determine the weights given to messages in each direction.
Node and edge updates use row-wise LayerNorm.

\paragraph{GraphGPS.}
GraphGPS combines a GatedGCN local branch with masked four-head attention applied independently within each graph~\citep{rampasek2022gps}. It sums the local and global branches, then applies a pointwise feed-forward network with twice the hidden width. Cross-graph attention is excluded.

\paragraph{Stable-ChebNet.}
We use the Stable-ChebNet model introduced by \citet{hariri2025chebnet}. Here $K=20$ is the maximum Chebyshev degree and graph hop, so the block contains 21 terms $T_0,\ldots,T_{20}$, where $T_k$ is the degree-$k$ Chebyshev polynomial. Its Euler step size is $\varepsilon=0.1$ and its dissipative shift is $\gamma=0.05$.

\paragraph{AMP.}
AMP uses a discretized Folded-Normal posterior on global depths 1--20, initialized with location and scale $(\mu,\sigma)=(10,5)$; its renormalized prior uses $(\mu,\sigma)=(5,10)$~\citep{errica2025amp}. The same prediction head is applied at every depth. Training uses the posterior expectation of the per-depth task loss plus depth and Gaussian parameter-prior terms normalized by training-set size. Its reported depth distribution is global, not a per-node stopping rule. Each embedding-conditioned affine sigmoid message filter is channelwise.

\paragraph{MP-SSM.}
MP-SSM uses a sequential real recurrence~\citep{ceni2025mpssm}. Let $\mat{U}$ be the encoded node-feature matrix, $\mat{B}$ the input-projection matrix, $\mat{W}$ the recurrent channel-mixing matrix, $\mat{H}_t$ the recurrent node-state matrix after $t$ updates, and $\widetilde{\mat{A}}$ the symmetric-normalized adjacency after adding the architecture-required self-loops. The model initializes $\mat{H}_0=\mat{U}\mat{B}$ and performs
\begin{equation}
\mat{H}_{t+1}
=
\widetilde{\mat{A}}\mat{H}_t\mat{W}+\mat{U}\mat{B},
\qquad t=0,\ldots,19,
\end{equation}
A two-layer pointwise ReLU MLP and row-wise LayerNorm are applied only after the linear recurrence.

\paragraph{Anti-Symmetric DGN.}
A-DGN follows the shared-weight Euler discretization of the stable graph ODE in \citet{gravina2023adgn}. With node state $\mat{H}_t$, learned channel matrix $\mat{W}$, neighbor-sum message $\Phi(\mat{H}_t)$, step size $\varepsilon=0.1$, and dissipation $\gamma=0.1$, each of the 20 iterations has the form
\begin{equation}
\mat{H}_{t+1}=\mat{H}_t+\varepsilon\tanh\!\left(
  (\mat{W}-\mat{W}^{\mathsf T}-\gamma\mat{I})\mat{H}_t
  +\Phi(\mat{H}_t)+\mat{b}
\right).
\end{equation}
The skew-symmetric channel mixing and dissipative shift control stability. We use shared weights, neighbor-sum aggregation, no inserted self-loops, scalar edge weights, and exactly 20 iterations.

\paragraph{Tikhonov.}
The Tikhonov layer implements the generalized regularization operator of \citet{tremblay2026generalized},
\begin{equation}
\mat{R}=(p(\mat{L})+\mat{Q})^{-1}\mat{Q},
\end{equation}
where $\mat{L}$ is the symmetric normalized Laplacian, $p$ is a learned positive degree-5 Bernstein polynomial, and $\mat{Q}$ is a positive learned node-wise diagonal regularizer. A three-layer ChebNet with $K=3$ and residual connections, hidden width 8, and a two-layer head predicts $\mat{Q}$ from node features.
We approximate the inverse using graph-segmented preconditioned conjugate gradients, with 30 iterations in the GNN-family comparison and 50 in the molecular study, so batched molecular graphs remain independent.
The output projection is followed by ReLU and row-wise LayerNorm.

\paragraph{SONAR.}
SONAR~\citep{trenta2025sonar} evolves node states $\mat{H}_t$ and velocities $\mat{V}_t$ through a damped graph-wave recurrence. Starting from encoded features $\mat{H}_0$ and learned initial velocity $\mat{V}_0=v_\theta(\mat{H}_0)$, it performs 20 updates with shared parameters:
\begin{equation}
\mat{V}_{t+1}=\mat{V}_t-h\!\left[\mathcal{L}_{c(\mat{H}_t)}g_\theta(\mat{H}_t)+\gamma_\theta(\mat{H}_t)\odot\mat{V}_t-f_\theta(\mat{H}_t)\right],
\qquad \mat{H}_{t+1}=\mat{H}_t+h\mat{V}_{t+1}.
\end{equation}
Here $g_\theta$, $f_\theta$ and $v_\theta$ are learned affine maps, and $\gamma_\theta\geq0$ is learned damping. The unnormalized Laplacian acts as $(\mathcal{L}_c Z)_v=(\sum_w c_{vw})Z_v-\sum_u c_{uv}Z_u$, with nonnegative directed edge weights recomputed from the connected nodes' states at every update.
Each update uses only neighboring states, giving a maximum support of 20 hops. Pointwise normalization and a two-layer GELU MLP with residual connections follow the recurrence.

\paragraph{GRIT.}
GRIT~\citep{ma2023grit} uses four layers with eight-head global attention and relative random-walk features $I,P,\ldots,P^{20}$ to update node and pair representations.
Attention covers every node pair within a graph, so limiting the walk features to order 20 does not bound spatial support.
The model uses ReLU, full pair-state updates, degree scaling, and BatchNorm (momentum 0.1, $\epsilon=10^{-5}$). Dropout acts only on attention during training. At evaluation, predictions and influence calculations use the running means and variances accumulated during training and saved in the checkpoint.
Walk features and degrees are computed from unweighted graph connectivity, with each random-walk step choosing uniformly among the current node's neighbors. Molecular bond information enters through learned pair representations, as detailed in Appendix~\ref{app:molecular}.

\paragraph{Targeted controls.}
The upper-right panel of Figure~\ref{fig:architecture-summary} includes GraphGPS+RWSE as a structural-encoding control for the RangeProfileCopy comparison. 
It projects raw return probabilities for walks 1--20 to 16 channels and concatenates them with an $h-16$ base encoding, where $h$ is the hidden width. No activation or batch-statistic normalization is applied to these features, and the model retains the same total width, blocks, prediction head, optimizer, and selected learning rate as GraphGPS. 
The two bottom panels of Figure~\ref{fig:architecture-summary} instead probe individual propagation mechanisms through paired controls.
They compare GatedGCN-5 with GatedGCN-20, vanilla degree-20 ChebNet with Stable-ChebNet, GatedGCN-20 with AMP, GraphGPS without and with global attention, and MP-SSM without and with its state-space recurrence.

\subsection{PathCopy training, selection, and accuracy requirements}
\label{app:pathcopy-competence}

We evaluate three stages separately: fitting the training graphs, generalizing to new graphs within the trained distance range, and extrapolating beyond that range.
At each validation check during training, we compute the mean cross-entropy loss separately at every source--query distance from 2 to 10 hops and take the largest of these nine values as the selection score.
For each run, we select the saved model parameters from the epoch with the lowest selection score.
Using the largest loss prevents good performance at some distances from hiding poor performance at another.
Training and validation remain disjoint, and final-test graphs are accessed only after screening and confirmation.

The GatedGCN, Stable-ChebNet, AMP, GraphGPS, MP-SSM, MLP, A-DGN and Tikhonov searches start with three predefined training configurations, each evaluated once on validation data.
The three configurations are: (i) AdamW with weight decay $10^{-5}$, a plateau schedule, at most 200 epochs, and learning rate $3\times10^{-4}$ for GatedGCN and MP-SSM, $10^{-3}$ for MLP, and $3\times10^{-3}$ otherwise; (ii) constant-rate AdamW at $10^{-3}$; and (iii) constant-rate AdamW at $3\times10^{-4}$. The latter two use weight decay $10^{-5}$, at most 500 epochs, and early-stopping patience 80, compared with patience 40 for the first configuration.

If none of these configurations reaches 95\% validation accuracy at every distance 2--10, we consider three additional configurations: Adam at $10^{-3}$ with no weight decay and a 10-epoch warm-up followed by cosine decay; AdamW at $10^{-3}$ with weight decay $10^{-4}$, dropout 0.1, and a plateau schedule; and a model-specific variant of the constant-rate $10^{-3}$ configuration. This last variant removes normalization for GatedGCN; reduces the Stable-ChebNet step size and dissipation to 0.05 and 0.01; shifts AMP's depth-posterior mean and scale to 5 and 3; uses eight GraphGPS blocks with dropout 0.1; uses two residual 10-step MP-SSM blocks with inter-block LayerNorm and recurrent spectral radius 0.9; reduces A-DGN dissipation to 0.01 and uses symmetric neighbor normalization; initializes Tikhonov's $q$ to 0.01; or gives the MLP a two-layer head with dropout 0.1. Plateau schedules halve the rate after 15 stalled epochs, cosine schedules decay to $10^{-6}$, and all configurations clip gradients at norm 1.
For SONAR, we test each learning rate with step sizes $\{0.05,0.1,0.5\}$ using constant-rate AdamW. We also test initializing all edge weights to one, using only training and validation data. The selected model uses $h=0.05$ and the standard initialization (Table~\ref{tab:e1-competence-profiles}).
For GRIT, we cross learning rates $\{10^{-4},3\times10^{-4},10^{-3}\}$ with attention dropout $\{0,0.2,0.5\}$, using AdamW with a ten-epoch linear warm-up from 0.01 times the learning rate and cosine decay over a 500-epoch horizon. The selected attention dropout is zero; training/evaluation batch sizes are 64/32. Table~\ref{tab:e1-competence-profiles} summarizes the selected settings.

\begin{table}[t]
\centering
\footnotesize
\setlength{\tabcolsep}{2.5pt}
\renewcommand{\arraystretch}{1.12}
\caption{Marked PathCopy model selection. ``Screened configs'' counts configurations in the main validation search; the two best are repeated on three additional seeds. WD/clip gives weight decay and gradient-clipping norm. $\dagger$ SONAR's nine combinations vary learning rate and integration step; additional initialization checks use only training and validation data.}
\label{tab:e1-competence-profiles}
\begin{adjustbox}{max width=\textwidth}
\begin{tabular}{lcccccc}
\toprule
Model & \shortstack{Screened\\configs} & Width & Optimizer & LR / schedule & Epochs / patience & WD / clip \\
\midrule
GatedGCN & 3 & 32 & AdamW & \shortstack{$10^{-3}$\\constant} & 500 / 80 & $10^{-5}$ / 1 \\
Stable-ChebNet & 3 & 96 & AdamW & \shortstack{$3\times10^{-4}$\\constant} & 500 / 80 & $10^{-5}$ / 1 \\
AMP & 3 & 32 & AdamW & \shortstack{$3\times10^{-4}$\\constant} & 500 / 80 & $10^{-5}$ / 1 \\
GraphGPS & 3 & 40 & AdamW & \shortstack{$10^{-3}$\\constant} & 500 / 80 & $10^{-5}$ / 1 \\
MP-SSM & 6 & 152 & Adam & \shortstack{$10^{-3}$\\warm-up/cosine} & 500 / 80 & 0 / 1 \\
MLP & 6 & 216 & AdamW & \shortstack{$10^{-3}$\\plateau} & 200 / 40 & $10^{-5}$ / 1 \\
A-DGN & 3 & 216 & AdamW & \shortstack{$3\times10^{-4}$\\constant} & 500 / 80 & $10^{-5}$ / 1 \\
Tikhonov & 3 & 288 & AdamW & \shortstack{$10^{-3}$\\constant} & 500 / 80 & $10^{-5}$ / 1 \\
SONAR & $9^{\dagger}$ & 128 & AdamW & \shortstack{$10^{-3}$\\constant} & 500 / 80 & $10^{-5}$ / 1 \\
GRIT & 9 & 48 & AdamW & \shortstack{$10^{-4}$\\warm-up/cosine} & 500 / 80 & $10^{-5}$ / 1 \\
\bottomrule
\end{tabular}
\end{adjustbox}
\end{table}

For the three-configuration searches, the ``Screened configs'' column in Table~\ref{tab:e1-competence-profiles} reports three if at least one configuration achieves at least 95\% validation accuracy at every source--query distance from 2 to 10 hops.
If each of the three initial configurations falls below 95\% at one or more distances, we test the three additional configurations described above, bringing the total to six.
The two highest-ranked configurations are each evaluated on three additional training seeds. Configurations are ranked by their minimum validation accuracy over distances 2--10, then by mean worst-distance cross-entropy across the available runs, and finally by parameter count. A configuration is selected only if every confirmation run attains at least 0.95 validation accuracy at every distance 2--10; among those that pass, the highest-ranked configuration is selected. For MLP and SONAR, where none passes, the top-ranked configuration is retained.

Widths are fixed before tuning to place models near a common $10^5$-parameter budget; when a model-specific variant changes the parameterization, its width is recomputed under the same budget.
All configurations use row-wise LayerNorm and zero dropout except the additional AdamW configuration with dropout 0.1 and the model-specific exceptions described above.
Only MP-SSM and the MLP control required the three additional configurations.
The selected configurations retain each model's architecture; MP-SSM uses Adam with warm-up and cosine decay.

The independent in-range test is never used to choose a configuration. We interpret longer-distance errors as failures of extrapolation only when accuracy is at least 95\% on both training and independent test graphs at every distance 2--10 in at least four of five runs. Eight GNNs exceed this requirement in all five runs, with 100\% training accuracy and at least 99.6\% test accuracy. SONAR fails to learn the training task and remains at 20\% test accuracy at every evaluated distance, so its longer-distance errors do not isolate extrapolation failure. MLP is the no-propagation control.

\subsection{Robustness to the choice of configuration}
\label{app:pathcopy-candidate-robustness}

For seven of the GNNs that learn the training task, we compare the training configuration used in the main PathCopy results with the other leading candidate from the validation search described in Appendix~\ref{app:pathcopy-competence}. The candidates share the same architecture but differ in learning rate or, for MP-SSM, in optimizer, schedule, and regularization. We ask whether switching to the alternative changes the finding that near-perfect accuracy at the trained distances 2--10 can fall to chance at distances 15 and 20. We evaluate three independent runs of each configuration (Table~\ref{tab:e1-candidate-robustness}); neither candidate was chosen using performance beyond the trained range.

For GatedGCN, AMP, and A-DGN, both configurations succeed at the trained distances but fall near five-class chance at distances 15 and 20. In contrast, both GraphGPS configurations remain perfectly accurate, and both Stable-ChebNet configurations exceed 97\% accuracy at distance 20. Changing between these candidates therefore does not reverse the main qualitative contrast: success within the trained range need not carry over to longer distances, even when those distances lie within the model's support.

The alternative configurations reach 95\% accuracy at every distance 2--10 on both validation and independent test graphs in 0/3 runs for MP-SSM and 2/3 for Tikhonov (Table~\ref{tab:e1-candidate-robustness}).
Their longer-distance averages include runs that do not meet this requirement, so they do not isolate generalization beyond the trained range.

\begin{table}[t]
\centering
\footnotesize
\setlength{\tabcolsep}{2.5pt}
\caption{Does another strong configuration change extrapolation? Accuracies are percentages; Selected denotes the configuration used in the main comparison, and Alternative the other leading validation candidate. Minima cover all three runs and distances 2--10. Pass / 3 counts runs reaching 95\% at every trained distance on both validation and independent test graphs. The last two columns show means and 95\% bootstrap intervals across runs and test graphs, including unsuccessful runs.}
\label{tab:e1-candidate-robustness}
\begin{adjustbox}{max width=\textwidth}
\begin{tabular}{llrrrcc}
\toprule
Model & Setting & Min. val. & Min. test & Pass / 3 & $d=15$ & $d=20$ \\
\midrule
GatedGCN & Selected & 100.00 & 99.80 & 3/3 & 20.40 [18.33, 22.60] & 20.13 [18.00, 22.27] \\
GatedGCN & Alternative & 100.00 & 100.00 & 3/3 & 20.00 [18.07, 22.20] & 20.00 [18.00, 22.13] \\
Stable-ChebNet & Selected & 100.00 & 100.00 & 3/3 & 99.93 [99.73, 100.00] & 97.73 [96.40, 98.87] \\
Stable-ChebNet & Alternative & 100.00 & 100.00 & 3/3 & 99.47 [98.80, 100.00] & 97.87 [96.20, 99.13] \\
AMP & Selected & 100.00 & 99.40 & 3/3 & 20.07 [18.13, 22.20] & 20.07 [18.07, 22.13] \\
AMP & Alternative & 100.00 & 100.00 & 3/3 & 19.47 [17.53, 21.47] & 20.13 [18.07, 22.13] \\
GraphGPS & Selected & 100.00 & 100.00 & 3/3 & 100.00 [100.00, 100.00] & 100.00 [100.00, 100.00] \\
GraphGPS & Alternative & 100.00 & 100.00 & 3/3 & 100.00 [100.00, 100.00] & 100.00 [100.00, 100.00] \\
MP-SSM & Selected & 100.00 & 100.00 & 3/3 & 21.20 [18.60, 23.80] & 22.00 [19.80, 24.20] \\
MP-SSM & Alternative & 20.00 & 20.00 & 0/3 & 20.00 [18.00, 22.00] & 20.00 [17.93, 22.00] \\
A-DGN & Selected & 100.00 & 99.80 & 3/3 & 18.73 [15.80, 21.47] & 19.80 [17.87, 21.80] \\
A-DGN & Alternative & 100.00 & 100.00 & 3/3 & 19.60 [17.27, 21.87] & 20.67 [18.20, 23.13] \\
Tikhonov & Selected & 100.00 & 100.00 & 3/3 & 100.00 [100.00, 100.00] & 62.27 [55.87, 68.87] \\
Tikhonov & Alternative & 19.60 & 19.60 & 2/3 & 73.07 [20.53, 100.00] & 55.67 [21.07, 76.07] \\
\bottomrule
\end{tabular}
\end{adjustbox}
\end{table}

\subsection{RangeProfileCopy training and selection}

RangeProfileCopy uses independent training, validation, and test graphs. All models use AdamW, validation-loss checkpoint selection, and at most 200 epochs; results average three independent runs. Validation searches use learning rates $\{3\times10^{-4},10^{-3},3\times10^{-3}\}$ (the first two rates for GRIT). Dropout is zero apart from GRIT's attention dropout of 0.2.
SONAR additionally searches step sizes $\{0.05,0.1,0.5\}$ and confirms the two best configurations across three runs, selecting $h=0.05$ and learning rate $3\times10^{-4}$. GRIT selects learning rate $10^{-3}$ directly from one validation run per rate. Both selections use the resolvent task and transfer unchanged to the local and inverse-distance targets. GRIT retains the PathCopy warm-up/cosine schedule with a 200-epoch horizon, patience 40, weight decay $10^{-5}$, clipping at norm 1, and training/evaluation batches of 32/16.

\subsection{Architectural contrasts and precision replay}

The five model comparisons in Section~\ref{sec:gnn-families} (Figure~\ref{fig:architecture-summary}, bottom panels) use the hard resolvent target of RangeProfileCopy, $\mat{Y}=(1-\alpha)(\Id-\alpha\mat{S})^{-1}\mat{X}$ with $\alpha=0.995$ (Appendix~\ref{app:rangeprofilecopy}).
Each model is trained in three independent runs.
The pairs are GatedGCN-5 versus GatedGCN-20, vanilla degree-20 ChebNet versus Stable-ChebNet, fixed-depth GatedGCN-20 versus AMP, GraphGPS with local message passing only versus GraphGPS with global attention, and one-step versus recurrent MP-SSM. Widths are chosen under the common parameter budget, and every comparison uses validation-selected checkpoints. For AMP, we also analyze its learned depth distribution and message gates on all three RangeProfileCopy rules.

We assess the effect of numerical precision on the hard resolvent task by reevaluating GatedGCN-20, Stable-ChebNet, AMP, GraphGPS, MP-SSM-20, SONAR, and GRIT with their trained parameters fixed.
For each checkpoint, we use the same 256 test graphs to compute predictions and influence in FP64, FP32, and BF16, taking FP64 as the reference.
The BF16 results are reported in Table~\ref{tab:e2-precision} and Figure~\ref{fig:e2-precision-amp}(a).

\subsection{Primary predictive metrics}
\label{app:architecture-results}

\begin{table}[t]
\centering
\small
\setlength{\tabcolsep}{3.2pt}
\caption{RangeProfileCopy test NMSE (mean $\pm$ sample standard deviation across three runs) for validation-selected checkpoints. GraphGPS+RWSE is a matched structural-information control.}
\label{tab:e1-range-predictive}

\begin{adjustbox}{max width=\textwidth}
\begin{tabular}{lrrr}
\toprule
\multicolumn{4}{l}{\textbf{RangeProfileCopy NMSE}} \\
Model & Local & Resolvent & Inverse distance \\
\midrule
GatedGCN & \textbf{$1.539{\scriptstyle\pm0.208}\!\times\!10^{-4}$} & $1.262{\scriptstyle\pm0.056}\!\times\!10^{-1}$ & $2.058{\scriptstyle\pm0.020}\!\times\!10^{-1}$ \\
Stable-ChebNet & $2.148{\scriptstyle\pm0.023}\!\times\!10^{-2}$ & $2.283{\scriptstyle\pm0.022}\!\times\!10^{-2}$ & $1.486{\scriptstyle\pm0.014}\!\times\!10^{-1}$ \\
AMP & $2.160{\scriptstyle\pm0.344}\!\times\!10^{-2}$ & $4.559{\scriptstyle\pm0.163}\!\times\!10^{-1}$ & $3.519{\scriptstyle\pm0.065}\!\times\!10^{-1}$ \\
GraphGPS & $3.643{\scriptstyle\pm2.559}\!\times\!10^{-3}$ & $2.655{\scriptstyle\pm0.007}\!\times\!10^{-1}$ & $7.128{\scriptstyle\pm0.092}\!\times\!10^{-2}$ \\
MP-SSM & $2.706{\scriptstyle\pm0.057}\!\times\!10^{-2}$ & $1.603{\scriptstyle\pm0.019}\!\times\!10^{-1}$ & $2.185{\scriptstyle\pm0.080}\!\times\!10^{-1}$ \\
MLP & $8.480{\scriptstyle\pm0.003}\!\times\!10^{-1}$ & $9.463{\scriptstyle\pm0.018}\!\times\!10^{-1}$ & $7.632{\scriptstyle\pm0.024}\!\times\!10^{-1}$ \\
A-DGN & $1.728{\scriptstyle\pm0.051}\!\times\!10^{-2}$ & $2.262{\scriptstyle\pm0.303}\!\times\!10^{-1}$ & $2.346{\scriptstyle\pm0.176}\!\times\!10^{-1}$ \\
Tikhonov & $9.861{\scriptstyle\pm0.002}\!\times\!10^{-2}$ & \textbf{$2.804{\scriptstyle\pm0.107}\!\times\!10^{-4}$} & $1.771{\scriptstyle\pm0.108}\!\times\!10^{-1}$ \\
\addlinespace[1pt]
GraphGPS+RWSE & -- & $2.432{\scriptstyle\pm0.027}\!\times\!10^{-1}$ & $5.978{\scriptstyle\pm0.470}\!\times\!10^{-2}$ \\
SONAR & $0.0308\pm0.0002$ & $0.0431\pm0.0017$ & $0.1534\pm0.0030$ \\
GRIT & $0.0006\pm0.0002$ & $0.0126\pm0.0003$ & \textbf{$0.0081\pm0.0014$} \\
\bottomrule
\end{tabular}
\end{adjustbox}
\end{table}

\begin{table}[t]
\centering
\footnotesize
\setlength{\tabcolsep}{2.5pt}
\renewcommand{\arraystretch}{1.12}
\caption{Marked PathCopy test accuracy (\%) beyond the trained distances 2--10. Entries show means over five runs, with 95\% uncertainty intervals below, using 500 shared, class-balanced graphs per distance. The seven baseline GNNs achieve at least 95\% training and test accuracy at every trained distance; SONAR does not in any of its five runs, so its longer-distance errors do not isolate extrapolation failure. MLP is the no-propagation control. GRIT also meets the accuracy requirement in all five runs.}
\label{tab:e1-competence-extrapolation}
\begin{adjustbox}{max width=\textwidth}
\begin{tabular}{lcccccccccc}
\toprule
$d$ & GatedGCN & \shortstack{Stable-\\ChebNet} & AMP & GraphGPS & MP-SSM & A-DGN & Tikhonov & MLP & SONAR & GRIT \\
\midrule
$11$ & \shortstack{54.56\\{\scriptsize[39.00, 71.60]}} & \shortstack{100.00\\{\scriptsize[100.00, 100.00]}} & \shortstack{42.64\\{\scriptsize[20.44, 73.76]}} & \shortstack{100.00\\{\scriptsize[100.00, 100.00]}} & \shortstack{96.72\\{\scriptsize[94.68, 98.24]}} & \shortstack{92.72\\{\scriptsize[88.16, 96.36]}} & \shortstack{100.00\\{\scriptsize[100.00, 100.00]}} & \shortstack{20.00\\{\scriptsize[18.52, 21.56]}} & \shortstack{20.00\\{\scriptsize[18.48, 21.52]}} & \shortstack{100.00\\{\scriptsize[100.00, 100.00]}} \\
$12$ & \shortstack{20.00\\{\scriptsize[18.40, 21.60]}} & \shortstack{100.00\\{\scriptsize[100.00, 100.00]}} & \shortstack{20.00\\{\scriptsize[18.40, 21.56]}} & \shortstack{100.00\\{\scriptsize[100.00, 100.00]}} & \shortstack{58.64\\{\scriptsize[52.32, 64.60]}} & \shortstack{37.68\\{\scriptsize[30.80, 45.12]}} & \shortstack{100.00\\{\scriptsize[100.00, 100.00]}} & \shortstack{20.00\\{\scriptsize[18.48, 21.56]}} & \shortstack{20.00\\{\scriptsize[18.44, 21.56]}} & \shortstack{100.00\\{\scriptsize[100.00, 100.00]}} \\
$13$ & \shortstack{20.00\\{\scriptsize[18.40, 21.64]}} & \shortstack{99.84\\{\scriptsize[99.56, 100.00]}} & \shortstack{20.36\\{\scriptsize[18.76, 22.04]}} & \shortstack{100.00\\{\scriptsize[100.00, 100.00]}} & \shortstack{32.32\\{\scriptsize[29.36, 35.28]}} & \shortstack{21.64\\{\scriptsize[19.68, 23.52]}} & \shortstack{99.96\\{\scriptsize[99.84, 100.00]}} & \shortstack{20.00\\{\scriptsize[18.52, 21.60]}} & \shortstack{20.00\\{\scriptsize[18.48, 21.72]}} & \shortstack{100.00\\{\scriptsize[100.00, 100.00]}} \\
$14$ & \shortstack{20.00\\{\scriptsize[18.56, 21.52]}} & \shortstack{99.88\\{\scriptsize[99.56, 100.00]}} & \shortstack{19.96\\{\scriptsize[18.48, 21.52]}} & \shortstack{100.00\\{\scriptsize[100.00, 100.00]}} & \shortstack{24.24\\{\scriptsize[22.04, 26.40]}} & \shortstack{20.08\\{\scriptsize[18.44, 21.64]}} & \shortstack{99.96\\{\scriptsize[99.84, 100.00]}} & \shortstack{20.00\\{\scriptsize[18.44, 21.48]}} & \shortstack{20.00\\{\scriptsize[18.36, 21.60]}} & \shortstack{100.00\\{\scriptsize[100.00, 100.00]}} \\
$15$ & \shortstack{20.00\\{\scriptsize[18.44, 21.56]}} & \shortstack{99.72\\{\scriptsize[99.32, 100.00]}} & \shortstack{19.52\\{\scriptsize[17.72, 21.36]}} & \shortstack{100.00\\{\scriptsize[100.00, 100.00]}} & \shortstack{19.96\\{\scriptsize[18.20, 21.84]}} & \shortstack{20.04\\{\scriptsize[18.48, 21.72]}} & \shortstack{97.68\\{\scriptsize[95.00, 100.00]}} & \shortstack{20.00\\{\scriptsize[18.48, 21.56]}} & \shortstack{20.00\\{\scriptsize[18.44, 21.60]}} & \shortstack{100.00\\{\scriptsize[100.00, 100.00]}} \\
$16$ & \shortstack{20.00\\{\scriptsize[18.36, 21.52]}} & \shortstack{99.92\\{\scriptsize[99.68, 100.00]}} & \shortstack{20.36\\{\scriptsize[18.72, 22.12]}} & \shortstack{100.00\\{\scriptsize[100.00, 100.00]}} & \shortstack{21.56\\{\scriptsize[19.68, 23.40]}} & \shortstack{20.56\\{\scriptsize[18.84, 22.28]}} & \shortstack{88.56\\{\scriptsize[75.64, 98.08]}} & \shortstack{20.00\\{\scriptsize[18.52, 21.52]}} & \shortstack{20.00\\{\scriptsize[18.48, 21.64]}} & \shortstack{100.00\\{\scriptsize[100.00, 100.00]}} \\
$17$ & \shortstack{20.00\\{\scriptsize[18.44, 21.52]}} & \shortstack{98.84\\{\scriptsize[97.56, 99.92]}} & \shortstack{19.56\\{\scriptsize[17.76, 21.20]}} & \shortstack{100.00\\{\scriptsize[100.00, 100.00]}} & \shortstack{20.40\\{\scriptsize[18.80, 22.00]}} & \shortstack{19.44\\{\scriptsize[17.88, 21.04]}} & \shortstack{73.08\\{\scriptsize[45.72, 95.08]}} & \shortstack{20.00\\{\scriptsize[18.60, 21.48]}} & \shortstack{20.00\\{\scriptsize[18.44, 21.60]}} & \shortstack{100.00\\{\scriptsize[100.00, 100.00]}} \\
$18$ & \shortstack{20.00\\{\scriptsize[18.48, 21.56]}} & \shortstack{99.04\\{\scriptsize[98.12, 99.80]}} & \shortstack{20.28\\{\scriptsize[18.60, 21.88]}} & \shortstack{100.00\\{\scriptsize[100.00, 100.00]}} & \shortstack{21.92\\{\scriptsize[19.88, 23.92]}} & \shortstack{19.96\\{\scriptsize[18.36, 21.56]}} & \shortstack{61.60\\{\scriptsize[38.56, 85.88]}} & \shortstack{20.00\\{\scriptsize[18.48, 21.60]}} & \shortstack{20.00\\{\scriptsize[18.40, 21.60]}} & \shortstack{100.00\\{\scriptsize[100.00, 100.00]}} \\
$19$ & \shortstack{20.00\\{\scriptsize[18.40, 21.56]}} & \shortstack{94.08\\{\scriptsize[91.20, 96.84]}} & \shortstack{19.60\\{\scriptsize[17.88, 21.24]}} & \shortstack{100.00\\{\scriptsize[100.00, 100.00]}} & \shortstack{19.84\\{\scriptsize[18.00, 21.72]}} & \shortstack{20.56\\{\scriptsize[18.76, 22.36]}} & \shortstack{43.16\\{\scriptsize[27.08, 66.28]}} & \shortstack{20.00\\{\scriptsize[18.44, 21.52]}} & \shortstack{20.00\\{\scriptsize[18.48, 21.56]}} & \shortstack{100.00\\{\scriptsize[100.00, 100.00]}} \\
$20$ & \shortstack{20.00\\{\scriptsize[18.48, 21.60]}} & \shortstack{95.36\\{\scriptsize[92.44, 98.20]}} & \shortstack{19.64\\{\scriptsize[18.04, 21.32]}} & \shortstack{100.00\\{\scriptsize[100.00, 100.00]}} & \shortstack{21.12\\{\scriptsize[19.24, 22.92]}} & \shortstack{20.08\\{\scriptsize[18.44, 21.76]}} & \shortstack{30.12\\{\scriptsize[21.08, 44.64]}} & \shortstack{20.00\\{\scriptsize[18.48, 21.64]}} & \shortstack{20.00\\{\scriptsize[18.44, 21.52]}} & \shortstack{100.00\\{\scriptsize[100.00, 100.00]}} \\
$21$ & \shortstack{20.00\\{\scriptsize[18.48, 21.64]}} & \shortstack{20.04\\{\scriptsize[18.32, 21.84]}} & \shortstack{20.48\\{\scriptsize[18.76, 22.28]}} & \shortstack{100.00\\{\scriptsize[100.00, 100.00]}} & \shortstack{20.76\\{\scriptsize[19.00, 22.56]}} & \shortstack{20.40\\{\scriptsize[18.84, 21.96]}} & \shortstack{24.40\\{\scriptsize[19.40, 32.60]}} & \shortstack{20.00\\{\scriptsize[18.44, 21.52]}} & \shortstack{20.00\\{\scriptsize[18.48, 21.56]}} & \shortstack{100.00\\{\scriptsize[100.00, 100.00]}} \\
$22$ & \shortstack{20.00\\{\scriptsize[18.44, 21.56]}} & \shortstack{19.00\\{\scriptsize[15.88, 22.12]}} & \shortstack{20.04\\{\scriptsize[18.48, 21.64]}} & \shortstack{100.00\\{\scriptsize[100.00, 100.00]}} & \shortstack{19.68\\{\scriptsize[17.76, 21.52]}} & \shortstack{20.08\\{\scriptsize[18.44, 21.72]}} & \shortstack{24.04\\{\scriptsize[18.96, 32.48]}} & \shortstack{20.00\\{\scriptsize[18.48, 21.60]}} & \shortstack{20.00\\{\scriptsize[18.48, 21.56]}} & \shortstack{100.00\\{\scriptsize[100.00, 100.00]}} \\
$23$ & \shortstack{20.00\\{\scriptsize[18.48, 21.56]}} & \shortstack{21.04\\{\scriptsize[19.28, 22.84]}} & \shortstack{20.32\\{\scriptsize[18.72, 22.08]}} & \shortstack{100.00\\{\scriptsize[100.00, 100.00]}} & \shortstack{20.16\\{\scriptsize[18.04, 22.36]}} & \shortstack{19.92\\{\scriptsize[18.28, 21.48]}} & \shortstack{22.16\\{\scriptsize[18.88, 26.96]}} & \shortstack{20.00\\{\scriptsize[18.48, 21.64]}} & \shortstack{20.00\\{\scriptsize[18.52, 21.52]}} & \shortstack{100.00\\{\scriptsize[100.00, 100.00]}} \\
$24$ & \shortstack{20.00\\{\scriptsize[18.52, 21.64]}} & \shortstack{19.52\\{\scriptsize[17.56, 21.44]}} & \shortstack{20.12\\{\scriptsize[18.52, 21.68]}} & \shortstack{100.00\\{\scriptsize[100.00, 100.00]}} & \shortstack{20.36\\{\scriptsize[18.56, 22.16]}} & \shortstack{20.04\\{\scriptsize[18.44, 21.56]}} & \shortstack{20.04\\{\scriptsize[18.52, 21.60]}} & \shortstack{20.00\\{\scriptsize[18.44, 21.56]}} & \shortstack{20.00\\{\scriptsize[18.48, 21.56]}} & \shortstack{100.00\\{\scriptsize[100.00, 100.00]}} \\
$25$ & \shortstack{20.00\\{\scriptsize[18.48, 21.56]}} & \shortstack{18.48\\{\scriptsize[16.72, 20.24]}} & \shortstack{20.00\\{\scriptsize[18.48, 21.60]}} & \shortstack{100.00\\{\scriptsize[100.00, 100.00]}} & \shortstack{19.00\\{\scriptsize[17.28, 20.72]}} & \shortstack{20.12\\{\scriptsize[18.52, 21.76]}} & \shortstack{20.04\\{\scriptsize[18.48, 21.60]}} & \shortstack{20.00\\{\scriptsize[18.44, 21.64]}} & \shortstack{20.00\\{\scriptsize[18.48, 21.60]}} & \shortstack{100.00\\{\scriptsize[100.00, 100.00]}} \\
$26$ & \shortstack{20.00\\{\scriptsize[18.44, 21.64]}} & \shortstack{21.44\\{\scriptsize[19.52, 23.44]}} & \shortstack{20.28\\{\scriptsize[18.72, 21.92]}} & \shortstack{100.00\\{\scriptsize[100.00, 100.00]}} & \shortstack{20.60\\{\scriptsize[18.88, 22.32]}} & \shortstack{20.08\\{\scriptsize[18.48, 21.64]}} & \shortstack{20.00\\{\scriptsize[18.40, 21.60]}} & \shortstack{20.00\\{\scriptsize[18.44, 21.60]}} & \shortstack{20.00\\{\scriptsize[18.40, 21.56]}} & \shortstack{100.00\\{\scriptsize[100.00, 100.00]}} \\
$27$ & \shortstack{20.00\\{\scriptsize[18.44, 21.64]}} & \shortstack{21.80\\{\scriptsize[19.72, 23.88]}} & \shortstack{19.96\\{\scriptsize[18.40, 21.48]}} & \shortstack{100.00\\{\scriptsize[100.00, 100.00]}} & \shortstack{21.92\\{\scriptsize[19.72, 24.08]}} & \shortstack{20.20\\{\scriptsize[18.64, 21.80]}} & \shortstack{20.00\\{\scriptsize[18.48, 21.52]}} & \shortstack{20.00\\{\scriptsize[18.48, 21.52]}} & \shortstack{20.00\\{\scriptsize[18.52, 21.60]}} & \shortstack{100.00\\{\scriptsize[100.00, 100.00]}} \\
$28$ & \shortstack{20.00\\{\scriptsize[18.44, 21.60]}} & \shortstack{20.32\\{\scriptsize[18.36, 22.28]}} & \shortstack{19.96\\{\scriptsize[18.40, 21.52]}} & \shortstack{100.00\\{\scriptsize[100.00, 100.00]}} & \shortstack{20.12\\{\scriptsize[18.12, 22.24]}} & \shortstack{20.20\\{\scriptsize[18.36, 21.92]}} & \shortstack{20.00\\{\scriptsize[18.40, 21.60]}} & \shortstack{20.00\\{\scriptsize[18.44, 21.60]}} & \shortstack{20.00\\{\scriptsize[18.52, 21.48]}} & \shortstack{100.00\\{\scriptsize[100.00, 100.00]}} \\
$29$ & \shortstack{20.00\\{\scriptsize[18.40, 21.56]}} & \shortstack{20.96\\{\scriptsize[18.84, 23.24]}} & \shortstack{19.88\\{\scriptsize[18.36, 21.44]}} & \shortstack{100.00\\{\scriptsize[100.00, 100.00]}} & \shortstack{20.24\\{\scriptsize[18.60, 21.84]}} & \shortstack{19.84\\{\scriptsize[18.12, 21.48]}} & \shortstack{20.00\\{\scriptsize[18.44, 21.60]}} & \shortstack{20.00\\{\scriptsize[18.48, 21.56]}} & \shortstack{20.00\\{\scriptsize[18.48, 21.48]}} & \shortstack{100.00\\{\scriptsize[100.00, 100.00]}} \\
$30$ & \shortstack{20.00\\{\scriptsize[18.44, 21.56]}} & \shortstack{18.48\\{\scriptsize[16.56, 20.32]}} & \shortstack{19.52\\{\scriptsize[17.76, 21.20]}} & \shortstack{100.00\\{\scriptsize[100.00, 100.00]}} & \shortstack{19.40\\{\scriptsize[17.72, 21.00]}} & \shortstack{19.76\\{\scriptsize[18.12, 21.40]}} & \shortstack{20.00\\{\scriptsize[18.48, 21.56]}} & \shortstack{20.00\\{\scriptsize[18.48, 21.60]}} & \shortstack{20.00\\{\scriptsize[18.52, 21.60]}} & \shortstack{100.00\\{\scriptsize[100.00, 100.00]}} \\
$40$ & \shortstack{20.00\\{\scriptsize[18.40, 21.56]}} & \shortstack{18.84\\{\scriptsize[16.68, 21.12]}} & \shortstack{20.12\\{\scriptsize[18.56, 21.68]}} & \shortstack{100.00\\{\scriptsize[100.00, 100.00]}} & \shortstack{19.20\\{\scriptsize[17.40, 21.04]}} & \shortstack{19.68\\{\scriptsize[18.08, 21.28]}} & \shortstack{20.00\\{\scriptsize[18.44, 21.64]}} & \shortstack{20.00\\{\scriptsize[18.48, 21.56]}} & \shortstack{20.00\\{\scriptsize[18.44, 21.56]}} & \shortstack{100.00\\{\scriptsize[100.00, 100.00]}} \\
\bottomrule
\end{tabular}
\end{adjustbox}
\end{table}

The GNN-family experiments in Section~\ref{sec:gnn-families} test whether models can retrieve a distant source's label on Marked PathCopy and predict weighted combinations of inputs on RangeProfileCopy.
The numerical results for the top panels of Figure~\ref{fig:architecture-summary} are given in  Table~\ref{tab:e1-range-predictive} (RangeProfileCopy NMSE) and Table~\ref{tab:e1-competence-extrapolation} (PathCopy accuracy beyond the trained distances 2--10).
We estimate PathCopy's 95\% uncertainty intervals by resampling training runs and their graphs 2,000 times. Intervals describe each model's accuracy, rather than differences between models.
Table~\ref{tab:e1-competence-extrapolation} covers every integer distance from 11 to 30, plus distance 40. Each distance uses 500 class-balanced test graphs, shared across all models and their five independent runs.
The results show a sharp drop for Stable-ChebNet, from 95.36\% accuracy at distance 20 to 20.04\% at 21, just beyond its support. Tikhonov declines more gradually, reaching 24.40\% at distance 21 and 20.04\% at 24. Figure~\ref{fig:architecture-summary} ends at distance 30 to make these transitions easier to see.

\subsection{Source-use and task-alignment diagnostics}
\label{app:competence-diagnostics}

In the Marked PathCopy experiment in Section~\ref{sec:gnn-families} (Figure~\ref{fig:architecture-summary}, top-left), several models fall to chance beyond the trained distances 2--10, even at distances their architectures can reach.
The tests below examine whether these models still use the marked source and whether moving the query closer restores correct copying.

Keeping the trained model parameters fixed, we modify test inputs at source--query distances 2, 5, 10, 15, 20, 30, and 40. Source substitution changes the marked source class, distractor variation changes the three irrelevant classes, role swaps move the source marker to a distractor, and query relocation moves the query to another valid node. Assignments are averaged within each graph, and intervals resample seeds and graphs hierarchically. Table~\ref{tab:e1-competence-source_accuracy} reports the central source-substitution result.

The results distinguish several behaviors. GraphGPS and GRIT follow changed source values and ignore distractor values while remaining correct. They also track relocated source/query roles throughout the tested ranges, although GRIT reaches 99.0\% role-swap accuracy when the moved source lies within one hop of the query. Stable-ChebNet follows source changes through distance 20 (95.50\% accuracy there), with source-following ending at its 20-hop support. GatedGCN and AMP change no predicted classes after source substitutions at distance 15; A-DGN changes 0.01\% and MP-SSM 2.50\%. All eight GNNs that learned the training task copy correctly when relocated queries lie within 2--10 hops, including on the larger test graphs, localizing their observed failures to distance-dependent source use. MP-SSM changes its prediction on 51.32\% of the sampled distance-15 distractor assignments (graph-first average), even though it rarely follows source changes at that distance. Tikhonov retains source-following at distance 15 (98.25\%) but degrades at distance 20 (29.44\%) despite its graph-wide operator.

\begin{table}[t]
\centering
\footnotesize
\setlength{\tabcolsep}{2.5pt}
\renewcommand{\arraystretch}{1.12}
\caption{Source-substitution accuracy: the prediction follows the replacement source class. Values are percentages with 95\% seed/graph bootstrap intervals; assignments are averaged within graphs first. All 500 graphs per distance and all four replacement source classes are used. Graph resamples are paired across models.}
\label{tab:e1-competence-source_accuracy}
\begin{adjustbox}{max width=\textwidth}
\begin{tabular}{lccccccc}
\toprule
Model & $d=2$ & $d=5$ & $d=10$ & $d=15$ & $d=20$ & $d=30$ & $d=40$ \\
\midrule
GatedGCN & \shortstack{100.00\\{\scriptsize[100.00, 100.00]}} & \shortstack{100.00\\{\scriptsize[100.00, 100.00]}} & \shortstack{100.00\\{\scriptsize[100.00, 100.00]}} & \shortstack{20.00\\{\scriptsize[19.59, 20.38]}} & \shortstack{20.00\\{\scriptsize[19.60, 20.38]}} & \shortstack{20.00\\{\scriptsize[19.60, 20.39]}} & \shortstack{20.00\\{\scriptsize[19.60, 20.40]}} \\
Stable-ChebNet & \shortstack{99.99\\{\scriptsize[99.96, 100.00]}} & \shortstack{100.00\\{\scriptsize[100.00, 100.00]}} & \shortstack{100.00\\{\scriptsize[100.00, 100.00]}} & \shortstack{99.82\\{\scriptsize[99.53, 100.00]}} & \shortstack{95.50\\{\scriptsize[92.77, 98.20]}} & \shortstack{20.38\\{\scriptsize[19.88, 20.87]}} & \shortstack{20.29\\{\scriptsize[19.75, 20.82]}} \\
AMP & \shortstack{100.00\\{\scriptsize[100.00, 100.00]}} & \shortstack{100.00\\{\scriptsize[100.00, 100.00]}} & \shortstack{99.98\\{\scriptsize[99.94, 100.00]}} & \shortstack{20.12\\{\scriptsize[19.68, 20.60]}} & \shortstack{20.09\\{\scriptsize[19.67, 20.50]}} & \shortstack{20.12\\{\scriptsize[19.70, 20.57]}} & \shortstack{19.97\\{\scriptsize[19.58, 20.37]}} \\
GraphGPS & \shortstack{100.00\\{\scriptsize[100.00, 100.00]}} & \shortstack{100.00\\{\scriptsize[100.00, 100.00]}} & \shortstack{100.00\\{\scriptsize[100.00, 100.00]}} & \shortstack{100.00\\{\scriptsize[100.00, 100.00]}} & \shortstack{100.00\\{\scriptsize[100.00, 100.00]}} & \shortstack{100.00\\{\scriptsize[100.00, 100.00]}} & \shortstack{100.00\\{\scriptsize[100.00, 100.00]}} \\
MP-SSM & \shortstack{100.00\\{\scriptsize[100.00, 100.00]}} & \shortstack{100.00\\{\scriptsize[100.00, 100.00]}} & \shortstack{99.98\\{\scriptsize[99.92, 100.00]}} & \shortstack{21.16\\{\scriptsize[20.69, 21.65]}} & \shortstack{19.72\\{\scriptsize[19.22, 20.17]}} & \shortstack{20.15\\{\scriptsize[19.74, 20.57]}} & \shortstack{20.20\\{\scriptsize[19.72, 20.63]}} \\
MLP & \shortstack{20.00\\{\scriptsize[19.59, 20.39]}} & \shortstack{20.00\\{\scriptsize[19.59, 20.40]}} & \shortstack{20.00\\{\scriptsize[19.61, 20.41]}} & \shortstack{20.00\\{\scriptsize[19.60, 20.39]}} & \shortstack{20.00\\{\scriptsize[19.59, 20.41]}} & \shortstack{20.00\\{\scriptsize[19.60, 20.39]}} & \shortstack{20.00\\{\scriptsize[19.62, 20.40]}} \\
A-DGN & \shortstack{100.00\\{\scriptsize[100.00, 100.00]}} & \shortstack{100.00\\{\scriptsize[100.00, 100.00]}} & \shortstack{99.82\\{\scriptsize[99.66, 99.94]}} & \shortstack{19.99\\{\scriptsize[19.55, 20.37]}} & \shortstack{19.98\\{\scriptsize[19.56, 20.38]}} & \shortstack{20.06\\{\scriptsize[19.64, 20.45]}} & \shortstack{20.08\\{\scriptsize[19.71, 20.48]}} \\
Tikhonov & \shortstack{100.00\\{\scriptsize[100.00, 100.00]}} & \shortstack{100.00\\{\scriptsize[100.00, 100.00]}} & \shortstack{100.00\\{\scriptsize[100.00, 100.00]}} & \shortstack{98.25\\{\scriptsize[96.26, 100.00]}} & \shortstack{29.44\\{\scriptsize[20.66, 43.78]}} & \shortstack{20.00\\{\scriptsize[19.59, 20.39]}} & \shortstack{20.00\\{\scriptsize[19.62, 20.38]}} \\
GRIT & \shortstack{100.00\\{\scriptsize[100.00, 100.00]}} & \shortstack{100.00\\{\scriptsize[100.00, 100.00]}} & \shortstack{100.00\\{\scriptsize[100.00, 100.00]}} & \shortstack{100.00\\{\scriptsize[100.00, 100.00]}} & \shortstack{100.00\\{\scriptsize[100.00, 100.00]}} & \shortstack{100.00\\{\scriptsize[100.00, 100.00]}} & \shortstack{100.00\\{\scriptsize[100.00, 100.00]}} \\
\bottomrule
\end{tabular}
\end{adjustbox}
\end{table}

\subsection{Sensitivity to the Tikhonov solver}
\label{app:tikhonov-replay}

This analysis checks whether numerical approximation affects Tikhonov's RangeProfileCopy results in Section~\ref{sec:gnn-families} (Figure~\ref{fig:architecture-summary}, top-right).
Tikhonov computes updated node features by solving a system of linear equations linking nodes across the graph.
The exact solution can use distant inputs, but each forward pass requires a numerical solve.
Limiting the number of solver iterations reduces computation but can change both predictions and their dependence on distant inputs.

We keep each validation-selected model fixed and compare the standard 30-iteration solver with a direct 64-bit solution and a solver stopped after five iterations.
This tests whether the standard solver adequately approximates the learned operator and how stopping earlier affects prediction error and measured influence range (Table~\ref{tab:tikhonov-replay}).

\begin{table}[h]
\centering
\small
\setlength{\tabcolsep}{3.5pt}
\caption{Effect of the Tikhonov solver on RangeProfileCopy. ``Direct solve'' is the 64-bit reference, the standard solver uses 30 iterations, and ``5 iterations'' deliberately stops early. Prediction rows report test NMSE and the relative difference between the standard and direct predictions; the final row reports the average resolvent $R_{90}$ in graph hops. Entries are means over three runs.}
\label{tab:tikhonov-replay}
\begin{tabular}{llrlrrr}
\toprule
Task & Quantity & \shortstack{Standard--direct\\difference} & Metric & \shortstack{Direct\\solve} & \shortstack{Standard\\solver} & \shortstack{5\\iterations} \\
\midrule
RangeProfileCopy & Local & .005 & NMSE & .0926 & .0925 & .196 \\
 & Inverse-distance & $3.8{\times}10^{-5}$ & NMSE & .164 & .164 & .242 \\
 & Resolvent & .012 & NMSE & .000117 & .000260 & .240 \\
 & Resolvent influence & -- & $R_{90}$ & 20.1 & 18.8 & 38.8 \\
\bottomrule
\end{tabular}
\end{table}

The standard solver produces predictions close to the direct reference. Stopping after five iterations sharply increases both prediction error and the measured resolvent range. Thus, a broader influence profile can reflect solver error rather than more accurate use of distant inputs.

\subsection{Architectural-contrast and precision results}
\label{app:mechanism-precision}

\begin{table}[h!]
\centering
\small
\setlength{\tabcolsep}{2pt}
\caption{Architectural contrasts on hard-resolvent RangeProfileCopy. From top to bottom, the pairs vary GatedGCN depth and width, ChebNet stabilization, fixed versus adaptive depth, GraphGPS global attention, and MP-SSM state recurrence. NMSE is pooled over the test set; influence statistics average 2,000 topology--size--target cases per run. Values are means over three independent runs; entries displaying $\pm$ include the sample standard deviation.}
\label{tab:e2-architectural-contrasts}
\begin{tabular}{@{}lrrrrrr@{}}
\toprule
Control $\rightarrow$ full & \multicolumn{1}{c}{\shortstack{NMSE\\control}} & \multicolumn{1}{c}{\shortstack{NMSE\\full}} & \multicolumn{1}{c}{\shortstack{$R_{90}$\\control}} & \multicolumn{1}{c}{\shortstack{$R_{90}$\\full}} & $\Delta$ cosine & Full $M_{>20}$ \\
\midrule
GatedGCN-5 $\rightarrow$ GatedGCN-20 & $0.458{\scriptstyle\pm0.012}$ & $0.126{\scriptstyle\pm0.006}$ & $4.19{\scriptstyle\pm0.01}$ & $10.06{\scriptstyle\pm0.12}$ & $+0.1313$ & $0.000$ \\
Vanilla ChebNet $\rightarrow$ Stable-ChebNet & $0.024{\scriptstyle\pm0.001}$ & $0.023{\scriptstyle\pm0.000}$ & $13.75{\scriptstyle\pm0.07}$ & $13.78{\scriptstyle\pm0.07}$ & $-0.0001$ & $0.000$ \\
GatedGCN-20 $\rightarrow$ AMP & $0.126{\scriptstyle\pm0.006}$ & $0.456{\scriptstyle\pm0.016}$ & $10.06{\scriptstyle\pm0.12}$ & $3.65{\scriptstyle\pm0.19}$ & $-0.1551$ & $0.000$ \\
GraphGPS-local $\rightarrow$ GraphGPS & $0.542{\scriptstyle\pm0.012}$ & $0.265{\scriptstyle\pm0.001}$ & $3.66{\scriptstyle\pm0.09}$ & $48.00{\scriptstyle\pm0.09}$ & $+0.0932$ & $0.385$ \\
MP-SSM-local $\rightarrow$ MP-SSM-20 & $0.823{\scriptstyle\pm0.003}$ & $0.160{\scriptstyle\pm0.002}$ & $1.00{\scriptstyle\pm0.00}$ & $7.53{\scriptstyle\pm0.04}$ & $+0.4072$ & $0.000$ \\
\bottomrule
\end{tabular}
\end{table}

\begin{table}[h!]
\centering
\small
\setlength{\tabcolsep}{3pt}
\caption{BF16 precision replay relative to FP64 on hard-resolvent RangeProfileCopy. Each run replays 256 graphs and 512 influence targets using the same trained parameters. Values are means over three independent runs; entries displaying $\pm$ include the sample standard deviation.}
\label{tab:e2-precision}
\begin{tabular}{@{}lrrrrr@{}}
\toprule
Model & \multicolumn{1}{c}{\shortstack{Output drift\\(\%)}} & $\Delta$NMSE & \multicolumn{1}{c}{\shortstack{$R_{90}$ retention\\(\%)}} & \multicolumn{1}{c}{\shortstack{Exact $R_{90}$\\(\%)}} & $\Delta\mathcal{E}_{\rm tail}$ \\
\midrule
GatedGCN-20 & $3.78{\scriptstyle\pm0.11}$ & $-0.0012$ & $100.02{\scriptstyle\pm0.15}$ & $89.5{\scriptstyle\pm3.0}$ & $+0.000$ \\
Stable-ChebNet & $17.76{\scriptstyle\pm2.69}$ & $+0.0768$ & $99.96{\scriptstyle\pm0.05}$ & $95.9{\scriptstyle\pm1.8}$ & $+0.000$ \\
AMP & $4.42{\scriptstyle\pm1.53}$ & $+0.0001$ & $99.98{\scriptstyle\pm0.12}$ & $99.5{\scriptstyle\pm0.1}$ & $+0.000$ \\
GraphGPS & $1.56{\scriptstyle\pm0.10}$ & $-0.0001$ & $100.00{\scriptstyle\pm0.03}$ & $95.4{\scriptstyle\pm1.6}$ & $-0.037$ \\
MP-SSM-20 & $7.22{\scriptstyle\pm0.14}$ & $+0.0205$ & $99.81{\scriptstyle\pm0.24}$ & $98.4{\scriptstyle\pm2.2}$ & $+0.000$ \\
SONAR & $5.35{\scriptstyle\pm0.34}$ & $+0.0037$ & $100.03{\scriptstyle\pm0.05}$ & $98.0{\scriptstyle\pm0.7}$ & $+0.000$ \\
GRIT & $1.67{\scriptstyle\pm0.25}$ & $-0.0001$ & $99.99{\scriptstyle\pm0.16}$ & $88.5{\scriptstyle\pm3.1}$ & $-0.002$ \\
\bottomrule
\end{tabular}
\end{table}

Table~\ref{tab:e2-architectural-contrasts} gives the numerical results for the five architectural comparisons in Section~\ref{sec:gnn-families} (Figure~\ref{fig:architecture-summary}, bottom panels).
Alongside prediction error and $R_{90}$, it reports the change in cosine similarity to the target distance profile (higher is better) and the fraction of influence beyond 20 hops, $M_{>20}$ (Appendix~\ref{app:range}).
GraphGPS assigns 38.5\% of its influence beyond 20 hops, yet its NMSE (0.265) exceeds Stable-ChebNet's (0.023), whose influence is confined to 20 hops. Thus, using more distant inputs does not by itself ensure lower prediction error.

Table~\ref{tab:e2-precision} tests whether the same trained models retain their predictions and measured range when evaluated in BF16 instead of FP64.
Mean $R_{90}$ retention is near 100\% for all seven models, but Stable-ChebNet has 17.76\% output drift and an NMSE increase of 0.0768; MP-SSM has 7.22\% drift and an NMSE increase of 0.0205.
A nearly unchanged range statistic can therefore conceal a loss of prediction accuracy.
Output drift and range retention are defined in Appendix~\ref{app:range}; ``Exact $R_{90}$'' is the percentage of evaluated targets whose radius is unchanged.

\begin{figure}[h]
\centering
\includegraphics[width=\textwidth]{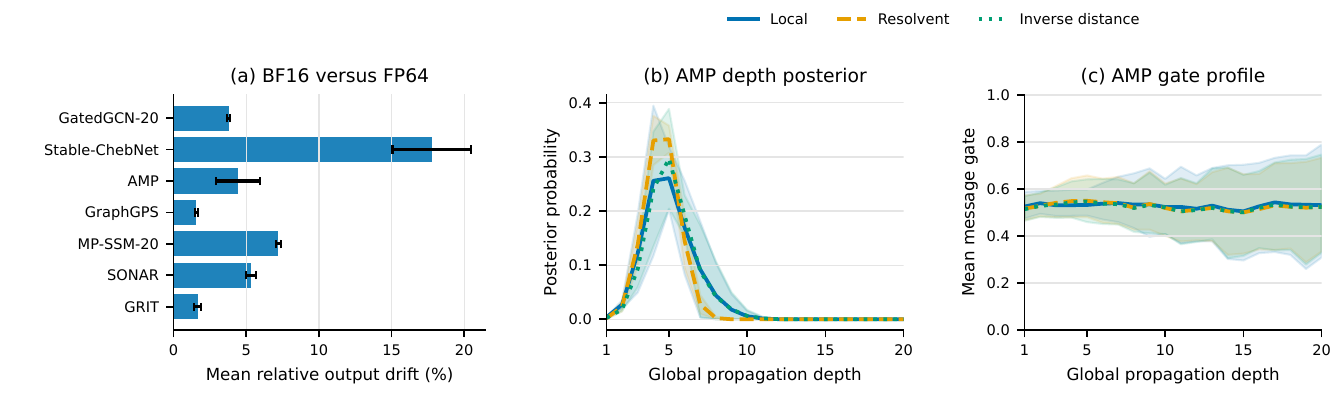}
\caption{Complementary hard-resolvent and AMP diagnostics. (a) BF16 output drift relative to FP64 for different architectures; bars show sample standard deviations across runs. (b) AMP's learned global depth posterior; lines are run means and bands are sample standard deviations. (c) AMP's mean message-gate profile; lines average run-level means and bands average the within-run 10th--90th percentiles.}
\label{fig:e2-precision-amp}
\end{figure}

Figure~\ref{fig:e2-precision-amp}(a) plots the output drift reported in Table~\ref{tab:e2-precision}.
Panels (b) and (c) examine AMP. In the main experiment, AMP has a shorter realized range and higher prediction error than the fixed-depth GatedGCN-20 on the hard-resolvent target. Panel (b) helps explain this result: although AMP can use depths 1--20, its learned global posterior places most mass on depths 3--7. The same preference for shallow computation appears for all three RangeProfileCopy transformations: local two-hop averaging, the hard resolvent, and inverse-distance weighting. Panel (c) shows that the mean channelwise message gate remains close to 0.5 across layers, while individual gate values become more dispersed at later layers. Differences among tasks are therefore more visible in AMP's depth posterior than in its average gate strength.

\section{Molecular Validation}
\label{app:molecular}

This appendix provides the models, training protocol, and full results for the molecular study introduced in Section~\ref{sec:remote-pairs}. Models are trained on ECHO-Charge-noCOM and evaluated both on its full test split and with RemotePairs, which applies the trained models without further training and is defined in Appendix~\ref{app:echo-remotepairs}.

\subsection{Models and training}

\begin{table}[ht]
\centering
\small
\setlength{\tabcolsep}{4pt}
\caption{Model sizes and learning rates for molecular training. Width and total trainable parameter counts include the encoder and prediction head.}
\label{tab:molecular-training}
\begin{tabular}{lrrc}
\toprule
Model & Width & Parameters & Learning rate \\
\midrule
GatedGCN-3 & 176 & 506,886 & $10^{-3}$ \\
GatedGCN-20 & 72 & 542,958 & $10^{-3}$ \\
Stable-ChebNet & 208 & 499,518 & $10^{-3}$ \\
GraphGPS & 96 & 503,526 & $10^{-3}$ \\
MP-SSM-20 & 320 & 501,126 & $3\times10^{-4}$ \\
A-DGN & 424 & 501,598 & $10^{-3}$ \\
Tikhonov & 528 & 501,581 & $10^{-3}$ \\
SONAR & 264 & 488,407 & $10^{-3}$ \\
GRIT & 104 & 511,062 & $3\times10^{-4}$ \\
\bottomrule
\end{tabular}
\end{table}

Atomic number receives a categorical embedding. For bond type $t$ and quantized bond-length midpoint $b$, the shared scalar edge encoding is $\operatorname{softplus}(e_t+wb)+10^{-6}$, where $\operatorname{softplus}(z)=\log(1+e^z)$ and $e_t$ and $w$ are learned scalars; GatedGCN/GraphGPS use it as a one-dimensional edge state, Stable-ChebNet/MP-SSM/A-DGN use it as an adjacency weight, and Tikhonov uses it in both qCheb and its normalized Laplacian. SONAR multiplies it by the learned directed edge weight. GRIT projects it into bonded pair states and adds the random-walk encoding. We compare GatedGCN-3, GatedGCN-20, Stable-ChebNet with $K=20$ and fixed largest normalized-Laplacian eigenvalue $\lambda_{\max}=2$, four-block global GraphGPS, 20-update MP-SSM, 20-iteration A-DGN, the degree-5 Tikhonov layer with 50 PCG iterations, 20-update SONAR, and four-layer GRIT. Results use three independent training runs per model.

All models use zero weight decay and minimize the base-10 logarithm of the atom-weighted minibatch MSE; epoch summaries average batch losses with atom-count weights. Model selection uses validation loss only, with test and RemotePairs analyses performed on the best checkpoint. Table~\ref{tab:molecular-training} reports model sizes and selected learning rates. Training uses Adam with batches of 256 and no scheduler for GatedGCN, Stable-ChebNet, GraphGPS, MP-SSM, A-DGN, Tikhonov and SONAR, and AdamW with batches of 32 for GRIT.

SONAR and GRIT use gradient clipping at norm 1, a 60-epoch minimum, and early stopping after 15 epochs without a validation-loss improvement of 0.002; other models use no clipping. SONAR additionally screens integration steps $\{0.01,0.005,0.001\}$ at learning rate $10^{-3}$, selecting $h=0.01$, with a 200-epoch final cap. GRIT screens rates $\{3\times10^{-4},10^{-3}\}$ for 60 epochs with attention dropout 0.2. It uses FP32, evaluation batches of 32, and a ten-epoch warm-up followed by cosine decay over a fixed 200-epoch horizon in both screening and finals. Its final epoch caps are 200/100/100, with 196/100/100 epochs completed.

\subsection{ECHO-Charge-noCOM and RemotePairs results}
\label{app:molecular-results}

For every model, full and cropped predictions use the same checkpoint in evaluation mode, with graph-dependent quantities computed from the supplied graph. For GRIT, these include random-walk features and node degrees; BatchNorm statistics remain fixed at their checkpoint values. The comparison measures the predictive benefit of remote context for each complete model.

\begin{table}[t]
\centering
\small
\caption{RemotePairs full-molecule and radius-crop intervention. Intervals are paired hierarchical 95\% bootstrap intervals.}
\label{tab:e3-remote-context-v2}
\begin{tabular}{lrrrr}
\toprule
Model & Test MAE & Full acc. & Crop acc. & Full $-$ crop \\ 
\midrule
GRIT & 0.0059 & 0.690 [0.651, 0.726] & 0.500 [0.500, 0.500] & 0.190 [0.151, 0.226] \\
GatedGCN-20 & 0.0060 & 0.624 [0.584, 0.663] & 0.500 [0.500, 0.500] & 0.124 [0.084, 0.163] \\
GraphGPS & 0.0064 & 0.639 [0.595, 0.685] & 0.500 [0.500, 0.500] & 0.139 [0.095, 0.185] \\
A-DGN & 0.0066 & 0.638 [0.598, 0.678] & 0.500 [0.500, 0.500] & 0.138 [0.098, 0.178] \\
GatedGCN-3 & 0.0069 & 0.507 [0.497, 0.517] & 0.500 [0.500, 0.500] & 0.007 [-0.003, 0.017] \\
SONAR & 0.0077 & 0.615 [0.573, 0.653] & 0.500 [0.500, 0.500] & 0.115 [0.073, 0.153] \\
Stable-ChebNet & 0.0078 & 0.581 [0.521, 0.636] & 0.500 [0.497, 0.502] & 0.081 [0.021, 0.136] \\
MP-SSM & 0.0085 & 0.597 [0.558, 0.634] & 0.500 [0.496, 0.500] & 0.098 [0.060, 0.135] \\
Tikhonov & 0.0090 & 0.570 [0.525, 0.613] & 0.500 [0.499, 0.501] & 0.070 [0.025, 0.113] \\
\bottomrule
\end{tabular}
\end{table}

\begin{table}[t]
\centering
\small
\caption{RemotePairs balanced charge-order accuracy. Entries are medians with paired hierarchical 95\% bootstrap intervals. The predeclared claim rule uses the full-molecule median.}
\label{tab:e3-remote-context-v2-balanced}
\begin{tabular}{lrr}
\toprule
Model & Full balanced acc. & Crop balanced acc. \\
\midrule
A-DGN & 0.640 [0.600, 0.680] & 0.500 [0.500, 0.500] \\
GRIT & 0.687 [0.648, 0.724] & 0.500 [0.500, 0.500] \\
GatedGCN-20 & 0.624 [0.584, 0.663] & 0.500 [0.500, 0.500] \\
GatedGCN-3 & 0.507 [0.497, 0.517] & 0.500 [0.500, 0.500] \\
GraphGPS & 0.634 [0.591, 0.680] & 0.500 [0.500, 0.500] \\
MP-SSM & 0.594 [0.555, 0.631] & 0.500 [0.496, 0.500] \\
SONAR & 0.613 [0.572, 0.652] & 0.500 [0.500, 0.500] \\
Stable-ChebNet & 0.575 [0.511, 0.633] & 0.499 [0.496, 0.502] \\
Tikhonov & 0.569 [0.523, 0.612] & 0.500 [0.498, 0.501] \\
\bottomrule
\end{tabular}
\end{table}

\begin{table}[t]
\centering
\scriptsize
\caption{RemotePairs intervention by matching radius, pooled across the three seeds.}
\label{tab:e3-remote-context-v2-radius}
\begin{tabular}{llrrrr}
\toprule
Model & Radius & Pairs & Full acc. & Crop acc. & Full $-$ crop \\ 
\midrule
A-DGN & 2 & 348 & 0.730 & 0.500 & +0.230 \\
A-DGN & 3 & 207 & 0.626 & 0.500 & +0.126 \\
A-DGN & 4 & 165 & 0.558 & 0.500 & +0.058 \\
GRIT & 2 & 348 & 0.767 & 0.500 & +0.267 \\
GRIT & 3 & 207 & 0.652 & 0.500 & +0.152 \\
GRIT & 4 & 165 & 0.648 & 0.500 & +0.148 \\
GatedGCN-20 & 2 & 348 & 0.744 & 0.500 & +0.244 \\
GatedGCN-20 & 3 & 207 & 0.589 & 0.500 & +0.089 \\
GatedGCN-20 & 4 & 165 & 0.539 & 0.500 & +0.039 \\
GatedGCN-3 & 2 & 348 & 0.520 & 0.500 & +0.020 \\
GatedGCN-3 & 3 & 207 & 0.500 & 0.500 & +0.000 \\
GatedGCN-3 & 4 & 165 & 0.500 & 0.500 & +0.000 \\
GraphGPS & 2 & 348 & 0.704 & 0.500 & +0.204 \\
GraphGPS & 3 & 207 & 0.609 & 0.500 & +0.109 \\
GraphGPS & 4 & 165 & 0.606 & 0.500 & +0.106 \\
MP-SSM & 2 & 348 & 0.651 & 0.500 & +0.151 \\
MP-SSM & 3 & 207 & 0.601 & 0.500 & +0.101 \\
MP-SSM & 4 & 165 & 0.536 & 0.497 & +0.039 \\
SONAR & 2 & 348 & 0.710 & 0.500 & +0.210 \\
SONAR & 3 & 207 & 0.623 & 0.500 & +0.123 \\
SONAR & 4 & 165 & 0.506 & 0.500 & +0.006 \\
Stable-ChebNet & 2 & 348 & 0.658 & 0.501 & +0.157 \\
Stable-ChebNet & 3 & 207 & 0.541 & 0.498 & +0.043 \\
Stable-ChebNet & 4 & 165 & 0.542 & 0.500 & +0.042 \\
Tikhonov & 2 & 348 & 0.615 & 0.500 & +0.115 \\
Tikhonov & 3 & 207 & 0.531 & 0.500 & +0.031 \\
Tikhonov & 4 & 165 & 0.561 & 0.500 & +0.061 \\
\bottomrule
\end{tabular}
\end{table}

Table~\ref{tab:e3-remote-context-v2} reports quantized-input atom MAE and the full-versus-crop intervention. Full-test MAE is the mean across the three validation-selected checkpoints. Accuracy entries are medians with paired hierarchical 95\% intervals. The cropped condition stays at the $0.5$ local baseline up to small floating-point deviations, while the full condition is higher for every model except the depth-three negative control under the predeclared claim rule.
The intervals compare each model with its own crop, not one architecture against another.

Table~\ref{tab:e3-remote-context-v2-balanced} reports balanced accuracy for identifying which atom in a matched pair has the larger charge.
At each matching radius, we average the accuracy on pairs where the left atom has the larger charge and the accuracy on pairs where the right atom does.
Always choosing the same side therefore scores $0.5$, even if that side has the larger charge more often.
We count a result as evidence that remote context helps only if the full-molecule median balanced accuracy exceeds $0.5$ and the entire 95\% interval for the gain in ordinary accuracy over cropped inputs is above zero.
The cropped balanced accuracies are shown for comparison only.

Table~\ref{tab:e3-remote-context-v2-radius} reports the intervention by matching radius, pooled across the three seeds while preserving full/crop pairing. A positive average gain need not hold at every radius: SONAR's full-molecule accuracy is 50.6\% at radius four. The depth-three GatedGCN can use context beyond a radius-two crop, but at radii three and four both full and cropped paired predictions agree within $1.2\times10^{-7}$ and tie-aware accuracy is exactly $0.5$ for every seed, as expected from its three-hop support.

\section{Theoretical Results and Proofs}
\label{app:proofs}

\subsection{Proof of Proposition~\ref{prop:support}}

\begin{proof}
For every node $v$, let $\mathcal N(v)$ be its neighbor set and let $\mathcal D_t(v)$ be the set of input nodes on which its representation can depend after at most $t$ sequential one-hop aggregation operations. Initially, $\mathcal D_0(v)=\{v\}$.
A node-wise operation does not change this set. A one-hop aggregation at $v$ can use only the current representations of $v$ and its neighbors, hence
\[
\mathcal D_{t+1}(v)
\subseteq
\bigcup_{w\in\mathcal N(v)\cup\{v\}}
\mathcal D_t(w).
\]
If every node in $\mathcal D_t(w)$ is at distance at most $t$ from $w$, then every such node is at distance at most $t+1$ from $v$. Induction therefore gives $\mathcal D_T(v)\subseteq\{u:d(u,v)\leq T\}$.
Residual connections, pointwise nonlinearities, channel mixing, and parallel branches take unions or transformations of existing dependency sets and do not enlarge their radius. Parallel branches increase the radius only when their outputs are subsequently composed through additional cross-node operations.
\end{proof}

\subsection{Proof of Theorem~\ref{thm:truncation}}

\begin{proof}
Because $\Phi_{\mat{X}}$ is a contraction on a finite-dimensional normed space, the Banach fixed-point theorem gives a unique fixed point $\mat{H}^\star=\Phi_{\mat{X}}(\mat{H}^\star)$.
Moreover,
\[
\norm{\mat{H}_{t+1}-\mat{H}^\star}
=
\norm{\Phi_{\mat{X}}(\mat{H}_t)-\Phi_{\mat{X}}(\mat{H}^\star)}
\leq
\theta\norm{\mat{H}_t-\mat{H}^\star}.
\]
Applying the inequality recursively yields $\norm{\mat{H}_T-\mat{H}^\star}\leq \theta^T\norm{\mat{H}_0-\mat{H}^\star}$.

In the affine case, contractivity with factor $\theta$ implies $\norm{\mat{M}}\leq\theta<1$. Therefore, $\Id-\mat{M}$ is invertible and $\mat{R}=(\Id-\mat{M})^{-1}\mat{B}=\sum_{t=0}^{\infty}\mat{M}^t\mat{B}$.
Unrolling $T$ updates from $\mat{H}_0=\mat{C}\mat{X}$ gives $\mat{P}_T=\mat{M}^T\mat{C}+\sum_{t=0}^{T-1}\mat{M}^t\mat{B}$.
Therefore,
\begin{align*}
\mat{R}-\mat{P}_T
&=
\sum_{t=T}^{\infty}\mat{M}^t\mat{B}-\mat{M}^T\mat{C}\\
&=
\mat{M}^T\left(
\sum_{s=0}^{\infty}\mat{M}^s\mat{B}-\mat{C}
\right)\\
&=
\mat{M}^T(\mat{R}-\mat{C}),
\end{align*}
which proves Equation~\eqref{eq:finite-gap-identity}. Taking operator norms gives $\norm{\mat{R}-\mat{P}_T}\leq\norm{\mat{M}^T}\norm{\mat{R}-\mat{C}}$, and submultiplicativity yields $\norm{\mat{M}^T}\leq\theta^T$.
\end{proof}

The contraction assumption is sufficient but not necessary for a linear fixed point. If only $\rho(\mat{M})<1$ is known, where $\rho(\mat{M})$ is the spectral radius of $\mat{M}$, Equation~\eqref{eq:finite-gap-identity} remains valid, but $\norm{\mat{M}^T}$ can increase transiently when $\mat{M}$ is non-normal before eventually converging to zero.

\subsection{Proof of Corollary~\ref{cor:restart}}
\label{app:restart}

\begin{proof}
Let $0<\beta=|a|<1$. Unrolling the recurrence in Corollary~\ref{cor:restart} from $\mat{H}_0=\mat{X}$ gives $\mat{H}_T=\mat{P}_{a,T}\mat{X}$, where
\[
\mat{P}_{a,T}
=
(a\mat{S})^T
+
(1-\beta)\sum_{t=0}^{T-1}(a\mat{S})^t.
\]
Since $\norm{a\mat{S}}_2<1$, $\mat{R}_a=(1-\beta)\sum_{t=0}^{\infty}(a\mat{S})^t$.
Subtracting and using the geometric series gives
\begin{align*}
\mat{R}_a-\mat{P}_{a,T}
&=
(1-\beta)(a\mat{S})^T(\Id-a\mat{S})^{-1}-(a\mat{S})^T\\
&=
(a\mat{S})^T(a\mat{S}-\beta\Id)(\Id-a\mat{S})^{-1}.
\end{align*}

Because $\mat{S}$ is symmetric, let $\operatorname{spec}(\mat{S})$ denote its spectrum and let $\lambda$ denote an eigenvalue in that spectrum. The spectral theorem gives
\[
\norm{\mat{R}_a-\mat{P}_{a,T}}_2
=
\max_{\lambda\in\operatorname{spec}(\mat{S})}
|a\lambda|^T
\frac{|a\lambda-\beta|}
{|1-a\lambda|}.
\]
Write $a=s\beta$, with $s\in\{-1,1\}$, and set $x=s\lambda\in[-1,1]$. The scalar response becomes
\[
\beta^{T+1}|x|^T\frac{1-x}{1-\beta x}
\leq
\beta^{T+1}\frac{1-x}{1-\beta x},
\]
because $|x|\leq1$. The function $x\mapsto(1-x)/(1-\beta x)$ is non-increasing on $[-1,1]$, and therefore
\[
\norm{\mat{R}_a-\mat{P}_{a,T}}_2
\leq
\frac{2\beta^{T+1}}{1+\beta}.
\]
If $-\operatorname{sign}(a)\in\operatorname{spec}(\mat{S})$, where $\operatorname{sign}(a)\in\{-1,1\}$ denotes the sign of the nonzero coefficient $a$, then $x=-1$ is available, $|x|^T=1$, and the bound in Equation~\eqref{eq:truncation} is attained.
\end{proof}

Solving $2\beta^{T+1}/(1+\beta)\leq\varepsilon$ gives the iteration threshold
\begin{equation}
T
\geq
\frac{
\log\!\big(\varepsilon(1+\beta)/2\big)
}{
\log \beta
}
-1,
\label{eq:restart-iterations}
\end{equation}
rounded up to the next integer.

\subsection{Transport assumptions and a general range theorem}
\label{app:transport-assumptions}

Sections~\ref{app:transport-assumptions}--\ref{app:modeproof} provide the formal development behind the theoretical claims summarized in Section~\ref{sec:theory}. The argument proceeds in two parts. Sections~\ref{app:transport-assumptions}--\ref{app:cheb-acceleration} study how a finite update budget changes the spatial range of diffusion: we first state a general transport theorem, specialize it to restart diffusion, sharpen it on the path, and then show how a signed Chebyshev polynomial avoids the diffusive slowdown. Sections~\ref{app:precision} and~\ref{app:modeproof} turn to representation, explaining how conditioning, cancellation, and the number of stable real modes constrain accurate realization.

We begin with the general question suggested by the random-walk discussion in Section~\ref{sec:theory}: if an ideal positive operator averages walks of a characteristic length, how much of its range can a finite implementation retain? The answer should depend both on the available walk depth and on how quickly a walk spreads on the graph. The transport assumption in Equation~\eqref{eq:transport-assumption} and the characteristic-scale assumption in Equation~\eqref{eq:depth-scale} separate these two ingredients. Appendix~\ref{app:transport-assumptions} also specifies the pre-mixing regime needed on finite graphs, where no distance scale can grow indefinitely. Theorem~\ref{thm:transport}, introduced here, then converts temporal depth into spatial range.

Let $\mat{P}$ be a nonnegative row-stochastic operator with $P_{vu}=0$ whenever $d(u,v)>1$, let $d_w\geq1$ be the \emph{walk dimension}, and let $Z_t$ be its random walk started at $v$. Write
\[
F_t^v(r)=\Pr_v\!\left(d(v,Z_t)\leq r\right),
\qquad
Q_q(t;v)=\min\{r:F_t^v(r)\geq q\}.
\]
To include growing finite graphs without claiming that transport continues after saturation, associate each graph with a transport horizon $h_{\set{G}}\in(0,\infty]$. The \emph{quantile transport assumption} means that, for every fixed $q\in(0,1)$, there are constants $0<c_q\leq C_q<\infty$ and $t_q<\infty$, uniform over the nodes and graph family under consideration, such that
\begin{equation}
c_q t^{1/d_w}
\leq
Q_q(t;v)
\leq
C_q t^{1/d_w}
\qquad
(t_q\leq t\leq h_{\set{G}}).
\label{eq:transport-assumption}
\end{equation}
On infinite graphs take $h_{\set{G}}=\infty$. On finite graph families, Theorem~\ref{thm:transport} is an asymptotic statement along sequences with $m/h_{\set{G}}\to0$; its arbitrary $T$-step mixture claim likewise requires $T/h_{\set{G}}\to0$. These pre-mixing conditions keep every fixed-constant neighborhood of the relevant depth below the transport-saturation scale. Two-sided Gaussian heat-kernel estimates provide standard settings satisfying Equation~\eqref{eq:transport-assumption} with $d_w=2$~\citep{delmotte1999harnack}, while sub-Gaussian estimates provide standard anomalous-diffusion settings with $d_w>2$~\citep{barlow2005subgaussian}.

A nonnegative integer depth $N_m$, independent of the walk, has \emph{characteristic scale} $m$ if, for every $\eta\in(0,1)$, there are $0<a_\eta\leq b_\eta<\infty$ such that
\begin{equation}
\Pr\!\left(a_\eta m\leq N_m\leq b_\eta m\right)
\geq
1-\eta
\label{eq:depth-scale}
\end{equation}
for all sufficiently large $m$. Thus $N_m/m$ is both tight and bounded away from zero in probability. Constants hidden by $O(\cdot)$ and $\Theta(\cdot)$ in Theorem~\ref{thm:transport} and its proof may depend on the fixed quantile but are uniform over the admissible nodes and graph family.

\paragraph{Transfer to symmetric normalization.}
Theorem~\ref{thm:transport} is stated for the random-walk normalization. Its range scaling also transfers to the symmetric normalization used in the experiments under a bounded degree-ratio condition. Indeed, for an undirected graph with $\mat{P}=\mat{D}^{-1}\mat{A}$ and $\mat{S}=\mat{D}^{-1/2}\mat{A}\mat{D}^{-1/2}$,
\[
\mat{S}^t
=
\mat{D}^{1/2}\mat{P}^t\mat{D}^{-1/2}.
\]
For normalized weights $w_t\geq0$, let $\mu_v(u)=\sum_t w_t(\mat P^t)_{vu}$ be the random-walk endpoint law. After normalizing the corresponding positive influence row of $\sum_t w_t\mat S^t$, its distribution is
\begin{equation}
\widetilde\mu_v(u)
=
\frac{\mu_v(u)d_u^{-1/2}}
{\sum_x\mu_v(x)d_x^{-1/2}},
\label{eq:symmetric-row-profile}
\end{equation}
because the common factor $\sqrt{d_v}$ cancels. If $d_{\max}/d_{\min}\leq C$ uniformly over the graph family, then $\frac{d\widetilde\mu_v}{d\mu_v}\in[C^{-1/2},C^{1/2}]$. Writing $R_q(\nu)$ for the radius of the smallest distance ball around $v$ containing probability $q$ under a node law $\nu$, consequently
\begin{equation}
R_{q/\sqrt C}(\mu_v)
\leq
R_q(\widetilde\mu_v)
\leq
R_{1-(1-q)/\sqrt C}(\mu_v).
\label{eq:symmetric-quantile-transfer}
\end{equation}
Both comparison levels are fixed in $(0,1)$, so the walk dimension and range scaling are unchanged. Paths and ladders satisfy the degree-ratio condition. The transfer in Equation~\eqref{eq:symmetric-quantile-transfer} does not cover negative poles or a signed combination of branches. The exact total-variation bound in Equation~\eqref{eq:mixture-tv} is stated for the row-stochastic operator $\mat P$; Equation~\eqref{eq:symmetric-quantile-transfer} transfers its range scaling but does not by itself give the same TV constant after symmetric reweighting.

For an independent depth $N_m$ of characteristic scale $m$, define
\[
\mat K_m=\mathbb E[\mat P^{N_m}],
\qquad
\mat K_{m,T}=\mathbb E[\mat P^{N_m\wedge T}].
\]
Write $p_m^{\,v}=p_{\mat K_m}^{\,v}$ and $p_{m,T}^{\,v}=p_{\mat K_{m,T}}^{\,v}$ for their distance profiles. All two-sided asymptotic statements in Theorem~\ref{thm:transport} take $m\to\infty$. For the capped operator, let $T=T_m$ and require
\begin{equation}
s_m:=\min\{m,T_m\}\longrightarrow\infty.
\label{eq:capped-scale-divergence}
\end{equation}
On finite graph families also require $m/h_{\set{G}}\to0$. The arbitrary $T$-step mixture statement is taken as $T\to\infty$, with $T/h_{\set{G}}\to0$ on finite graph families.

\begin{theorem}[Depth limits the range of positive diffusion]
\label{thm:transport}
Under the quantile transport assumption in Equation~\eqref{eq:transport-assumption} and the characteristic-scale assumption in Equation~\eqref{eq:depth-scale}, for every fixed $q\in(0,1)$, as $m\to\infty$ and, for the capped realization, along sequences satisfying Equation~\eqref{eq:capped-scale-divergence},
\begin{equation}
R_q(\mat K_m;v)=\Theta\!\left(m^{1/d_w}\right),
\qquad
R_q(\mat K_{m,T};v)
=\Theta\!\left(\min\{m,T\}^{1/d_w}\right).
\label{eq:transport-range-laws}
\end{equation}
Moreover, as $T\to\infty$, any normalized nonnegative mixture $\mat L_T=\sum_{t=0}^{T}w_t\mat P^t$ satisfies $R_q(\mat L_T;v)=O(T^{1/d_w})$. Hence, if $R_q(\mat K_{m,T};v)\geq\kappa R_q(\mat K_m;v)$ for any fixed $\kappa>0$, then $T=\Omega(m)$; choosing $T=\Theta(m)$ makes the capped range a nonvanishing fraction of the ideal range. In addition,
\begin{equation}
\norm{p_{m,T}^{\,v}-p_m^{\,v}}_{\mathrm{TV}}
\leq
\Pr(N_m>T).
\label{eq:mixture-tv}
\end{equation}
Consequently, whenever $\Pr(N_m>T)\leq\varepsilon<\min\{q,1-q\}$,
\begin{equation}
R_{q-\varepsilon}(\mat K_m;v)
\leq
R_q(\mat K_{m,T};v)
\leq
R_{q+\varepsilon}(\mat K_m;v).
\label{eq:mixture-quantile-sandwich}
\end{equation}
All statements are uniform over the admissible nodes and graph family in the divergence regime of Equation~\eqref{eq:capped-scale-divergence} and the finite-graph pre-mixing regime specified in Appendix~\ref{app:transport-assumptions}.
\end{theorem}

\begin{proof}[Proof of Theorem~\ref{thm:transport}]
Set $\gamma=1/d_w$. We first note that capping preserves a characteristic depth scale. On the event in Equation~\eqref{eq:depth-scale}, with $s=\min\{m,T\}$,
\[
\min\{a_\eta,1\}s
\leq
N_m\wedge T
\leq
\max\{b_\eta,1\}s.
\]
Thus the ideal depth has scale $s=m$, and the capped depth has scale $s=\min\{m,T\}$. By Equation~\eqref{eq:capped-scale-divergence}, these scales diverge, so their fixed positive multiples eventually exceed every threshold $t_q$ used in the quantile transport assumption.

The same quantile argument handles both cases. Fix $q\in(0,1)$, and choose $q_-<q<q_+$ and $\eta>0$ so that
\[
(1-\eta)q_-+\eta<q,
\qquad
(1-\eta)q_+>q.
\]
Let $Y_s$ denote either $N_m$ or $N_m\wedge T$. With probability at least $1-\eta$, it lies in $[\underline a s,\overline b s]$ for positive constants $\underline a$ and $\overline b$. For every integer
$r<c_{q_-}(\underline a s)^\gamma$, the lower bound in Equation~\eqref{eq:transport-assumption} gives $F_{Y_s}^v(r)<q_-$ on this event. Consequently, the mixture distribution satisfies
\[
\Pr_v\!\left(d(v,Z_{Y_s})\leq r\right)
\leq
(1-\eta)q_-+\eta
<q.
\]
This gives $R_q(\mathbb E[\mat P^{Y_s}];v)=\Omega(s^\gamma)$. Conversely, for
$r=\lceil C_{q_+}(\overline b s)^\gamma\rceil$, the upper transport bound gives $F_{Y_s}^v(r)\geq q_+$ throughout the same event, and hence
\[
\Pr_v\!\left(d(v,Z_{Y_s})\leq r\right)
\geq
(1-\eta)q_+
>q.
\]
This gives the matching $O(s^\gamma)$ bound and proves both range laws in Equation~\eqref{eq:transport-range-laws}. In particular,
\[
\frac{R_q(\mat K_{m,T};v)}{R_q(\mat K_m;v)}
=
\Theta\!\left(
\min\left\{1,\left(\frac{T}{m}\right)^{1/d_w}\right\}
\right).
\]

For an arbitrary normalized nonnegative mixture $\mat L_T=\sum_{t=0}^{T}w_t\mat P^t$, choose $q_+\in(q,1)$. Equation~\eqref{eq:transport-assumption} supplies one radius $O(T^\gamma)$ containing at least $q_+$ of every fixed-time endpoint law for $t_q\leq t\leq T$; the one-hop property handles the finitely many smaller times. Averaging gives $R_q(\mat L_T;v)=O(T^\gamma)$. If this range is at least $\kappa R_q(\mat K_m;v)=\Omega(m^\gamma)$ for any fixed $\kappa>0$, then $T=\Omega(m)$. Taking $T$ to be any fixed positive multiple of $m$ makes both capped and ideal ranges $\Theta(m^\gamma)$, so their ratio stays bounded away from zero. This proves the nonvanishing-fraction sufficiency in Theorem~\ref{thm:transport} without claiming every prescribed same-$q$ fraction.

Finally, couple $Z_{N_m}$ and $Z_{N_m\wedge T}$ using the same depth and the same walk trajectory. Their endpoints agree whenever $N_m\leq T$, so the coupling inequality, followed by contraction of total variation under the distance map, gives Equation~\eqref{eq:mixture-tv}. If the right-hand side of Equation~\eqref{eq:mixture-tv} is at most $\varepsilon<\min\{q,1-q\}$, their cumulative distance probabilities differ by at most $\varepsilon$ at every radius, which gives Equation~\eqref{eq:mixture-quantile-sandwich}.
For each fixed $\varepsilon$, Equation~\eqref{eq:depth-scale} supplies an $\varepsilon$-dependent constant $b_\varepsilon$ such that $T=b_\varepsilon m$ makes the tail probability at most $\varepsilon$. Thus $\Theta(m)$ depth is sufficient for any fixed total-variation tolerance, although its constant depends on that tolerance and the depth family.
\end{proof}

\subsection{Restart diffusion specialization}
\label{app:restart-specialization}

Theorem~\ref{thm:transport} applies to any positive mixture whose walk lengths have a characteristic scale. We now connect that abstraction to the restart recurrence used in the paper. At equilibrium, restart diffusion is a geometric mixture of walks: $\beta$ controls the typical walk length, while unrolling the recurrence for $T$ updates caps every sampled length at $T$. Equation~\eqref{eq:capped-restart-mixture} formalizes this connection and yields the square-root range law summarized in Section~\ref{sec:theory}; Equation~\eqref{eq:restart-profile-tv} then gives a quantitative error guarantee for the complete distance profile.

\begin{corollary}[Restart range on diffusive graphs]
\label{cor:restart-range}
Assume the quantile transport condition in Equation~\eqref{eq:transport-assumption} with $d_w=2$ and the finite-graph pre-mixing regime specified in Appendix~\ref{app:transport-assumptions}, and let
\[
\mat R_\beta=(1-\beta)(\Id-\beta\mat P)^{-1},
\qquad
\mat P_{\beta,T}
=(1-\beta)\sum_{t=0}^{T-1}\beta^t\mat P^t+\beta^T\mat P^T.
\]
For every fixed $q\in(0,1)$, as $\beta\uparrow1$, along sequences for which $\min\{T,(1-\beta)^{-1}\}\to\infty$ and the pre-mixing conditions in Appendix~\ref{app:transport-assumptions} hold,
\begin{equation}
R_q(\mat R_\beta;v)=\Theta\!\left((1-\beta)^{-1/2}\right),
\qquad
R_q(\mat P_{\beta,T};v)
=\Theta\!\left(\sqrt{\min\{T,(1-\beta)^{-1}\}}\right).
\label{eq:restart-range-laws}
\end{equation}
Thus retaining a nonvanishing fraction of the equilibrium range has update cost $\Theta(R_q(\mat R_\beta;v)^2)$. The profile total-variation error is at most $\beta^{T+1}$, so error at most $\varepsilon$ is guaranteed by the sufficient cost $O(R_q(\mat R_\beta;v)^2\log(1/\varepsilon))$.
\end{corollary}

\begin{proof}
For $\Pr(N_\beta=t)=(1-\beta)\beta^t$,
\[
\mathbb E[\mat P^{N_\beta}]
=
(1-\beta)\sum_{t=0}^{\infty}\beta^t\mat P^t
=
(1-\beta)(\Id-\beta\mat P)^{-1}
=
\mat R_\beta.
\]
Unrolling $\mat H_0=\mat X$ and $\mat H_{t+1}=\beta\mat P\mat H_t+(1-\beta)\mat X$ for $T$ updates gives
\begin{equation}
\mat P_{\beta,T}
=
(1-\beta)\sum_{t=0}^{T-1}\beta^t\mat P^t
+
\beta^T\mat P^T
=
\mathbb E[\mat P^{N_\beta\wedge T}].
\label{eq:capped-restart-mixture}
\end{equation}
The final coefficient is $\Pr(N_\beta\geq T)=\beta^T$, which explains why initializing at $\mat H_0=\mat X$ produces a cap rather than a simple truncation.

Let $m_\beta=(1-\beta)^{-1}$. For every fixed $x>0$,
$\Pr(N_\beta/m_\beta>x)=\beta^{\lfloor xm_\beta\rfloor+1}\to e^{-x}$ as $\beta\uparrow1$. Hence $N_\beta$ has characteristic scale $m_\beta$, and Theorem~\ref{thm:transport} gives, for general $d_w$,
\[
R_q(\mat R_\beta;v)
=
\Theta\!\left((1-\beta)^{-1/d_w}\right),
\qquad
R_q(\mat P_{\beta,T};v)
=
\Theta\!\left(\min\{T,(1-\beta)^{-1}\}^{1/d_w}\right).
\]
Setting $d_w=2$ proves Equation~\eqref{eq:restart-range-laws} in Corollary~\ref{cor:restart-range}.

Because $\Pr(N_\beta>T)=\beta^{T+1}$, Equation~\eqref{eq:mixture-tv} becomes
\begin{equation}
\norm{p_{\beta,T}^{\,v}-p_\beta^{\,v}}_{\mathrm{TV}}
\leq
\beta^{T+1}.
\label{eq:restart-profile-tv}
\end{equation}
Thus total-variation error at most $\varepsilon$ is guaranteed by the integer threshold
\begin{equation}
T
\geq
\max\left\{0,
\left\lceil\frac{\log(1/\varepsilon)}{-\log\beta}\right\rceil-1
\right\}
=
O\!\left(\frac{\log(1/\varepsilon)}{1-\beta}\right).
\label{eq:restart-profile-threshold}
\end{equation}
Since $R_q(\mat R_\beta;v)^2=\Theta((1-\beta)^{-1})$ in the diffusive regime, Equation~\eqref{eq:restart-profile-threshold} gives the sufficient profile-accuracy cost in Corollary~\ref{cor:restart-range}. The quantile sandwich in Equation~\eqref{eq:mixture-quantile-sandwich} applies whenever $\beta^{T+1}<\min\{q,1-q\}$.
\end{proof}

\subsection{Sharp path and operator-norm refinement}
\label{app:path-restart}

Corollary~\ref{cor:restart-range}, particularly Equation~\eqref{eq:restart-range-laws}, gives the correct scaling for broad diffusive graph families but suppresses constants and does not claim an exact finite-$T$ lower bound. On the bi-infinite path, both the equilibrium Green's function and the finite-walk tails are explicit enough to sharpen the comparison. Proposition~\ref{prop:path-restart} shows directly that the equilibrium has an exponentially decaying spatial tail, whereas a $T$-step realization concentrates on the $\sqrt{T}$ scale. It also distinguishes two notions of fidelity used in the paper: agreement of the normalized distance profile and worst-case operator-norm error.

\begin{proposition}[Bi-infinite path refinement]
\label{prop:path-restart}
Let $\mat P=\mat A/2$ be simple random-walk propagation on the bi-infinite path, and define
\[
\rho(\beta)
=
\frac{1-\sqrt{1-\beta^2}}{\beta},
\qquad
\ell(\beta)
=
-\frac{1}{\log\rho(\beta)}.
\]
Fix $q\in(0,1)$.
For the equilibrium endpoint distance $D$, the exact tail is
\begin{equation}
\Pr_\beta(D>r)
=
\frac{2\rho(\beta)^{r+1}}{1+\rho(\beta)}
\qquad(r\in\{0,1,\ldots\}),
\label{eq:path-equilibrium-tail}
\end{equation}
and $R_q(\mat R_\beta;0)=\ell(\beta)\log(1/(1-q))+O(1)$ as $\beta\uparrow1$. For the capped recurrence and $T\geq1$,
\begin{equation}
\Pr_{\beta,T}(D>r)
\leq
2\exp\!\left(-\frac{r^2}{2T}\right),
\qquad
R_q(\mat P_{\beta,T};0)
\leq
\left\lceil\sqrt{2T\log\frac{2}{1-q}}\right\rceil.
\label{eq:path-capped-tail}
\end{equation}
Consequently, if $R_q(\mat P_{\beta,T};0)\geq\kappa R_q(\mat R_\beta;0)$ for fixed $\kappa>0$, then
\begin{equation}
T
\geq
\frac{(\kappa R_q(\mat R_\beta;0)-1)_+^2}
{2\log(2/(1-q))}
=
\Omega\!\left(R_q(\mat R_\beta;0)^2\right).
\label{eq:path-range-depth-lower}
\end{equation}
Moreover,
\begin{equation}
\norm{p_{\beta,T}-p_\beta}_{\mathrm{TV}}
\geq
\sup_{r\in\{0,1,\ldots\}}
\left[
\frac{2\rho(\beta)^{r+1}}{1+\rho(\beta)}
-
2\exp\!\left(-\frac{r^2}{2T}\right)
\right]_+.
\label{eq:path-tv-lower}
\end{equation}
Along any sequence with $\beta\uparrow1$ and $T/\ell(\beta)^2\to0$, the left-hand side of Equation~\eqref{eq:path-tv-lower} tends to one. Finally, the stationary restart operator error on this path is exactly $2\beta^{T+1}/(1+\beta)$. Thus, if $T_\varepsilon(\beta)$ is its smallest depth with operator error at most a fixed $\varepsilon\in(0,1/2)$,
\begin{equation}
T_\varepsilon(\beta)
=
2\ell(\beta)^2\log\frac1\varepsilon\,(1+o(1)).
\label{eq:path-operator-depth}
\end{equation}
\end{proposition}

\begin{proof}
The Green's function of $(\Id-\beta\mat P)^{-1}$ is
\begin{equation}
\left[(\Id-\beta\mat P)^{-1}\right]_{0,d}
=
\frac{1}{\sqrt{1-\beta^2}}
\left(
\frac{1-\sqrt{1-\beta^2}}{\beta}
\right)^{|d|}.
\label{eq:pathgreen}
\end{equation}
Indeed, its coefficients solve
$g_d-\frac{\beta}{2}(g_{d-1}+g_{d+1})=\mathbf 1\{d=0\}$; retaining the characteristic root inside the unit circle and applying the equation at zero gives Equation~\eqref{eq:pathgreen}. Multiplication by $1-\beta$ and normalization over the one node at distance zero and two nodes at every positive distance yield
\[
p_\beta(0)=\frac{1-\rho}{1+\rho},
\qquad
p_\beta(d)=\frac{2(1-\rho)}{1+\rho}\rho^d
\quad(d\geq1),
\]
where $\rho=\rho(\beta)$. Summation proves Equation~\eqref{eq:path-equilibrium-tail}, and inversion gives
\[
R_q(\mat R_\beta;0)
=
\max\left\{0,
\left\lceil
\frac{\log((1-q)(1+\rho)/2)}{\log\rho}-1
\right\rceil
\right\}
=
\ell(\beta)\log\frac1{1-q}+O(1).
\]

Conditioned on $N_\beta\wedge T=n$, the endpoint is a sum of $n$ independent signs. Hoeffding's inequality gives $\Pr(|Z_n|>r)\leq2\exp(-r^2/(2n))$ for $n\geq1$, while the probability is zero for $n=0$. Since $n\leq T$, averaging proves the first inequality in Equation~\eqref{eq:path-capped-tail}; solving for a tail of at most $1-q$ gives its quantile bound and then Equation~\eqref{eq:path-range-depth-lower}.

Total variation dominates the probability difference of every tail event, so Equations~\eqref{eq:path-equilibrium-tail} and~\eqref{eq:path-capped-tail} prove Equation~\eqref{eq:path-tv-lower}. If $T/\ell^2\to0$ and $T\geq1$, choose
$r=\lfloor T^{1/4}\ell^{1/2}\rfloor$. Then $r/\ell\to0$ while $r^2/T\to\infty$: the equilibrium tail tends to one and the capped-tail bound tends to zero. The $T=0$ case has the same limit directly because its endpoint is always the origin.

Finally, $-\log\rho(\beta)=\sqrt{2(1-\beta)}+O(1-\beta)$, so
$\ell(\beta)=[2(1-\beta)]^{-1/2}(1+o(1))$. The path spectrum is $[-1,1]$, and the opposite endpoint attains the bound in Corollary~\ref{cor:restart}, giving exact error $2\beta^{T+1}/(1+\beta)$. Solving with Equation~\eqref{eq:restart-iterations} and using $-\log\beta=(1-\beta)+O((1-\beta)^2)$ proves Equation~\eqref{eq:path-operator-depth}.
\end{proof}

\subsection{Accelerated polynomial realization}
\label{app:cheb-acceleration}

The $\Theta(R_q^2)$ costs in Corollary~\ref{cor:restart-range} and Proposition~\ref{prop:path-restart} are properties of positive diffusion, not universal lower bounds for local graph filters. A polynomial of degree $T$ is still limited to $T$-hop support, but signed coefficients can combine different walk lengths with cancellation instead of waiting for a random walk to spread. Proposition~\ref{prop:cheb-acceleration} makes this distinction concrete: truncating the Chebyshev expansion of the same restart resolvent reduces the required depth from quadratic to essentially linear in its spatial range. Because its coefficients are signed, Proposition~\ref{prop:cheb-acceleration} is an operator-approximation result intentionally outside Theorem~\ref{thm:transport}.

\begin{proposition}[Chebyshev acceleration of normalized restart]
\label{prop:cheb-acceleration}
Let
\[
f_\beta(x)=\frac{1-\beta}{1-\beta x},
\qquad
\rho=\frac{1-\sqrt{1-\beta^2}}{\beta},
\]
and let $T_k$ be the $k$th Chebyshev polynomial. The degree-$T$ polynomial
\[
C_{\beta,T}(x)
=
\frac{1-\rho}{1+\rho}
\left(1+2\sum_{k=1}^{T}\rho^kT_k(x)\right)
\]
satisfies
\begin{equation}
\sup_{x\in[-1,1]}|f_\beta(x)-C_{\beta,T}(x)|
=
\frac{2\rho^{T+1}}{1+\rho}.
\label{eq:cheb-restart-error}
\end{equation}
Hence, for every symmetric $\mat S$ with spectrum in $[-1,1]$,
\[
\norm{(1-\beta)(\Id-\beta\mat S)^{-1}-C_{\beta,T}(\mat S)}_2
\leq
\frac{2\rho^{T+1}}{1+\rho},
\]
with equality whenever $1\in\operatorname{spec}(\mat S)$. By contrast, stationary restart has error $2\beta^{T+1}/(1+\beta)$ when $-1\in\operatorname{spec}(\mat S)$. Thus, for fixed $\varepsilon\in(0,1/2)$ and $\beta\uparrow1$, the respective worst-case sufficient depths have orders
\begin{equation}
T_{\mathrm{restart}}
=
\Theta\!\left(\ell(\beta)^2\log\frac1\varepsilon\right),
\qquad
T_{\mathrm{Cheb}}
=
\Theta\!\left(\ell(\beta)\log\frac1\varepsilon\right).
\label{eq:restart-cheb-depth-orders}
\end{equation}
When $\mat S=\mat D^{-1/2}\mat A\mat D^{-1/2}$ is the symmetric normalization of a diffusive graph family satisfying the bounded degree-ratio transfer in Equation~\eqref{eq:symmetric-quantile-transfer}, denote its row-normalized equilibrium influence profile from Equation~\eqref{eq:symmetric-row-profile} by $\widetilde\mu_{\beta,v}$. For every fixed $q\in(0,1)$, $R_q(\widetilde\mu_{\beta,v})=\Theta(\ell(\beta))$, so Proposition~\ref{prop:cheb-acceleration} gives a uniform operator-norm guarantee in depth $O(R_q(\widetilde\mu_{\beta,v})\log(1/\varepsilon))$. It uses signed coefficients in the walk-power basis and therefore lies outside Theorem~\ref{thm:transport}; Proposition~\ref{prop:cheb-acceleration} does not by itself assert convergence of a normalized absolute-influence profile.
\end{proposition}

\begin{proof}
The relation $\beta=2\rho/(1+\rho^2)$ and the Poisson-kernel identity give
\[
f_\beta(x)
=
\frac{1-\rho}{1+\rho}
\left(1+2\sum_{k=1}^{\infty}\rho^kT_k(x)\right).
\]
Since $|T_k(x)|\leq1$ on $[-1,1]$, the omitted tail is at most
\[
\frac{2(1-\rho)}{1+\rho}
\sum_{k=T+1}^{\infty}\rho^k
=
\frac{2\rho^{T+1}}{1+\rho}.
\]
At $x=1$, every $T_k(1)=1$, so equality holds in Equation~\eqref{eq:cheb-restart-error}. The spectral theorem proves the operator bound in Proposition~\ref{prop:cheb-acceleration} and its equality condition. The exact scalar errors imply threshold formulas
\[
T_{\mathrm{Cheb}}
=
\frac{\log(2/(\varepsilon(1+\rho)))}{-\log\rho}+O(1),
\qquad
T_{\mathrm{restart}}
=
\frac{\log(2/(\varepsilon(1+\beta)))}{-\log\beta}+O(1).
\]
Now $-\log\rho=\ell(\beta)^{-1}$ and $-\log\beta=[2\ell(\beta)^2]^{-1}(1+o(1))$, which gives Equation~\eqref{eq:restart-cheb-depth-orders}. A degree-$T$ polynomial in a one-hop operator still has support at most $T$, so acceleration removes the diffusive slowdown but not the linear locality cost.
\end{proof}

\subsection{Conditioning and cancellation: formal statement and proof}
\label{app:precision}

We now leave transport and address the representation and execution gaps described in Section~\ref{sec:theory}: a model may have adequate range yet still be difficult to fit or evaluate reliably. Let $m$ be the number of basis functions $b_j$ and fitting points $z_i$, indexed by $i,j=1,\ldots,m$, and set $B_{ij}=b_j(z_i)$. For nonzero targets $\vect{y}$ and nonsingular $\mat{B}$, the coefficients $\vect{c}=\mat{B}^{-1}\vect{y}$ give the response $r(z)=\sum_{j=1}^{m}c_jb_j(z)$. With a small denominator floor $\eta>0$, define $\chi_\eta(z)=\sum_j|c_jb_j(z)|/(|r(z)|+\eta)$; it is large when the terms are much larger than their signed sum.

\begin{theorem}[Formal conditioning and cancellation bounds]
\label{thm:realization-formal}
Let $\kappa_2(\mat{B})=\norm{\mat{B}}_2\norm{\mat{B}^{-1}}_2$ be the ratio of its largest to smallest singular value. For perturbations $\Delta\mat{B}$ and $\Delta\vect{y}$, let $\Delta\vect{c}$ be the resulting coefficient change, so that $(\mat{B}+\Delta\mat{B})(\vect{c}+\Delta\vect{c})=\vect{y}+\Delta\vect{y}$. Define $\delta_B=\norm{\Delta\mat{B}}_2/\norm{\mat{B}}_2$ and $\delta_y=\norm{\Delta\vect{y}}_2/\norm{\vect{y}}_2$. When $\kappa_2(\mat{B})\delta_B<1$,
\begin{equation}
\frac{\norm{\Delta\vect{c}}_2}{\norm{\vect{c}}_2}
\leq
\frac{\kappa_2(\mat{B})}
{1-\kappa_2(\mat{B})\delta_B}
\left(\delta_B+\delta_y\right).
\label{eq:synthesis-bound}
\end{equation}
Under the standard floating-point model for sequential summation, let $\operatorname{fl}(r(z))$ be the computed sum with unit roundoff $u$, and assume $(m-1)u<1$. Then
\begin{equation}
\frac{|\operatorname{fl}(r(z))-r(z)|}
{|r(z)|+\eta}
\leq
\gamma_{m-1}\chi_\eta(z),
\qquad
\gamma_{m-1}
=
\frac{(m-1)u}{1-(m-1)u}.
\label{eq:cancellation-bound}
\end{equation}
\end{theorem}

The theorem isolates two mechanisms. The first is sensitivity of the fitted coefficients to perturbations in the basis system; the second is loss of numerical accuracy when large signed terms nearly cancel. The proof treats these mechanisms separately, which is important because a well-conditioned equilibrium solve need not imply a well-conditioned basis representation, or conversely.

\begin{proof}
Let $\mat{E}=\mat{B}^{-1}\Delta\mat{B}$. From $(\mat{B}+\Delta\mat{B})(\vect{c}+\Delta\vect{c})=\vect{y}+\Delta\vect{y}$ and $\mat{B}\vect{c}=\vect{y}$, we obtain $(\Id+\mat{E})(\vect{c}+\Delta\vect{c})=\vect{c}+\mat{B}^{-1}\Delta\vect{y}$.
Thus,
\[
\Delta\vect{c}
=
(\Id+\mat{E})^{-1}
\big(
\mat{B}^{-1}\Delta\vect{y}-\mat{E}\vect{c}
\big).
\]

The assumption $\kappa_2(\mat{B})\delta_B<1$ implies $\norm{\mat{E}}_2\leq\kappa_2(\mat{B})\delta_B<1$ and hence
\[
\norm{(\Id+\mat{E})^{-1}}_2
\leq
\frac{1}{1-\norm{\mat{E}}_2}.
\]
Dividing by $\norm{\vect{c}}_2$ gives
\[
\frac{\norm{\Delta\vect{c}}_2}{\norm{\vect{c}}_2}
\leq
\frac{1}{1-\kappa_2(\mat{B})\delta_B}
\left(
\frac{
\norm{\mat{B}^{-1}}_2\norm{\Delta\vect{y}}_2
}{
\norm{\vect{c}}_2
}
+
\kappa_2(\mat{B})\delta_B
\right).
\]
Since $\norm{\vect{y}}_2=\norm{\mat{B}\vect{c}}_2\leq\norm{\mat{B}}_2\norm{\vect{c}}_2$, we have
\[
\frac{
\norm{\mat{B}^{-1}}_2\norm{\Delta\vect{y}}_2
}{
\norm{\vect{c}}_2
}
\leq
\kappa_2(\mat{B})
\frac{\norm{\Delta\vect{y}}_2}{\norm{\vect{y}}_2}
=
\kappa_2(\mat{B})\delta_y.
\]
Substitution proves Equation~\eqref{eq:synthesis-bound}.

For the evaluation bound, write $t_j=c_jb_j(z)$.
Under the standard floating-point model for sequential summation,
\[
\left|
\operatorname{fl}\!\left(
\sum_{j=1}^{m}t_j
\right)
-
\sum_{j=1}^{m}t_j
\right|
\leq
\gamma_{m-1}
\sum_{j=1}^{m}|t_j|,
\]
where $\gamma_{m-1}=(m-1)u/[1-(m-1)u]$.
Dividing by $|r(z)|+\eta$ yields Equation~\eqref{eq:cancellation-bound}.
\end{proof}

Equation~\eqref{eq:cancellation-bound} isolates the accumulation of already formed summands. If the products $c_jb_j(z)$ and the basis values themselves are also computed in floating point, their rounding errors contribute additional terms. Equivalently, a standard floating-point dot-product bound can be used with a slightly larger $\gamma$ factor. Our experiments measure the error of the complete implementation rather than summation alone.

\paragraph{Equilibrium-system conditioning is a separate quantity.}

If $\mat{S}$ is symmetric with spectrum in $[-1,1]$ and $\beta=|a|<1$, the singular values of $\Id-a\mat{S}$ lie in $[1-\beta,1+\beta]$. Therefore,
\begin{equation}
\kappa_2(\Id-a\mat{S})
\leq
\frac{1+\beta}{1-\beta}.
\label{eq:equilibrium-conditioning}
\end{equation}
Equation~\eqref{eq:equilibrium-conditioning} concerns the stability of applying a fixed equilibrium operator. It is unrelated in general to the condition number of a basis-synthesis matrix used to fit that operator or to cancellation among its coordinates.

\subsection{Stable real modes needed for exact-distance transport}
\label{app:modeproof}

Theorem~\ref{thm:mode} asks how many stable real-pole modes are needed to represent a sharply localized, exact-distance response. Such a response oscillates across the graph-frequency interval, while each additional real pole contributes only limited numerator complexity after the modes are placed over a common denominator. The alternating extrema of the Chebyshev target therefore give a simple counting argument: too few modes cannot match all sign changes with uniform error below one. Theorem~\ref{thm:mode} complements Theorem~\ref{thm:realization} and Appendix~\ref{app:precision} by identifying an expressivity requirement that remains even in exact arithmetic.

\begin{theorem}[Mode lower bound]
\label{thm:mode}
Let $d\geq1$ and $K\geq0$ be integers, where $d$ is the target distance and $K$ is the number of stable real-pole modes, and let
\[
r_K(x)
=
c_0+
\sum_{k=1}^{K}
\frac{c_k}{1-a_kx},
\]
where $c_0,c_1,\ldots,c_K$ are real coefficients and $a_1,\ldots,a_K$ are real pole parameters satisfying $|a_k|<1$. If $\norm{r_K-T_d}_{\infty}<1$ on $[-1,1]$, where $T_d$ is the degree-$d$ Chebyshev polynomial and $\norm{\cdot}_\infty$ is the uniform norm on $[-1,1]$, then $K\geq d$. If the direct term $c_0$ is absent, then $K\geq d+1$.
\end{theorem}

\begin{proof}
The case $K=0$ without a direct term is immediate, since then $r_K\equiv0$ and $\norm{r_K-T_d}_\infty=1$. Hence assume $K\geq1$ in that case.
Put the rational function over the common denominator $D(x)=\prod_{k=1}^{K}(1-a_kx)$. Since $|a_k|<1$, every factor is positive on $[-1,1]$, so $D(x)>0$ throughout the interval.

With the direct term present, the numerator $P(x)=D(x)r_K(x)$ has degree at most $K$. Without the direct term, each summand omits one factor, so $\deg P\leq K-1$.

Consider the $d+1$ extrema $x_j=\cos(j\pi/d)$, $j=0,\ldots,d$, for which $T_d(x_j)=(-1)^j$. The assumption $|r_K(x_j)-T_d(x_j)|<1$ forces $r_K(x_j)$ to have the same strict sign as $T_d(x_j)$. Because $D(x_j)>0$, the numerator $P(x_j)$ alternates sign at these ordered points.

By the intermediate value theorem, $P$ has at least one zero between every consecutive pair. It therefore has at least $d$ distinct zeros in $[-1,1]$, which implies $\deg P\geq d$. Combining the degree bounds gives $K\geq d$ when $c_0$ is present and $K-1\geq d$ when it is absent.
\end{proof}

Theorem~\ref{thm:mode} explains the threshold observed in the high-precision oracle: an exact-distance response with $d$ alternating extrema requires at least $d$ stable real modes to achieve uniform error below one. Interpolating a finite graph spectrum does not guarantee accurate behavior between its eigenvalues, which motivates the cross-spectrum evaluation in Section~\ref{sec:controlled} and Table~\ref{tab:cross-spectrum}.

\end{document}